\documentclass[11pt]{article}
\usepackage{CJKutf8} 
\usepackage{bbm}
\usepackage{algorithm,algpseudocode,amsmath}
\usepackage{eqnarray,amsmath,amsfonts,amsthm,mathrsfs}
\usepackage{appendix}
\usepackage{color}
\usepackage{bm}
\usepackage{amssymb}
\usepackage{graphicx}  %插入图片的宏包
\usepackage{float}  %设置图片浮动位置的宏包
\usepackage{subfigure}  %插入多图时用子图显示的宏包
\usepackage{epsfig}
\usepackage{epsf}
\usepackage{amsfonts,amsmath,amsthm,amssymb,graphicx,float,fancyhdr,multirow,hyperref}
\usepackage{booktabs,longtable,authblk}
\usepackage{mathrsfs,hhline}
\usepackage{caption}
\DeclareUnicodeCharacter{0221E}{$ \infty $}
\usepackage{fancyhdr}
\usepackage[top=2.5cm, bottom=2.5cm, left=3cm, right=3cm]{geometry}
\usepackage{graphicx}
\newtheorem{theorem}{Theorem}
\newtheorem{assumption}{Assumption}
\newtheorem{corollary}{Corollary}
\newtheorem{definition}{Definition}

\newtheorem{lemma}{Lemma}
\newtheorem{proposition}{Proposition}
\newtheorem{remark}{Remark}
\newtheorem{example}{Example}

\allowdisplaybreaks[4]
\numberwithin{equation}{section}

\def\begeqn{\begin{equation}}
\def\endeqn{\end{equation}}
\def\begth{\begin{theorem}}
\def\endth{\end{theorem}}
\def\begprop{\begin{proposition}}
\def\endprop{\end{proposition}}
\def\begcor{\begin{corollary}}
\def\endcor{\end{corollary}}
\def\begdef{\begin{definition}}
\def\enddef{\end{definition}}
\def\beglemm{\begin{lemma}}
\def\endlemm{\end{lemma}}
\def\begexm{\begin{example}}
\def\endexm{\end{example}}
\def\begrem{\begin{remark}}
\def\endrem{\end{remark}}
\def\begassum{\begin{assumption}}
\def\endassum{\end{assumption}}

\def\ee{\mathcal{E}}

\def\EE{\mathscr{E}}

\def\lan{\left\langle}
\def\ran{\right\rangle}
\def\O{\mathcal{O}}

\def\X{\mathcal{X}}

\def\NN{\mathbb{N}}
\def\RR{\mathbb{R}}

\def\X{\mathcal{X}}

\def\PP{\mathcal{P}}
\def\LLd{\mathcal{L}^2\left(\mathbb{S}^{d-1}\right)}

\def\NN{{\mathbb N}}

\def\EE{\mathbb{E}}
\def\SS{\mathbb{S}}
\def\HH{\mathcal{H}}
\def\LL{\mathcal{L}}
\def\SS{\mathbb{S}}

\def\YY{\mathcal{Y}}
\def\WW{\mathcal{W}}
\def\LL{\mathcal{L}}

\def\X{\mathcal{X}}

\def\DD{\mathcal{D}}

\title{Truncated Kernel Stochastic Gradient Descent with General Losses and Spherical Radial Basis Functions$^\dag$\footnotetext{\dag~The work of Lei Shi was partially supported by the National Natural Science Foundation of China under Grants No.12171093 and No.12571099. Email addresses: 24110180001@m.fudan.edu.cn (J. Bai), andreas.christmann@uni-bayreuth.de (A. Christmann), leishi@fudan.edu.cn (L. Shi). The corresponding author is Lei Shi. Authors are listed in alphabetical order and contributed equally to this work.}}

\author{Jinhui Bai\thanks{School of Mathematical Sciences, Fudan University, Shanghai 200433, China},\  Andreas Christmann\thanks{Department of Mathematics, University of Bayreuth, Bayreuth 95440, Germany}\ \ and Lei Shi\thanks{School of Mathematical Sciences, Fudan University, Shanghai 200433, China; Shanghai Key Laboratory for Contemporary Applied Mathematics, Fudan University, Shanghai 200433, China; and Center for Applied Mathematics, Fudan University, Shanghai 200433, China}}

\date{}
\begin{document}

\maketitle

\begin{abstract}
In this paper, we propose a novel kernel stochastic gradient descent (SGD) algorithm for large-scale supervised learning with general losses. Compared to traditional kernel SGD, our algorithm improves efficiency and scalability through an adaptive regularization strategy. By leveraging the infinite series expansion of spherical radial basis functions, this strategy projects the stochastic gradient onto a finite-dimensional hypothesis space, which is adaptively scaled according to the bias-variance trade-off, thereby enhancing generalization performance. To handle the gradient nonlinearity arising from general losses, we develop a new generalization framework combining an inequality-based characterization of the kernel-induced covariance operator with optimization techniques. We prove that both the last iterate and the suffix average converge at minimax-optimal rates, and we further establish optimal strong convergence in the reproducing kernel Hilbert space. Our framework accommodates a broad class of classical loss functions, including least-squares, Huber, and logistic losses. Moreover, the proposed algorithm significantly reduces computational complexity and achieves optimal storage complexity by incorporating coordinate-wise updates from linear SGD, thereby avoiding the costly pairwise operations typical of kernel SGD and enabling efficient processing of streaming data. Finally, extensive numerical experiments provide empirical support for the theoretical results and the computational advantages of our algorithm.
\end{abstract}



\section{Introduction}\label{section: introduction}

Spherical data naturally occur in numerous scientific domains, such as wind directions and ocean currents in geosciences, and cosmic microwave background radiation in astronomy \cite{ganachaud2000improved,hu2002cosmic}. Developing efficient approaches for modeling ubiquitous spherical data has therefore attracted considerable attention across disciplines \cite{di2014nonparametric,rosenthal2014spherical,hesse2017radial,di2019nonparametric,lin2024kernel,bai2025truncated}. In this paper, we study nonparametric supervised learning on spheres, where estimator performance is evaluated under general losses. In addition to globally convex losses, our framework also accommodates losses that are only locally strongly convex and locally smooth, thereby encompassing a wide range of commonly used loss functions in supervised learning. Formally, let the input space be the $d$-dimensional unit sphere $\mathbb{S}^{d-1}$ and the output space be an arbitrary nonempty set $\mathcal{Y}$. While our primary motivation stems from nonparametric regression---where $\mathcal{Y}$ is typically a compact subset of $\RR$---our analysis also extends to classification tasks, such as binary classification with $\mathcal{Y} = \{-1, 1\}$.  We consider samples $\{(X_i, Y_i)\}_{i \geq 1} \subset \SS^{d-1} \times \mathcal{Y}$ drawn independently from an unknown Borel distribution $\rho$ and arriving sequentially. The goal is to learn a function $f:\SS^{d-1}\to\RR$ that minimizes the population risk associated with the loss $\ell:\RR\times \YY\to \RR_+$:
\begin{equation}\label{population risk }
\min_{f\in \WW}\ee(f) := \min_{f\in \WW} \EE_\rho\left[\ell(f(X),Y)\right],
\end{equation}
where $\WW$ is a subset of an infinite-dimensional reproducing kernel Hilbert space (RKHS) induced by a kernel $K(x,x')$ constructed from spherical radial basis functions (see \autoref{subsection:Preliminaries of Spherical Harmonics} for details).

In kernel-based algorithms, appropriate regularization strategies play a crucial role in enhancing generalization performance. Traditional kernel-based stochastic gradient descent (SGD) typically introduces regularization by approximating the regularization path or adjusting the step size. However, these approaches are not imposed directly on the hypothesis space and therefore have only a limited influence on its complexity.  As a result, the hypothesis space in traditional kernel SGD does not adapt to the difficulty or ill-conditioning of the problem \eqref{population risk }, which may cause excessively rapid variance accumulation and lead to suboptimal convergence rates. In contrast,  the stochastic approximation framework proposed in this paper updates the estimator by projecting $K(X_n, \cdot)$ onto a finite-dimensional hypothesis space tailored to the difficulty of the problem. We show that this regularization strategy not only improves generalization but also substantially reduces computational complexity, while ensuring optimal memory. Specifically, the algorithm requires only $\mathcal{O}(n^{1+\frac{d}{d-1}\epsilon})$ time and $\mathcal{O}(n^{\epsilon})$ memory, where $n$ denotes the sample size. The parameter $\epsilon \in \left(0, \frac{1}{2}\right)$ can be chosen arbitrarily small, provided that the minimizer or the underlying hypothesis space possesses sufficient smoothness.

\subsection{Related Works and Discussion}

Nonparametric regression based on reproducing kernels is both theoretically well understood and widely applied across diverse areas of science and engineering \cite{bernhard2001learning,welsh2002marginal,pintore2006spatially,wu2007robust,cucker2007learning,christmann2008support,mendelson2010regularization,yang2017randomized}. Recent work has investigated the comparability between specific classes of deep neural networks and kernel methods \cite{jacot2018neural, zhu2022transition}, sparking growing interest in scalable kernel techniques for large datasets. Within the framework of nonparametric least-squares regression under batch learning, where the entire dataset is available upfront, substantial progress has been made toward improving the computational efficiency of large-scale kernel methods \cite{raskutti2014early,zhang2015divide,haim2017faster,yang2017randomized,rudi2018fast,abedsoltan2023toward,zhang2025stochastic,zhao2025nugpr}. Algorithms such as EigenPro 3.0 \cite{abedsoltan2023toward} and FALKON-BLESS \cite{rudi2018fast} leverage gradient-based optimization, preconditioning strategies, and low-rank kernel approximations to effectively reduce both memory and computational costs. The quadratic structure of the least-squares loss greatly simplifies theoretical analysis and facilitates practical implementation \cite{zhang2015divide,alessandro2017falkon}. More fundamentally, different loss functions elicit different statistical targets in population risk minimization. For instance, the least-squares loss targets the conditional mean, the logistic loss targets the conditional log-odds in binary classification, and the asymmetric least-squares loss targets conditional expectiles \cite{zhang2004statistical,gneiting2011making}. Different losses therefore address intrinsically different supervised learning tasks. Even within regression problems, from a robustness perspective, Lipschitz-continuous losses, such as the Huber loss, are often preferred to the least-squares loss because of their reduced sensitivity to outliers \cite{huber1964robust,alquier2019estimation}. Consequently, earlier works \cite{bartlett2006convexity,christmann2007consistency} studied the statistical properties of such losses, including consistency and robustness, while more recent studies \cite{marteau2019beyond,della2024nystrom,wu2025benefits} analyze the convergence of empirical risk minimization (ERM) with non-quadratic losses. However, efficient optimization with these losses on large-scale datasets remains a significant barrier. Unlike regularized ERM with the least-squares loss, which admits closed-form solutions, regularized ERM with non-quadratic losses typically requires iterative numerical solvers, thereby incurring additional computational costs. Existing large-scale kernel methods are primarily designed for least-squares loss, and extending them to handle non-quadratic losses without sacrificing efficiency is a nontrivial task. Designing a kernel method that is both computationally scalable and statistically optimal for general losses thus remains an open and pressing problem.

In online learning, where samples arrive sequentially, the estimator must be updated upon receiving each sample. This naturally motivates the use of SGD, known for its efficiency in optimization \cite{robbins1951stochastic,polyak1992acceleration,lan2020first,bottou2018optimization,jain2021making}. Consequently, SGD has been widely applied to nonparametric least-squares regression, giving rise to kernel SGD \cite{kivinen2004online}. A series of studies has analyzed the convergence of kernel SGD, beginning with \cite{smale2006online,ying2008online}, with subsequent refinements in \cite{dieuleveut2016nonparametric} and optimal rates established in \cite{guo2019fast,zhang2022sieve}. More specifically, the difficulty of the nonparametric least-squares regression problem is characterized by the spectral structure of the Hessian and by the regularity conditions that describe the smoothness of the optimal solution. Since the least-squares loss and related risk functionals (e.g., population risk, excess risk), which measure the generalization performance of the algorithm, are quadratic, the gradient of these risks reduces to an analytically tractable linear operator. As a result, convergence analyses in this setting typically rely on precise characterizations of the Hessian operator and the associated trace inequalities. In contrast, analyzing general loss functions is considerably more challenging: the gradient of the population risk \eqref{population risk } is no longer linear in the current iterate; instead, its local variation is governed by an iterate-dependent Hessian operator, unlike in the least-squares case, where the Hessian reduces to a fixed and well-understood covariance operator independent of the iterate. In such cases, analyzing the properties of the Hessian operator, particularly precisely characterizing its spectral structure, is highly nontrivial. From an optimization perspective,  \eqref{population risk } can be reformulated as a stochastic optimization problem with ill-conditioned objectives, since the Hessian eigenvalues typically decay to zero. For ill-conditioned instances of \eqref{population risk }, classical optimization techniques---typically applicable in finite-dimensional hypothesis spaces and without requiring regularity of the optimal solution---yield only the usual optimal slow rate $\mathcal{O}\left(\tfrac{1}{\sqrt{n}}\right)$ \cite{shamir2013stochastic}. However, if the objective function is well-conditioned (i.e., the eigenvalues of the Hessian are bounded away from zero), SGD in finite-dimensional spaces generally attains the optimal rate $\mathcal{O}\left(\tfrac{1}{n}\right)$ \cite{moulines2011non,shamir2013stochastic}. In the case of nonparametric least-squares regression in infinite-dimensional hypothesis spaces, strong regularity conditions on the optimal solution can improve the well-posedness of \eqref{population risk }, thereby enabling convergence rates faster than $\mathcal{O}\left(\tfrac{1}{\sqrt{n}}\right)$. This motivates us to integrate optimization techniques with generalization analysis under regularity assumptions, with the goal of establishing fast convergence rates for kernel SGD with general losses, in analogy to the least-squares setting.

In the online setting, the generalization performance of kernel SGD has been extensively investigated \cite{tarres2014online,dieuleveut2016nonparametric}, but the method inevitably incurs a quadratic computational cost in the sample size since its $n$-th update involves all preceding samples. In our recent work \cite{bai2025truncated}, we proposed a kernel SGD algorithm for the least-squares loss that incorporates coordinate-wise updates, inspired by linear SGD\footnote{Linear SGD is equivalent to kernel SGD with a linear kernel \cite{zhang2001introduction}.}. Compared with standard kernel SGD, this algorithm not only reduces the computational burden but also overcomes the saturation phenomenon in convergence rates---a limitation widely observed in the analysis of kernel SGD \cite{dieuleveut2016nonparametric,guo2019fast}---thereby achieving statistical optimality. Numerical experiments further show that, relative to popular large-scale kernel methods in the batch setting \cite{dieuleveut2016nonparametric,alessandro2017falkon,schafer2021sparse,abedsoltan2023toward}, the proposed algorithm delivers superior empirical performance, exhibiting faster convergence of the generalization error with comparable running time. The present work extends the algorithm introduced in \cite{bai2025truncated} from nonparametric least-squares regression to a substantially wider class of supervised learning tasks. In doing so, we retain its regularization strategy based on adaptively expanding finite-dimensional hypothesis spaces and its coordinate-wise update mechanism. This broader statistical scope also gives rise to a fundamentally different analytical problem. In the least-squares analysis of \cite{bai2025truncated}, the gradient depends linearly on the current iterate through a fixed covariance operator, allowing the use of spectral decompositions and trace inequalities. For a general loss, however, the population risk gradient is typically nonlinear in the current iterate, so these arguments are no longer directly applicable. We develop a new analytical framework based on a sequence of asymptotic norm inequalities and incorporate techniques from stochastic optimization into the generalization analysis. This framework extends the minimax-optimal excess risk guarantees of T-kernel SGD beyond the least-squares setting and establishes optimal convergence rates in the RKHS norm. Under the Sobolev space characterization in \autoref{section:mean result}, this strong RKHS norm convergence also implies convergence of the corresponding weak derivatives, which is important for physical consistency in many spherical applications. At the algorithmic level, we replace the full Polyak averaging and monomial coordinates used in \cite{bai2025truncated} with suffix averaging and orthonormal spherical harmonic coordinates, respectively. We prove optimal guarantees for both the suffix averaging estimator and the last iterate, while the spherical harmonic representation, combined with adaptive truncation and coordinate-wise updates, enables efficient computation and optimal storage complexity.

\subsection{Algorithm Overview and Main Contributions}

Based on spherical radial basis functions (SBFs), we develop the T-kernel SGD algorithm for general losses. The underlying hypothesis space $\mathcal{H}$, an infinite-dimensional RKHS induced by SBFs, naturally incorporates the geometry of $\mathbb{S}^{d-1}$. Exploiting the infinite series expansion of SBFs, we construct an increasing sequence of finite-dimensional nested subspaces $\{\HH_{L_n}\}_{n\geq0} \subset \HH$, where $\HH_{L_n}$ serves as the hypothesis space at the $n$-th iteration of SGD. Specifically, upon receiving the $n$-th sample, the estimator is updated along the negative direction of the projection of the stochastic gradient of \eqref{population risk } onto $\HH_{L_n}$. This amounts to truncating the original gradient within $\HH_{L_n}$, and we therefore refer to this approach as truncated kernel stochastic gradient descent, or T-kernel SGD for short. As samples arrive sequentially, the algorithm adaptively tunes its regularization strength by controlling the complexity of the hypothesis space $\HH_{L_n}$. In \autoref{section:Preliminaries and T-kernel SGD}, we show that the projected stochastic gradient onto $\HH_{L_n}$ admits an explicit closed-form expression. For the output, we adopt suffix averaging \cite{shamir2013stochastic}, which combines the advantages of Polyak averaging and the last iterate. We further discuss two approaches to extending T-kernel SGD beyond spherical domains. The first applies to a compact manifold $\Omega$ that is orientation-preserving $C^1$-diffeomorphic to $\SS^{d-1}$. The second, developed in \autoref{subsec:Truncated Kernel Stochastic Gradient Descent with General Basis}, extends the framework to kernels induced by orthonormal bases satisfying the continuity and uniform boundedness conditions. These extensions do not cover arbitrary compact input domains; thus, the generality of our framework primarily concerns the class of loss functions. Building on the new framework, we prove that T-kernel SGD achieves minimax optimal rates for general losses, up to logarithmic factors. Furthermore, we establish an optimal strong convergence result in the RKHS, which, to the best of our knowledge, is new for general losses. Finally, when the minimizer of \eqref{population risk } is sufficiently smooth, T-kernel SGD has computational complexity $\mathcal{O}(n^{1+\frac{d}{d-1}\epsilon})$ and optimal memory $\mathcal{O}(n^\epsilon)$, where $\epsilon\in(0,\tfrac{1}{2})$ may be chosen arbitrarily small.

The remainder of the paper is organized as follows. In \autoref{section:Preliminaries and T-kernel SGD}, we introduce the basic assumptions on the loss functions, briefly review the theoretical background of SBFs, and then present T-kernel SGD for general losses. In \autoref{section:mean result}, we develop the mathematical framework underlying T-kernel SGD, establish its convergence behavior, and analyze the effect of hyperparameter perturbations on convergence. In \autoref{section:Simulation Study}, we validate the theoretical results, examine the computational complexity through numerical experiments, and further demonstrate the algorithm's performance on real-world datasets and in settings involving latent physical constraints. All proofs of the theorems are deferred to the Appendix.

\section{Preliminaries and Algorithm}\label{section:Preliminaries and T-kernel SGD}

In this section, we outline the basic assumptions on the loss functions and give examples that satisfy them. We then review the theoretical foundations of spherical radial basis functions and their role in defining the hypothesis space. Finally, we introduce truncated kernel stochastic gradient descent and discuss its extension to broader input domains.

\subsection{Loss Functions}\label{subsection:Preliminaries of Losses}

The primary objective of this paper is to infer the function $f^*$ that minimizes the population risk over a subset $\WW$ of the underlying space, i.e.,
$$f^*:=\arg\min_{f\in \WW}\ee(f)=\arg\min_{f\in \WW}\EE_\rho\left[\ell\left(f\left(X\right),Y\right)\right],$$
where $\ell:\RR\times \YY\to \RR_+$ denotes a loss function. Intuitively, when the loss function exhibits locally quadratic behavior, one may expect the algorithm to achieve the same convergence rate as in the least-squares loss. Moreover, our assumptions are broad enough to encompass many standard losses in supervised learning, including least-squares, logistic, Poisson, and Cauchy losses. We next introduce several commonly used notions, such as local strong convexity and local smoothness, to characterize the loss function.

\begin{assumption}\label{assum:extence}
On the domain $[-B, B] \times \mathcal{Y}$, where $B > 0$ is a fixed constant, the loss function $\ell(u, v)$ is partially differentiable with respect to $u$, and its derivative is uniformly bounded; that is, there exists a constant $M > 0$ such that $|\partial_u \ell(u, v)| \leq M$ for all $(u, v) \in [-B, B] \times \mathcal{Y}$.
\end{assumption}

\begin{assumption}\label{assum:Lsmooth}
{\rm(Local $L$-smoothness)} The loss function $\ell(u,v)$ is $L$-smooth on $[-B,B]$; that is, there exists a constant $L > 0$ such that for all $u_1, u_2 \in [-B,B]$, it holds
\begin{align}\label{eq:Lsmooth}
\left|\partial_u\ell(u_1,v)-\partial_u\ell(u_2,v)\right|\leq L|u_1-u_2|, \ \ \forall v\in\YY.
\end{align}
\end{assumption}
\begin{assumption}\label{assum:mustorngconvex}
{\rm (Local $\mu$-strong convexity)} The loss function $\ell(u,v)$ is $\mu$-strongly convex with respect to its first argument $u$ over the interval $[-B,B]$; that is, there exists a constant $\mu > 0$ such that for all $u_1, u_2 \in [-B,B]$, one has
\begin{align}\label{eq:mustorngconvex}
\ell(u_1,v)-\ell(u_2,v)-\partial_u\ell(u_2,v)(u_1-u_2)\geq\frac{\mu}{2}(u_1-u_2)^2,\ \ \forall v\in\YY.
\end{align}
\end{assumption}

\autoref{assum:extence} and \autoref{assum:Lsmooth} together guarantee the existence of the Fr\'{e}chet derivative (see, e.g., \cite{ciarlet2013linear}) of the population risk, thereby ensuring that the stochastic gradient descent algorithm is well-defined.  Local smoothness, as formalized in \autoref{assum:Lsmooth}, is a standard and widely adopted assumption in optimization \cite{nesterov2013introductory}. In finite-dimensional hypothesis spaces, the local strong convexity of the loss is sufficient to guarantee the optimal rate $\mathcal{O}(\frac{1}{n})$ \cite{moulines2011non,jain2021making}, while assuming only convexity typically leads to the slow rate $\mathcal{O}(\frac{1}{\sqrt{n}})$ \cite{shamir2013stochastic}. \autoref{assum:Lsmooth} and \autoref{assum:mustorngconvex} are essential for establishing the fast rates we aim to prove. Moreover, these assumptions can be readily verified under the following sufficient condition: if the second-order partial derivative $\partial_{uu}^2 \ell(u, v)$ is positive and bounded above by $L>0$ and below by $\mu > 0$ on $[-B, B] \times \mathcal{Y}$, and if $\partial_u \ell(u, v)$ is also bounded, then \autoref{assum:extence}, \autoref{assum:Lsmooth}, and \autoref{assum:mustorngconvex} hold.

In nonparametric regression, the output space $\mathcal{Y}$ is typically assumed to be a subset of $\mathbb{R}$ \cite{christmann2008support,tarres2014online,dieuleveut2016nonparametric}.  In our framework, $\mathcal{Y}$  may be any nonempty set, allowing the response variable $Y$  to take values in a discrete set for classification or to represent sequences, functional data, and other types of outputs. Under the above three assumptions, our framework depends only on local properties of the loss. This enables it to cover certain negative log-likelihood losses that are well defined only on restricted domains, as well as several globally non-convex losses commonly used in robust regression, including the Cauchy loss \cite{black1996robust} and the Welsch loss \cite{holland1977robust}. Below, we list several commonly used losses in supervised learning that satisfy our assumptions. 
Unless otherwise specified, we assume that $(u,v)\in[-B,B]\times\YY$, where $B>0$ is fixed, and that $\YY$ is a compact subset of $\RR$.
\begin{itemize}\label{lossfunction}
\item Least-squares loss: $\ell(u,v)=(u-v)^2$. 
\item Logistic loss: $\ell(u,v)=\log(1+e^{-vu})$, where $\YY=\{-1,1\}$.
\item Loss in Poisson regression: $\ell(u,v)=e^u-uv+\log(v!)$, where $\YY$ is a finite subset of $\NN$.
\item Huber loss: $\ell(u,v)=W(v-u)$, for $W(t)=\sqrt{t^2+1}-1$ or $W(t)=\log\frac{e^t+e^{-t}}{2}$.
\item Cauchy loss: $\ell(u,v)=\log\left(1+\frac{(u-v)^2}{2}\right)$,  where $B=\frac{1}{3}$ and $\YY=\left[-\frac{1}{3},\frac{1}{3}\right]$.
\item Welsch loss: $\ell(u,v)=1-\exp\left(-\frac{(u-v)^2}{2}\right)$,  where $B=\frac{1}{3}$ and $\YY=\left[-\frac{1}{3},\frac{1}{3}\right]$.
\item Asymmetric least-squares loss: $\ell_{\tau}(u,v)=\left|\mathcal{X}_{\{u\geq v\}}-\tau\right|(u-v)^2$, where $\tau\in(0,1)$.
\end{itemize}
Here, $\mathcal{X}_{\{u\geq v\}}$ denotes the indicator of $\{u\geq v\}$. Among these, the third loss function is the standard choice for Poisson regression. Notably, both the Cauchy and Welsch losses are globally non-convex, and the latter has attracted considerable attention in the image processing community. For convex losses that may not satisfy local strong convexity, our algorithm still enjoys convergence guarantees under the following assumption.
\begin{assumption}\label{assum:convex}
{\rm (Convexity)} On the domain $\RR\times\mathcal{Y}$, suppose that, for every $v\in\mathcal{Y}$, the function $u\mapsto\ell(u,v)$ is convex and differentiable on $\RR$. Assume that $\partial_u\ell(u,v)$ is uniformly bounded and uniformly continuous in $u$, uniformly over $v\in\mathcal{Y}$. Specifically, there exists a constant $M > 0$ such that $|\partial_u \ell(u, v)| \leq M$ for all $(u, v) \in \RR\times \mathcal{Y}$. For every $\epsilon>0$, there exists $\delta_\epsilon>0$ such that, for all $u,u'\in\RR$ and $v\in\mathcal{Y}$ satisfying $|u-u'|<\delta_\epsilon$, one has
$\left|\partial_u \ell(u', v) - \partial_u \ell(u, v)\right| < \epsilon.$
\end{assumption}

\subsection{Spherical Radial Basis Functions}\label{subsection:Preliminaries of Spherical Harmonics}

In this subsection, we briefly introduce the theoretical background of spherical harmonics and spherical radial basis functions (SBFs). For more details on spherical harmonics, we refer the reader to Chapter 1 of \cite{dai2013approximation}.  Let $\omega$ denote the surface measure on $\SS^{d-1}$. The space $\LLd$ consists of functions that are square-integrable with respect to the measure $\omega$ and is equipped with the norm $\Vert\cdot\Vert_\omega$ induced by the inner product
$$\langle f,g\rangle_{\omega}:=\frac{1}{\Omega_{d-1}}\int_{\SS^{d-1}}f(x)g(x)d\omega(x),\ \
\forall f,g\in\LLd,$$
where $\Omega_{d-1}$ denotes the surface area of $\SS^{d-1}$. A function $P(x)$ is regarded as a homogeneous polynomial of degree $k$ on $\SS^{d-1}$, given by $P(x) = \sum_{|\alpha|=k} C_\alpha x^\alpha$, where $\alpha=(\alpha_1,\dots,\alpha_d)\in \NN^d$. The space of all homogeneous polynomials of degree $k$ on $\SS^{d-1}$ is denoted by $\PP_k^d$, while $\Pi_k^d$ denotes the space of all polynomials of degree at most $k$ defined on $\SS^{d-1}$. We denote by $\HH_k^d$ the space of spherical harmonics of degree $k$,
$$\HH_k^{d}:=\left\{P\in\PP_k^{d}\ |\ \Delta P=0\right\},$$
where $\Delta$ is the Laplacian operator. According to Chapter 1.2 of \cite{dai2013approximation}, the space $\HH_k^d$ is a reproducing kernel Hilbert space (RKHS) with kernel $K_k(x,x') = Q_k^d(\langle x, x'\rangle)$ for $d \geq 3$, where $Q_k^d$ denotes the generalized-Legendre polynomial and $\langle x,x'\rangle$ is the standard inner product in $\RR^d$. When $d=2$, $\HH_k^d$ is also an RKHS with kernel function $K_k(x,x')$ given in Chapter 1.6.1 of \cite{dai2013approximation}. The generalized Legendre polynomials $Q_k^d(u)$ for $d \geq 3$ are defined by $Q_0^d(u) := 1$,
\begin{align*}
\frac{1}{\Omega_{d-1}}\int_{-1}^1 Q_k^d(u)Q_j^d(u)(1-u^2)^{\frac{d-3}{2}}du&:=\frac{\dim\HH_k^d}{\Omega_{d-2}}\delta_{k,j}, \ \ \forall\ k,j\geq 0.
\end{align*}
For the orthonormal basis $\{Y_{k,j}\}_{1\leq j\leq\dim \HH_k^d}$ of the space $(\HH_k^d,\langle\cdot,\cdot\rangle_{\omega})$, we have $K_k(x,x')=\sum_{j=1}^{\dim \HH_k^d}Y_{k,j}(x)Y_{k,j}(x')$. Another important property is that the spaces $\{(\HH_k^d,\langle\cdot,\cdot\rangle_{\omega})\}_{k\geq0}$ are mutually orthogonal and form an orthogonal decomposition of both $\LLd$ and $\Pi_k^d$, where $\bigoplus$ denotes the direct sum of inner product spaces,
$$\Pi_k^{d}=\bigoplus_{0\leq j\leq k}\HH_{j}^{d}\quad\quad\text{and}\quad\quad\LLd=\bigoplus_{k\geq 0}\HH_{k}^{d}.$$

Consider a common class of SBFs $ Q(u) := \sum_{k=0}^\infty a_k Q_k^d(u),$ inducing the kernel function
\begin{equation}\label{eq:def kernel K}
K(x,x'):=\sum_{k=0}^\infty a_kQ_k^d(\langle x,x'\rangle)=\sum_{k=0}^\infty a_kK_k(x,x')=\sum_{k=0}^\infty a_k\sum_{j=1}^{\dim \HH_k^d}Y_{k,j}(x)Y_{k,j}(x').
\end{equation}
The coefficients $0< a_k \leq 1$ satisfy $l:=\lim_{k\to\infty} a_k\cdot\left(\dim \Pi_k^d\right)^{2s} \in(0,\infty)$ for some $s > \frac{1}{2}$, with $\left(\dim \Pi_k^d\right)^{2s} = \O(k^{2s(d-1)})$. For such a kernel $K(x,x')$, we established in \autoref{kernel uniformly convergence} of \autoref{appendix pre} that $K(x,x')$ converges uniformly and is therefore continuous. Together with its easily verifiable symmetry and positive definiteness, $K(x,x')$ is a Mercer kernel \cite{aronszajn1950theory}, inducing the RKHS $\HH_K$ given by
\begin{equation}\label{eq:RKHS HK}
\HH_K=\left\{f=\sum_{k=0}^\infty\sum_{1\leq j\leq\dim\HH_k^{d}} f_{k,j}Y_{k,j}\,\Bigg\vert\,\sum_{k=0}^\infty\sum_{1\leq j\leq\dim\HH_k^{d}} \frac{\left(f_{k,j}\right)^2}{a_k}<\infty\right\}\
\end{equation}
with inner product
\begin{equation}\label{eq:RKHS inner product}
\lan f,g\ran_K:= \sum_{k=0}^\infty\sum_{1\leq j\leq\dim\HH_k^{d}}\frac{f_{k,j}\cdot g_{k,j}}{a_k}.
\end{equation}
The corresponding covariance operator is defined as
\begin{equation}\label{eq:covariance operator}
\begin{aligned}
L_{\omega,K} : \LL^2(\SS^{d-1}) &\to \LL^2(\SS^{d-1}),\quad f \mapsto \frac{1}{\Omega_{d-1}} \int_{\SS^{d-1}} f(x) K(x, \cdot) d\omega(x).
\end{aligned}
\end{equation}
The capacity parameter $s$ is used to characterize the complexity of the hypothesis space $\HH_K$, and as $s$ increases, the space $\HH_K$ becomes smaller.
Under the new inner product $\langle\cdot,\cdot\rangle_K$, the spaces $\{\HH_k^d\}_{k\geq0}$ remain mutually orthogonal. Moreover, each $\mathcal{H}_k^d$ is an RKHS with kernel $a_k K_k(x, x')$ under $\langle \cdot, \cdot \rangle_K$. For further details on $\mathcal{H}_K$, we refer the reader to our previous work \cite{bai2025truncated}. Given an increasing sequence of non-negative integers  $\{L_n\}_{n\geq0}\subset\NN$, we define an increasing family of finite-dimensional, nested function spaces $\{\HH_{L_n}\}_{n\geq0} \subset \HH_K$ by $\mathcal{H}_{L_n}:=\bigoplus_{k=0}^{L_n}\HH_k^d$, as described in \autoref{section: introduction}. According to Theorem 12.20 of \cite{wainwright2019high} and the orthogonality of $\{\HH_k^d\}_{k\geq0}$, the space $\left(\mathcal{H}_{L_n}, \langle \cdot, \cdot \rangle_K\right)$ forms an RKHS with kernel $K_{L_n}^T(x, x')$, which expands as
\begin{equation}\label{eq:HH_K define}
K_{L_n}^T(x,x') = \sum_{k=0}^{L_n}a_kK_k(x,x')=\sum_{k=0}^{L_n} a_k\sum_{j=1}^{\dim \HH_k^d}Y_{k,j}(x)Y_{k,j}(x'),
\end{equation}
with inner product $\lan f,g\ran_K= \sum_{k=0}^{L_n}\sum_{1\leq j\leq\dim\HH_k^{d}}\frac{f_{k,j}\cdot g_{k,j}}{a_k}$ for all $f,g\in\HH_{L_n}$.

\subsection{Truncated Kernel Stochastic Gradient Descent}\label{subsection:T-kernel SGD}

First, we introduce some notation and definitions. Let $\rho_X$ denote the marginal distribution of $\rho$ with respect to $X$, supported on the sphere $\SS^{d-1}$. The space of square $\rho_X$-integrable functions is denoted by $\left(\LL^2_{\rho_X}\left(\SS^{d-1}\right),\langle\cdot,\cdot\rangle_{\rho_X}\right)$. We first construct truncated kernel stochastic gradient descent (T-kernel SGD) under local strong convexity and smoothness of the loss function, that is, under \autoref{assum:extence}, \autoref{assum:Lsmooth}, and \autoref{assum:mustorngconvex}. We then discuss the algorithmic design for merely convex losses, corresponding to \autoref{assum:convex}. The first three assumptions restrict the loss function to the set $[-B, B] \times \mathcal{Y}$, which in turn implies that the range of $f$ lies in $[-B, B]$, i.e., $\Vert f\Vert_\infty = \sup_{x \in \SS^{d-1}} |f(x)| \leq B$. This condition is easily satisfied by functions in $\mathcal{H}_K$ due to the reproducing property. By \autoref{kernel uniformly convergence} in \autoref{appendix pre}, we define $$\sup_{x,x'\in\SS^{d-1}}K(x,x')=\sup_{x\in\SS^{d-1}}\Vert K(x,\cdot)\Vert_{K}^2=:\kappa^2<\infty,$$
 so that $\sup_{x\in\SS^{d-1}}|f(x)|\leq\|f\|_K\sup_{x\in\SS^{d-1}}\Vert K(x,\cdot)\Vert_{K}=\kappa\|f\|_{K}$. Choosing $Q$ such that $\kappa Q<B$, define a  closed convex subset $\WW$ of $\HH_K$ as
\begin{equation}\label{eq:WW defintion}
\WW:=\left\{f\in\HH_K\,|\ \Vert f\Vert_K\leq Q\right\}.
\end{equation}
Hence, for all $f \in \WW$, we have $\Vert f\Vert_\infty \leq \kappa Q < B$. Under \autoref{assum:extence}, \autoref{assum:Lsmooth}, and the reproducing property of $\HH_K$, \autoref{lemma:A.1} yields the following relation characterizing Fr\'{e}chet differentiability \cite{ciarlet2013linear}. For any $f \in \WW$ and $h\in\HH_K$, it holds that
\begin{equation}\label{eq:derivative}
                \begin{aligned}
o(\Vert h\Vert_K)&=\EE\left[\ell(f(X)+h(X),Y)-\ell(f(X),Y)-\partial_u\ell(f(X),Y)h(X)\right]\\
&=\ee(f+h)-\ee(f)-\lan\EE\left[ \partial_u\ell(f(X),Y)K(X,\cdot)\right],h\ran_K.
\end{aligned} 
\end{equation}
The Fr\'{e}chet derivative of $\ee(f)$ in $\HH_K$ is  $\nabla\ee(f)\big|_{\HH_K}=\EE\left[ \partial_u\ell(f(X),Y)K(X,\cdot)\right]$, for which $\widehat{\nabla \ee(f)}\big|_{\HH_K} = \partial_u\ell(f(X_n),Y_n)K(X_n,\cdot)$ serves as an unbiased estimator.

We choose an increasing sequence of non-negative integers $\{L_n\}_{n\geq0}$, typically defined as $L_n=\min\left\{k\,\Big\vert\,\dim{\Pi_k^{d}}\geq n^\theta\right\}$ with $\theta>0$. At the $n$-th iteration, we project the unbiased estimator $\widehat{\nabla \ee(f)}\big|_{\HH_K}$ onto the space $\HH_{L_n} = \bigoplus_{k=0}^{L_n} \HH_k^d$ (see \eqref{eq:HH_K define} for more details), given by
\begin{equation*}
P_{\HH_{L_n}}\left(\widehat{\nabla\ee(f)}\big|_{\HH_K}\right)=\partial_u\ell(f(X_n),Y_n)K_{L_n}^T(X_n,\cdot)=\partial_u\ell(f(X_n),Y_n)\sum_{k=0}^{L_n}a_kK_k(X_n,\cdot),
\end{equation*}
where $P_{\HH_{L_n}}$ denotes the projection operator from $\HH_K$ onto $\HH_{L_n}$, and this result is established in \autoref{lemma:A.2}. \autoref{lemma:A.2} also shows that for any $f \in \HH_{L_n}\cap\WW$, $\partial_u \ell(f(X_n), Y_n) K_{L_n}^T(X_n, \cdot)$ is an unbiased estimator of the gradient of the population risk $\ee(f)$ in $\HH_{L_n}$. In the algorithm, by tuning the parameter $\theta$, which determines the dimensionality of the hypothesis space $\HH_{L_n}$, we establish a regularization mechanism that adapts to the complexity of $f^*$. Specifically, a smaller $\theta$ helps prevent overfitting when $f^*$ exhibits strong regularity, whereas a larger $\theta$ mitigates underfitting under weak regularity. In addition, we introduce the projection operator $P_{\WW}:\HH_K \rightarrow \WW$, which projects elements of $\HH_K$ onto $\WW$ to ensure that each iteration remains in $\WW$. Using unbiased estimates of the derivatives, we recursively define a sequence of iterates $\hat{f}_n\in \HH_{L_n}\cap\WW$, starting from the initialization $\hat{f}_0 = 0$, and
\begin{equation}\label{eq:T-kernel SGD expression}
\begin{aligned}
\hat{f}_n:=&P_{\WW}\left(\hat{f}_{n-1}-\gamma_n\partial_u\ell(\hat{f}_{n-1}(X_n),Y_n)K_{L_n}^T(X_n,\cdot)\right)\\
=&P_{\WW}\left(\hat{f}_{n-1}-\gamma_n\partial_u\ell(\hat{f}_{n-1}(X_n),Y_n)\sum_{k=0}^{L_n} a_k\sum_{j=1}^{\dim \HH_k^d}Y_{k,j}(X_n)Y_{k,j}\right)
\end{aligned}
\end{equation}
with step size $\gamma_n=\gamma_0n^{-t}$ for $t\in\left[\frac{1}{2},1\right)$ and $\gamma_0>0$. In \autoref{lemma:A.4}, we show that $P_{\WW}(f)\in \HH_{L_n}\cap\WW$ for any $f \in \HH_{L_n}$. By induction, since $\hat{f}_{n-1}\in \HH_{L_{n-1}}$ and $K_{L_n}^T(X_n,\cdot)\in \HH_{L_n}$, it follows that $\hat{f}_n\in\WW\cap \HH_{L_n}$. In \autoref{lemma:A.6}, we provide an explicit expression for the projection operator $P_\WW$ in the subspace $\HH_{L_n}$. For $f=\sum_{k=0}^{L_n}\sum_{j=1}^{\dim\HH_k^d}f_{k,j}Y_{k,j}\in\HH_{L_n}$, we have
\begin{equation}\label{eq:projection expression}
P_\WW(f)=
          \left\{
          \begin{aligned}
              &\frac{Q}{\Vert f\Vert_K} f=\frac{Q}{\left(\sum_{k=0}^{L_n}\sum_{j=1}^{\dim\HH_k^d}a_k^{-1}f_{k,j}^2\right)^{\frac{1}{2}}}f,\ \ \text{if}\ \Vert f\Vert_K>Q, \\
              & f,\quad\quad\quad\quad\quad\quad\quad\quad\quad\quad\quad\quad\quad\quad\quad\ \,\ \ \text{if}\ \Vert f\Vert_K\leq Q. 
          \end{aligned}
          \right.
      \end{equation}
In addition to outputting the last iterate $\hat{f}_n$, T-kernel SGD also adopts a more robust $\alpha$-suffix averaging scheme. Specifically, for a fixed averaging parameter $\alpha \in (0,1)$, we define
$$\bar{f}_{\alpha n}:=\frac{1}{\alpha n}\left(\hat{f}_{(1-\alpha)n}+\dots+\hat{f}_{n-2}+\hat{f}_{n-1}\right).$$

Since $\hat{f}_{n-1}\in\HH_{L_{n-1}}$, we write $\hat{f}_{n-1}=\sum_{k=0}^{L_{n}} \sum_{j=1}^{\dim\HH_k^d} f_{k,j}^{(n-1)}Y_{k,j}$ (with $f_{k,j}^{(n-1)}=0$ for $L_{n-1}<k\le L_n$) and define $\hat{g}_n := \hat{f}_{n-1}-\gamma_n\partial_u\ell(\hat{f}_{n-1}(X_n),Y_n)K_{L_n}^T(X_n,\cdot)$. In practice, the update of $\hat{g}_n$ is performed directly on the coefficients of its expansion, i.e.,
$$ \hat{g}_n= \sum_{k=0}^{L_{n}} \sum_{j=1}^{\dim\HH_k^d}g_{k,j}^{(n)}Y_{k,j}:=\sum_{k=0}^{L_{n}} \sum_{j=1}^{\dim\HH_k^d}
\left( f_{k,j}^{(n-1)}-\gamma_n\partial_u\ell(\hat{f}_{n-1}(X_n),Y_n)a_kY_{k,j}(X_n)\right)Y_{k,j}.$$
From \eqref{eq:projection expression}, the projection operation on $ \hat{g}_n$, i.e., $\hat{f}_n = P_\WW( \hat{g}_n)$, essentially only involves operations on the coefficients of the expansion of $\hat{g}_n$. In the recursion of the T-kernel SGD \eqref{eq:T-kernel SGD expression}, aside from computing the function value $\hat{f}_{n-1}(X_n)=\sum_{k=0}^{L_{n}} \sum_{j=1}^{\dim\HH_k^d} f_{k,j}^{(n-1)}Y_{k,j}(X_n)$, all other operations are performed on the coefficients of the basis $\{Y_{k,j}\}$. The explicit forms of the basis $\{Y_{k,j}\}$ and related details are provided in \autoref{subsection:Orthonormal Basis of the Spherical Harmonic Space}.

Projection, as a key step in algorithm \eqref{eq:T-kernel SGD expression}, is a standard operation in stochastic approximation \cite{harvey2019tight,lan2020first,jain2021making}. On the one hand, when the algorithm is applied to certain negative log-likelihood losses that are defined only on a restricted domain, the projection ensures that the estimator remains within the admissible region of the loss. On the other hand, in robust regression, the projection can constrain the norm of $\hat{f}_n$, thereby reducing the influence of outliers and improving the robustness of the algorithm. In addition, \eqref{eq:T-kernel SGD expression} may be viewed formally as an extension of proximal gradient descent to an infinite-dimensional setting, with the projection restricting the iterates to a well-behaved region of the loss function and thereby enabling fast rates. However, when we seek to establish global convergence for general losses in the absence of prior information on the RKHS norm of the minimizer or under misspecification (e.g., when $\Vert f^*\Vert_K=\infty$), the projection step may impose additional artificial constraints on the algorithm. This consideration leads us to study an unprojected version of T-kernel SGD for general losses. Under \autoref{assum:convex}, and arguing as in the preceding analysis, one can show that the gradients of $\ee(f)$ in $\HH_K$ and $\HH_{L_n}$ are given by $\nabla\ee(f)\big|_{\HH_K}=\EE\left[ \partial_u\ell(f(X),Y)K(X,\cdot)\right]$ and $\nabla\ee(f)\big|_{\HH_{L_n}}=\EE\left[ \partial_u\ell(f(X),Y)K_{L_n}^T(X,\cdot)\right]$, respectively. We choose a diminishing step size $\gamma_n = \gamma_0 n^{-t}$ with $0 < t \leq\frac12$ and initialize $\hat{f}_0 = 0$, 
\begin{equation}\label{eq:T-kernel SGD expression convex}
\begin{aligned}
\hat{f}_n:=&\hat{f}_{n-1}-\gamma_n\partial_u\ell(\hat{f}_{n-1}(X_n),Y_n)K_{L_n}^T(X_n,\cdot).
\end{aligned}
\end{equation}
We also employ the suffix-average $\bar{f}_{\alpha n}$ as the output. Therefore, we can directly present the T-kernel SGD in Algorithm \ref{alg:theoretical}.

In Algorithm \ref{alg:theoretical}, the computational cost of each iteration is mainly attributed to evaluating $\hat{f}_{n-1}(X_n)$, updating $\hat{g}_n$, $\hat{f}_n$, and computing $\Vert \hat{g}_n \Vert_K$. The latter three operations require comparable computational time $\O\left(\sum_{k=0}^{L_n}\dim\HH_k^d\right) = \O\left(\dim\Pi_{L_n}^d\right)$. The former requires computing the basis functions $\{Y_{k,j}(X_n)\}$. As shown in \autoref{subsection:Orthonormal Basis of the Spherical Harmonic Space}, the evaluation of each basis function $Y_{k,j}(X_n)$ for $0 \leq k \leq L_n$ can be performed in at most $\O(dL_n)$ time, which implies that the evaluation of $\hat{f}_{n-1}(X_n)$ takes at most $\O(dL_n\dim\Pi_{L_n}^d)$ time. The total computational time of T-kernel SGD for processing $n$ samples is $\O(dnL_n\dim\Pi_{L_n}^d)$. The dimension factor $d$ in the computational time may be an artifact of our analysis. In practice, evaluating each basis function in $\HH_0^d$, $\HH_1^d$, and $\HH_2^d$ in arbitrary dimensions requires at most 1, 2, and 10 operations, respectively, regardless of the dimension $d$. For storage, T-kernel SGD only requires the coefficients of $\hat{f}_n$ and $\hat{g}_n$, together with intermediate quantities represented in the coefficients of the basis functions $\{Y_{k,j}\}_{0\leq k\leq L_n,1\leq j\leq\dim\HH_k^d}$. The memory consumption of the algorithm is $\O(\dim \Pi_{L_n}^d)$. A more in-depth analysis of both computational and storage complexities is provided in \autoref{subsection:Optimal Rates for Excess Risk}.
\begin{algorithm}[htbp]
\footnotesize
\caption{Truncated Kernel Stochastic Gradient Descent}
\label{alg:theoretical}

\setlength{\baselineskip}{8pt} 
\begin{algorithmic}
\State{\textbf{set}:  $s>\frac{1}{2},\ \gamma_0>0,\ 0< t<1$,\ $\theta >0$, $Q>0$, $0<\alpha<1$, and $L_0=0$.}
\State{\textbf{initialize}: $\hat{f}_0 = 0,\ K^T_{L_0}(x,\cdot)=a_0K_0(x,\cdot)=a_0Y_{0,1}(x)Y_{0,1}$.}
\For{$n=1,2,3,\dots$}
\State{Collect a sample $(X_n,Y_n)$, calculate $\gamma_n = \gamma_0n^{-t}$ and $L_n$. Update $\hat{g}_n$\ :\ 
\begin{align*}
\hat{g}_n=&\hat{f}_{n-1}-\gamma_n\partial_u\ell(\hat{f}_{n-1}(X_n),Y_n)K_{L_n}^T(X_n,\cdot)\\
=&\sum_{k=0}^{L_{n}} \sum_{j=1}^{\dim\HH_k^d}g_{k,j}^{(n)}Y_{k,j}:=\sum_{k=0}^{L_{n}} \sum_{j=1}^{\dim\HH_k^d}
\left( f_{k,j}^{(n-1)}-\gamma_n\partial_u\ell(\hat{f}_{n-1}(X_n),Y_n)a_kY_{k,j}(X_n)\right)Y_{k,j}.
\end{align*}}
\If{ the projected version is used}
\State{Calculate $\Vert\hat{g}_n\Vert_K^2=\sum_{k=0}^{L_n}\sum_{j=1}^{\dim\HH_k^d}a_k^{-1}\left(g_{k,j}^{(n)}\right)^2$ and update $\hat{f}_n$:
\begin{align*}
\hat{f}_n=P_{\WW}(\hat{g}_n)=
          \left\{
          \begin{aligned}
              &\frac{Q}{\Vert \hat{g}_n\Vert_K}\cdot\sum_{k=0}^{L_{n}} \sum_{j=1}^{\dim\HH_k^d}g_{k,j}^{(n)}Y_{k,j},\ \ \text{if}\ \Vert \hat{g}_n\Vert_K>Q, \\
              & 
              \hat{g}_n,\ \ \text{if}\ \Vert \hat{g}_n\Vert_K\leq Q. 
          \end{aligned}
          \right.
\end{align*}}
\Else{}
\State{ $\hat{f}_n = \hat{g}_n$}
\EndIf{}
\EndFor{}
\State{\textbf{return}\ $\hat{f}_n,\bar{f}_{\alpha n}=\frac{1}{\alpha n}\sum_{i=(1-\alpha)n}^{n-1}\hat{f}_{i}$}
\end{algorithmic}

\end{algorithm}
Designing algorithms based on SBFs has long been a classical approach in spherical data analysis. Extending this classical methodology to certain well-behaved non-spherical data remains an interesting and open problem. Let $\Omega$ be a compact manifold, and suppose that the samples $\{(X_i, Y_i)\}_{i\geq1} \subset \Omega \times \YY$ are independent samples from an unknown Borel probability distribution $\rho$. We still denote by $\rho_X$ the marginal distribution of $\rho$ with respect to $X$. The space of square $\rho_X$-integrable functions is still denoted by $\left(\LL^2_{\rho_X}(\Omega), \langle \cdot, \cdot \rangle_{\rho_X} \right)$. Here, we choose an orientation-preserving $C^1$-diffeomorphism $F: \Omega \to \SS^{d-1}$ (see \cite{lee2013introduction} for details), with inverse $F^{-1}$, so that each $X_i$ is mapped onto the sphere by $F$, i.e., $F(X_i) \in \SS^{d-1}$.  Ellipsoids and $C^1$ closed hypersurfaces that are star-shaped with respect to the origin admit such a diffeomorphism $F$ (see \cite{gerhardt1990flow}). In this way, SBFs can be effectively applied to non-spherical manifolds. Note that for any $f \in \HH_K$, the composition $f \circ F$ belongs to $\LL^2_{\rho_X}(\Omega)$. Since $\Vert f\Vert_\infty \leq \kappa \Vert f\Vert_K$, we have
$$\Vert f\circ F\Vert_{\rho_X}^2=\int_\Omega |f\circ F(X)|^2d\rho_X\leq \Vert f\Vert_\infty^2 \leq \kappa^2 \Vert f\Vert_K^2.$$

We still consider the population risk minimization problem
$$f^*:=\arg\min_{f\in \WW}\ee(f)=\arg\min_{f\in \WW}\EE_\rho\left[\ell\left(f\circ F\left(X\right),Y\right)\right].$$
Using the mapping $F$, we generalize T-kernel SGD to samples from such a non-spherical manifold. With initialization $\hat{f}_0=0$, we define
\begin{equation}\label{eq:generalized T-kernel SGD}
\begin{aligned}
\hat{f}_n:=&P_{\WW}\left(\hat{f}_{n-1}-\gamma_n\partial_u\ell(\hat{f}_{n-1}\circ F(X_n),Y_n)K_{L_n}^T(F(X_n),\cdot)\right).
\end{aligned}
\end{equation}
We adopt the same hyperparameter settings as in the original T-kernel SGD, namely $L_n=\min\left\{k\,\Big\vert\,\dim{\Pi_k^{d}}\geq n^\theta\right\}$ with $\theta >0$ and $\gamma_n=\gamma_0n^{-t}$ for $t\in\left[\frac12,1\right)$. Similarly, an unprojected version of the algorithm can be defined as
\begin{equation}\label{eq:generalized T-kernel SGD unprojection}
\begin{aligned}
\hat{f}_n:=\hat{f}_{n-1}-\gamma_n\partial_u\ell(\hat{f}_{n-1}\circ F(X_n),Y_n)K_{L_n}^T(F(X_n),\cdot).
\end{aligned}
\end{equation}
In addition, we employ the $\alpha$-suffix averaging scheme $\bar{f}_{\alpha n} := \frac{1}{\alpha n} \sum_{i=(1-\alpha)n}^{n-1} \hat{f}_i$ as the output. In \autoref{subsec:Truncated Kernel Stochastic Gradient Descent with General Basis} of the Appendix, we illustrate how T-kernel SGD can be generalized from kernels induced by SBFs to kernels induced by general orthonormal bases, demonstrating the broader applicability of the algorithmic framework.

\section{Theoretical Results}\label{section:mean result}

This section focuses on establishing the optimal generalization guarantees of the T-kernel SGD algorithm. Our analysis builds on the concepts introduced at the end of \autoref{subsection:T-kernel SGD}, including the mapping $F$ and the unknown distribution $\rho$. We first introduce a sequence of spaces $\{\WW^{r}(\SS^{d-1})\}_{r>0}$, indexed by the regularity parameter $r$, to characterize the regularity of the minimizer $f^*$. Define 
\begin{equation}\label{eq:Sobolev space} 
\begin{aligned} 
\WW^{r}\left(\SS^{d-1}\right) = \left\{ f=\sum_{k=0}^{\infty} \sum_{1\leq j\leq\dim\HH_k^{d}} f_{k,j}Y_{k,j} \,\Bigg|\, \Vert f\Vert_{\WW^{r}}^2 := \sum_{k=0}^{\infty} \sum_{1\leq j\leq\dim\HH_k^{d}} \frac{(f_{k,j})^2}{a_k^{2r}} <\infty \right\}. 
\end{aligned} 
\end{equation} 
Equation \eqref{eq:Sobolev space} and the condition $0<a_k\leq1$ imply that $\WW^{r_1}\left(\SS^{d-1}\right)\subseteq\WW^{r_2}\left(\SS^{d-1}\right)$ for $r_1\geq r_2>0$. The identity $\WW^{1/2}(\SS^{d-1})=\HH_K$ also gives $\WW^r(\SS^{d-1})\subseteq\HH_K$ for $r\geq\frac12$. Thus, $r$ characterizes function regularity relative to the native RKHS $\HH_K$. Such sequences of spaces are standard for characterizing minimizer regularity in nonparametric regression \cite{smale2006online,ying2008online,tarres2014online, dieuleveut2016nonparametric,guo2024optimality}. To give a concrete interpretation of the regularity parameter $r$, we relate $\WW^r(\SS^{d-1})$ to classical spherical Sobolev spaces. Let $\theta_{i,j}$ denote the polar angle in the $(x_i,x_j)$-plane, and let $D_{i,j}:=\partial/\partial\theta_{i,j}$ denote the corresponding angular derivative operator. Following \cite{dai2013approximation}, for $m\in\NN_+$, the spherical Sobolev space $\WW_2^m(\SS^{d-1})$ is defined as 
\[ \WW_2^m\left(\SS^{d-1}\right) = \left\{ f\in\LLd \,\Bigg|\, \Vert f\Vert_{\WW_2^m(\SS^{d-1})} := \Vert f\Vert_{\omega} + \sum_{1\leq i<j\leq d} \Vert D_{i,j}^m f\Vert_{\omega} <\infty \right\}. \] 
By Theorem~4.7.2 of \cite{dai2013approximation}, together with the discussion in \cite{lin2021distributed}, if $m:=2sr(d-1)\in\NN_+$, then $\WW^r(\SS^{d-1})$ and $\WW_2^m(\SS^{d-1})$ coincide with equivalent norms. For $r=\frac12$, this equivalence holds between $\HH_K$ and $\WW_2^{s(d-1)}(\SS^{d-1})$ whenever $s(d-1)\in\NN_+$. Under this correspondence, $r=\frac12$ represents the native RKHS regularity, whereas $r>\frac12$, equivalently $m>s(d-1)$, corresponds to higher-order weak differentiability beyond the native Sobolev order $s(d-1)$. Such regularity may arise naturally in geophysical and atmospheric applications, where $f^*$ represents a spherical field modeled by a partial differential equation, such as the Laplace equation, a diffusion equation, or an equation of fluid dynamics.

\begin{assumption}\label{assum:independnt}
The samples $\{(X_i,Y_i)\}_{i\in\NN+}\subset \Omega\times\YY$ are independently and identically distributed (i.i.d.) according to the Borel probability distribution $\rho$.
\end{assumption}
\begin{assumption}\label{assum:regularity condition}
{\rm (Regularity condition $r\geq\frac{1}{2}$)} The minimizer $f^*$, defined as
$$f^*:=\arg\min_{f\in \WW}\ee(f)=\arg\min_{f\in \WW}\EE_\rho\left[\ell\left(f\circ F\left(X\right),Y\right)\right],$$
satisfies $f^*\in\WW^{r}\left(\SS^{d-1}\right)$.  Moreover, $f^*$ fulfills one of the following conditions:
\begin{itemize}
    \item[\rm{(a).}] $f^*$ lies in the interior of $\WW$, i.e., $\Vert f^*\Vert_K < Q$.
    \item[\rm{(b).}] There exists a constant $L > 0$ such that, for every $f \in \WW$,
\begin{equation}\label{eq:L-smooth}
\ee(f)-\ee(f^*)\leq \frac{L}{2}\left\Vert f\circ F-f^*\circ F\right\Vert_{\rho_X}^2.
\end{equation}
\end{itemize}
\end{assumption}
Let $\rho_F:=F_\#\rho_X$ denote the distribution of $F(X)$ on $\SS^{d-1}$.
\begin{assumption}\label{assum:regularity condition new}
{\rm (Regularity condition $0<r<\frac{1}{2}$)} The minimizer $f^*$, defined as
$$f^*:=\arg\min_{f\in \LL_{\rho_F}^2\left(\SS^{d-1}\right)}\ee(f)=\arg\min_{f\in \LL_{\rho_F}^2\left(\SS^{d-1}\right)}\EE_\rho\left[\ell\left(f\circ F\left(X\right),Y\right)\right],$$
 satisfies $f^*\in\WW^{r}\left(\SS^{d-1}\right)$ for some $0<r<\frac{1}{2}$. Moreover, there exists a constant $L > 0$ such that, for every $f \in \WW^{r}\left(\SS^{d-1}\right)$,
\begin{equation}
\ee(f)-\ee(f^*)\leq \frac{L}{2}\left\Vert f\circ F-f^*\circ F\right\Vert_{\rho_X}^2.
\end{equation}
\end{assumption}
\begin{assumption}\label{assum:Radon-Nikodym derivative}
The pushforward distribution $\rho_F$ is absolutely continuous with respect to the surface measure $\omega$ on $\mathbb{S}^{d-1}$. Moreover, there exist constants $0<b_\rho<B_\rho<\infty$ such that
\begin{equation}\label{eq:bounded of Radon–Nikodym derivative}
b_\rho \leq \frac{d\rho_F}{d\omega}(x) \leq B_\rho, \qquad \omega\text{-a.e. }x\in\mathbb{S}^{d-1}.
\end{equation}
\end{assumption}
For locally strongly convex and smooth losses, the algorithm \eqref{eq:generalized T-kernel SGD} achieves rates faster than $\mathcal{O}\left(n^{-\frac{1}{2}}\right)$ whenever \autoref{assum:regularity condition} holds with $r>\frac12$. For merely convex losses, if the minimizer satisfies the weaker regularity condition in \autoref{assum:regularity condition new}, then the algorithm \eqref{eq:generalized T-kernel SGD unprojection} attains the optimal slow rate. In the finite-dimensional setting, condition (b) of \autoref{assum:regularity condition} is a special case of the descent lemma for $L$-smooth functions \cite{nesterov2013introductory,beck2017first}. By analogy, in our analysis we combine condition (a) of \autoref{assum:regularity condition} with the $L$-smoothness property and, invoking \autoref{lemma:A.3}, establish the inequality stated in (b). Therefore, we do not distinguish between the Lipschitz constant $L$ in \autoref{assum:Lsmooth} and the constant $L$ in (b) of \autoref{assum:regularity condition}.

Compared with the assumptions on the unknown distribution $\rho$ in previous work on nonparametric regression \cite{smale2007learning,caponnetto2007optimal,dieuleveut2016nonparametric,guo2024optimality}, \autoref{assum:Radon-Nikodym derivative} is more direct. In particular, \autoref{assum:Radon-Nikodym derivative} plays a key role in establishing the equivalence between $\Vert f\circ F\Vert_{\rho_X}$ and $\Vert f\Vert_\omega$. As shown in \autoref{lemma:A.7},
\begin{equation}\label{eq:equivalence norm}
b_\rho\Omega_{d-1}\Vert f\Vert_\omega^2 \le \Vert f\circ F\Vert_{\rho_X}^2 \le B_\rho\Omega_{d-1}\Vert f\Vert_\omega^2,\quad \forall f\in\HH_K.
\end{equation}
This inequality is crucial for deriving one of the central analysis tools---the asymptotic equivalence between the RKHS norm $\Vert\cdot\Vert_K$ and the distribution-dependent norm $\Vert\cdot\Vert_{\rho_X}$.

\subsection{Optimal Rates for Excess Risk}\label{subsection:Optimal Rates for Excess Risk}

Our first main result establishes rate-optimal convergence guarantees for the expected excess risk, $\EE\left[\ee(\hat{f}_n) - \ee(f^*)\right]$, where $\hat{f}_n$ denotes the T-kernel SGD estimator.
\begin{theorem}\label{theorem:mean result}
Assume \autoref{assum:independnt} and \autoref{assum:Radon-Nikodym derivative} (with  $0<b_\rho<B_\rho$ in \eqref{eq:equivalence norm}) hold. 
\begin{itemize}
    \item[\rm{(a).}]\rm{The case $r\geq\frac{1}{2}$:} Suppose that \autoref{assum:extence} (with $M>0$), \autoref{assum:Lsmooth} (with $L>0$), \autoref{assum:mustorngconvex} (with $\mu>0$), and \autoref{assum:regularity condition} (with $r \ge 1/2$) hold. Let $\theta=\frac{1}{2s(2r+1)}$ and choose the step size $\gamma_n = \gamma_0 n^{-\frac{2r}{2r+1}} \log(n+1)$ with $\gamma_0 = c\frac{A_14(2d)^{2s}}{A^2_2b_{\rho}\mu\Omega_{d-1}}$ for some constant $c \in \left[\frac{1}{\log 2}, \frac{2}{\log 3}\right]$. Then, for any $\alpha \in (0,1)$, the following bounds hold:
        \begin{equation*}
\begin{aligned}
\EE\left[\ee\left(\hat{f}_n\right)-\ee\left(f^*\right)\right]&\leq \mathcal{O}\left(n^{-\frac{2r}{2r+1}}\left(\log(n+1)\right)^2\right),\\
\EE\left[\ee\left(\bar{f}_{\alpha n}\right)-\ee\left(f^*\right)\right]&\leq \O\left(n^{-\frac{2r}{2r+1}}\log(n+1)\right),
\end{aligned}
\end{equation*}
where $\hat{f}_n$ denotes the last iterate in \eqref{eq:generalized T-kernel SGD} and $\bar{f}_{\alpha n}$ is the $\alpha$-suffix average.
\item[\rm{(b).}]\rm{The case $0<r<\frac{1}{2}$:} Assume that \autoref{assum:convex} and \autoref{assum:regularity condition new} (with $0<r<\frac12$) hold. Let $\theta=\frac{1}{2s(2r+1)}$ and $\gamma_n = \gamma_0 n^{-\frac{2r}{2r+1}}$ with $\gamma_0 \in(0,1)$. Let $\bar{f}_{\alpha n}$ denote the $\alpha$-suffix average in \eqref{eq:generalized T-kernel SGD unprojection}. Then we have
        \begin{equation*}
\begin{aligned}
\EE\left[\ee\left(\bar{f}_{\alpha n}\right)-\ee\left(f^*\right)\right]&\leq \O\left(n^{-\frac{2r}{2r+1}}\right).
\end{aligned}
\end{equation*}
\end{itemize}
Here, $0<A_2 \leq1\leq A_1$ denote the lower and upper bounds of $a_k \cdot \left(\dim \Pi_k^d\right)^{2s}$, that is,
$A_2\left(\dim\Pi_k^d\right)^{-2s}\leq a_k\leq A_1 \left(\dim\Pi_k^d\right)^{-2s},\ \forall k\in\NN.$
\end{theorem}

In online nonparametric regression, existing minimax-optimality results are largely confined to the least-squares loss, while general losses remain much less studied. Classical kernel SGD typically suffers from the saturation phenomenon, where the convergence rate ceases to improve once the regularity of the minimizer $f^*$ exceeds a certain threshold. For unregularized kernel SGD, \cite{ying2008online} established convergence rates of $\mathcal{O}\left(n^{-\frac{2r}{2r+1}}\log n\right)$ for the regularity parameter $r\in(0,\frac{1}{2}]$, while \cite{guo2019fast} obtained optimal rates $\mathcal{O}\left(n^{-\frac{2r}{2r+1}}\right)$ using the capacity parameter $s$, valid for $r\in\left[\frac{1}{2},1-\frac{1}{4s}\right]$. By employing Polyak averaging, \cite{dieuleveut2016nonparametric} enhanced the robustness of the estimator and established optimal rates $\mathcal{O}\left(n^{-\frac{4sr}{4sr+1}}\right)$, which depend on the capacity parameter $s$, for $r\in\left[\frac{1}{2}-\frac{1}{4s},1-\frac{1}{4s}\right]$. Incorporating an additional regularization scheme into kernel SGD helps alleviate saturation. \cite{tarres2014online} analyzed regularized kernel SGD and obtained the optimal rates $\mathcal{O}\left(n^{-\frac{2r}{2r+1}}\left(\log\frac{2}{\alpha}\right)^4\right)$ with probability at least $1-\alpha$ for $r\in\left[\frac{1}{2},1\right]$. In contrast to previous results, which exhibit saturation when the regularity parameter $r>1$, our algorithm, when specialized to the least-squares case, effectively overcomes this phenomenon. For general losses, however, the nonlinear structure of the Hessian introduces substantial challenges in analyzing the convergence. In online learning, classical SGD analysis yields only the slow rate $\mathcal{O}(n^{-\frac{1}{2}})$, corresponding to saturation at $r=\frac{1}{2}$. Leveraging stronger regularity conditions ($r>\frac{1}{2}$) to accelerate kernel SGD has remained an open problem. \autoref{theorem:mean result} shows that, even when $f^*$ possesses weak regularity ($r\in(0,\frac{1}{2}]$), the T-kernel SGD recovers the optimal rate established for the least-squares loss. Under stronger regularity assumptions, T-kernel SGD attains fast rates and, to the best of our knowledge, provides the first saturation-free guarantees for online learning with general losses.

In the analysis of general losses, research has typically focused on convex losses, while non-convex losses have received comparatively less attention. In this paper, we concentrate on a class of non-convex losses commonly used in robust regression. Many such losses, including the Cauchy and Welsch losses, are locally strongly convex and smooth for small residuals $\delta = Y-\hat{f}_n(X)$, which supports effective optimization. For large residuals, however, they grow slowly or even become bounded, thereby limiting the influence of outliers and yielding greater robustness than conventional convex losses. Consider a regression model $Y = f^*(X) + \epsilon$, where $\epsilon$ denotes noise. In \autoref{lemma:A.8}, we show that for many robust losses, including certain non-convex ones, the regression function  $f^*$ is a global minimizer. More generally, \cite{steinwart2007compare} showed that, under a suitable symmetry condition on $\rho$, the regression function is the unique minimizer of a strictly convex and symmetric loss. In practice, by choosing a sufficiently large radius $Q$ or appropriately rescaling the output $Y$, one can ensure that $f^*$ lies in the set $\WW$, and hence that the algorithm converges to the global minimizer of the population risk.

In T-kernel SGD, the choice of the size of the hypothesis space $\HH_{L_n}$ is crucial for achieving optimal rates. When the minimizer $f^*$ is smoother, that is, when the regularity parameter $r$ is larger, a smaller $\theta$ helps reduce variance; when $f^*$ is less smooth, a larger $\theta$ is preferred to control bias. In the analysis of \autoref{theorem:mean result}, we set $\theta = \frac{1}{2s(2r+1)}$, which effectively balances bias and variance and yields the optimal convergence rate. In contrast to T-kernel SGD, classical kernel SGD uses different regularization mechanisms, such as approximating the regularization path or tuning the step size, which influence the complexity of the hypothesis space only indirectly and to a limited extent. As a result, when the minimizer $f^*$ has regularity $r>1$, these methods exhibit saturation and cannot fully exploit the additional smoothness. Moreover, the finite-dimensional structure of $\HH_{L_n}$ is essential for the convergence analysis. Building on the norm equivalence between $\Vert\cdot\Vert_{\rho_X}$ and $\Vert\cdot\Vert_\omega$ shown in \eqref{eq:equivalence norm}, we further establish the asymptotic equivalence between $\Vert\cdot\Vert_{\rho_X}$ and $\Vert\cdot\Vert_K$ (see \autoref{lemma:B.5}),
$$\frac{A^2_2}{A_1}\frac{b_{\rho}\Omega_{d-1}}{(2d)^{2s}}\,n^{-2\theta s}\Vert f\Vert_K^2\leq\Vert f\circ F\Vert_{\rho_X}^2\leq \kappa^2 \Vert f\Vert_K^2,\ \ \forall f\in\HH_{L_n}.$$
The asymptotic equivalence above gives an inequality-based characterization of the covariance operator $L_{\omega,K}$ (see \eqref{eq:covariance operator}), reflecting the decay of its eigenvalues. By combining optimization techniques with this inequality and the inequality-based characterization of the regularity of the minimizer $f^*$ in \autoref{lemma:B.6}, we prove \autoref{theorem:mean result} in \autoref{proof of theorem 1}. Applying the local strong convexity of losses, we then establish the following result in \autoref{proof of Proposition 1}.
\begin{proposition}\label{proposition:least-square result}
Suppose that the conditions in part (a) of \autoref{theorem:mean result} hold, and consider the same hyperparameters {\rm{($\theta$, $\gamma_n$)}} as in \autoref{theorem:mean result}. For $r \ge \frac12$, we have
\begin{equation*}
\begin{aligned}
\EE\left[\left\Vert\hat{f}_n\circ F-f^*\circ F\right\Vert_{\rho_X}^2\right]&\leq \O\left(n^{-\frac{2r}{2r+1}}\left(\log(n+1)\right)^2\right)\\
\EE\left[\left\Vert\bar{f}_{\alpha n}\circ F-f^*\circ F\right\Vert_{\rho_X}^2\right]&\leq \O\left(n^{-\frac{2r}{2r+1}}\log(n+1)\right).
\end{aligned}
\end{equation*}
\end{proposition}
In \autoref{proposition:least-square result}, for $r\geq\frac{1}{2}$, we show that convergence of the excess risk is equivalent to convergence in the $\Vert \cdot \Vert_{\rho_X}$ norm. Compared with the convergence in the RKHS discussed in the next subsection, this result can be interpreted as weak convergence.

The dimension $d$ affects both the algorithmic complexity and the convergence rate. From the perspective of spherical Sobolev spaces, if $f^*\in\WW_2^m\left(\SS^{d-1}\right)$ with $r=\frac{m}{2s(d-1)}$, the excess-risk rate in \autoref{theorem:mean result} can be equivalently written, up to logarithmic factors, as $\mathcal{O}\left(n^{-\frac{m}{m+s(d-1)}}\right)$. The threshold $r=\tfrac{1}{2}$ corresponds to $m=s(d-1)$; hence, this rate is faster than $\mathcal{O}\left(n^{-1/2}\right)$ precisely when $m>s(d-1)$. Applying the proofs of Lemma 2 and Lemma 4 in \cite{bai2025truncated}, we obtain $\dim\Pi_{L_n}^d \leq \left(1+\frac{d}{L_n}\right) n^{\theta}$ and $L_n \leq ((d-1)!\dim\Pi_{L_n}^d)^{1/(d-1)}$. Combining these bounds with the computational and storage complexity derived in \autoref{subsection:T-kernel SGD} shows that processing $n$ samples with T-kernel SGD requires $\O\left(d^2\left(1+\frac{d}{L_n}\right)^{\frac{d}{d-1}}n^{1+\frac{d}{d-1}\theta}\right)$ time and $\O\left(\left(1+\frac{d}{L_n}\right)n^{\theta}\right)$ memory. In the complexity analysis, both the computational and storage complexities grow rapidly with $d$. As discussed in \autoref{subsection:T-kernel SGD}, the dimension factor $d$ in the computational complexity may be partly an artifact of our analysis. Nevertheless, the MNIST experiment in \autoref{subsection:Binary Classification of High-Dimensional Non-Spherical Data} provides empirical evidence that T-kernel SGD remains computationally feasible for 784-dimensional inputs under the reported experimental settings.  For higher-dimensional settings, new strategies may be required to maintain computational efficiency. With $\theta=\frac{1}{2s(2r+1)}$ as chosen in \autoref{theorem:mean result}, the time and memory complexities are $\O\left(d^2\left(1+\frac{d}{L_n}\right)^{\frac{d}{d-1}}n^{1+\frac{d}{d-1}\frac{1}{2s(2r+1)}}\right)$   and $\O\left(\left(1+\frac{d}{L_n}\right)n^{\frac{1}{2s(2r+1)}}\right)$, respectively. For $r\geq\frac12$, these are significantly lower than the computational cost $\mathcal{O}(n^2)$ and the memory cost $\mathcal{O}(n)$ of classical kernel SGD.

Because computers store real numbers only with finite precision, additional errors may arise. To mitigate their effect on optimality, one may gradually increase the precision during the iteration. For example, using binary sequences of length $2\log_2(n)$ yields precision of order $\O(\frac{1}{n^2})$. Recently, \cite{zhang2022sieve} proposed a modified stochastic gradient descent algorithm that stores coefficients with precision increasing in the sample size $n$. This requires only an additional $\log(n)$ factor in storage and still achieves the theoretically optimal convergence rate. Therefore, with a simple modification of Algorithm \ref{alg:theoretical}, one can gradually increase the coefficient precision while incurring only an additional $\log(n)$ memory cost. Consequently, the storage complexity of the modified algorithm is $\O\left(\left(1+\frac{d}{L_n}\right)n^{\frac{1}{2s(2r+1)}}\log(n)\right)$. In practice, 64-bit double-precision arithmetic (as used in Python) is typically sufficient for T-kernel SGD, so we provide only a brief discussion here.

We now investigate the optimality of the storage complexity. The definitions and concepts used in the lower-bound analysis are adapted from Section 6.3 of \cite{zhang2022sieve}. We begin by introducing a description analogous to a probabilistic Turing machine to formally define the general estimator. An estimator can be viewed as a mapping $G_n$ from the sample space $\left(\Omega\times\YY\right)^n$ to $\WW$. Any estimator implementable on a computer necessarily involves an encoding–decoding procedure: the encoder $E_n$ maps the samples $\{(X_i,Y_i)\}_{1\leq i \leq n}$ to a binary sequence $b_n$, which is stored in memory, and the decoder $D_n$ translates the stored $b_n$ into the output function $\hat{f}_n$. In general, as the sample size increases, the estimator yields more accurate outputs, resulting in a longer binary sequence $b_n$. This motivates the following definition of a general estimator.
\begin{definition}
For $l_n \in \NN_+$, we define an $l_n$-sized estimator $G_n = D_n \circ E_n : (\Omega \times \YY)^n \to \WW$, that is, the composition of the encoder $E_n$ and the decoder $D_n$.
\begin{itemize}
    \item[\rm{(a).}] For $n \in \NN_+$, one may consider an encoding map $E_n : (\Omega \times \YY)^n \to \{0,1\}^{l_n}$, which can be randomized or deterministic.
    \item[\rm{(b).}]  The decoder $D_n : \{0,1\}^{l_n} \to \WW$ is a known, deterministic map that maps a binary sequence of length $l_n$ to a function in $\WW$.
\end{itemize}
\end{definition}
By combining the above definitions, one can derive a lower bound on the storage complexity.
\begin{lemma}\label{theorem:optimal storage}
Consider a positive integer sequence $\{l_n\}$ such that $l_n = o\left(n^{\frac{1}{2s(2r+1)}}\right)$ with $s>\frac{1}{2},\ r\geq\frac{1}{2}$, and let $G(l_n)$ denote the collection of all $l_n$-sized estimators. Then we have
$$ \lim_{n\to\infty}\inf_{G_n\in G(l_n)}\sup_{f^*\in\WW\cap \WW^{r}\left(\SS^{d-1}\right)}\EE\left[n^{\frac{2r}{2r+1}}\Vert G_n\left(\{(X_i,Y_i)\}_{1\leq i \leq n}\right)-f^*\Vert_{\omega}^2\right]=\infty.$$
\end{lemma}
The proof of \autoref{theorem:optimal storage} is provided in \autoref{subsection:proof of theorem:optimal storage}. \autoref{theorem:optimal storage} implies that no estimator can achieve the optimal convergence rate while using memory of order $o\left(n^{\frac{1}{2s(2r+1)}}\right)$; that is, the required storage is at least of order $n^{\frac{1}{2s(2r+1)}}$. Consequently, after accounting for the errors introduced by finite-precision memory, T-kernel SGD attains the optimal storage complexity up to a logarithmic factor.

\subsection{Optimal Rates for Strong Convergence}

Our second main result concerns convergence in the RKHS, often referred to as strong convergence, and is stated below.

\begin{theorem}\label{theorem:strong convergence}
Suppose that the assumptions in part (a) of \autoref{theorem:mean result} hold, and consider the same hyperparameters {\rm{($\theta$, $\gamma_n$)}} as in \autoref{theorem:mean result}. For $r \ge \frac12$, we have
\begin{align*}
&\EE\left[\left\Vert\hat{f}_n-f^*\right\Vert_{K}^2\right]\leq&\left(2Q^2+3A_1^{2r-1}\Vert f^*\Vert_{\WW^r}^2\right)(n+1)^{-\frac{2r-1}{2r+1}}+{P'}^2(\log(n+1))^2(n+1)^{-\frac{2r-1}{2r+1}},
\end{align*}
where ${P'}^2$ is a constant given by 
$${P'}^2=(4r+2)\gamma_0^2\left[\left(\left(\frac{\mu}{2}+\frac{8L^2}{\mu}\right)\frac{L}{\mu}+L\right)B_\rho\Omega_{d-1} A_1^{2r}\Vert f^*\Vert_{\WW^r}^2\frac{1}{\gamma_0\log(2)}+M^2\kappa^2\right].$$
\end{theorem}
Many spherical data sets in geophysics, meteorology, and climate science arise from complex physical systems and are often governed, at least implicitly, by underlying partial differential equations. In such settings, it is not sufficient to merely fit the observed data well; one also seeks predictions that are physically consistent \cite{Karniadakis2021Physics}. Such consistency is often reflected in local differential relations and global balance laws, including conservation and flux constraints \cite{hansen2023learning}. Because these structures depend explicitly on derivatives of the solution, accurate recovery of the minimizer and its derivatives is essential for producing physically meaningful predictions. However, convergence in excess risk alone does not, in general, guarantee convergence of higher-order derivatives. Suppose that $f^*\in\WW_2^m\left(\SS^{d-1}\right)$ with $m\in\NN_+$, and let $r=\frac{m}{2s(d-1)}$. When $s(d-1)\in\NN_+$, \autoref{theorem:strong convergence} yields  $\EE\left[ \left\Vert \hat{f}_n-f^* \right\Vert_ {\WW_2^{s(d-1)}\left(\SS^{d-1}\right)}^2 \right] = \O\left( n^{-\frac{m-s(d-1)}{m+s(d-1)}} \right)$ up to logarithmic factors. This result provides theoretical support for recovering derivative-based physical structures, even when the underlying governing law is not explicitly known. From the perspective of spherical Sobolev spaces, $\HH_K$ corresponds to $\WW_2^{s(d-1)}\left(\SS^{d-1}\right)$. To obtain a nontrivial strong convergence rate, we require $f^*$ to possess additional regularity, namely $r>\frac12$, or equivalently $m>s(d-1)$.

 Previous work has established strong convergence in various settings, including least-squares regression \cite{ying2008online,tarres2014online,guo2019fast} and robust regression \cite{guo2024optimality}. The above analyses are based on the classical kernel SGD algorithm, which requires handling all sample pairs $\{(X_i,X_j)\}_{1\le i<j\le n}$, leading to computational complexity $\mathcal{O}(n^2)$ and memory $\mathcal{O}(n)$. Such excessive costs severely limit its applicability to large-scale problems. Moreover, existing large-scale kernel methods \cite{rudi2018fast,abedsoltan2023toward} have focused primarily on convergence in excess risk, leaving the development of efficient algorithms that achieve optimal strong convergence rates largely unexplored. In contrast, our work establishes T-kernel SGD, which is efficient in terms of both computation and memory, and achieves capacity-independent optimal rates (see, e.g., \cite{blanchard2018optimal}) for strong convergence up to logarithmic factors.

Finally, we discuss how the local strong convexity and smoothness of the loss $\ell$ affect the convergence rate. In \autoref{theorem:mean result}, these two properties are essential for obtaining rates faster than $\O\left(\frac{1}{\sqrt{n}}\right)$. Under mere convexity, the best rate we obtain is $\O\left(\frac{1}{\sqrt{n}}\right)$. As follows from the definition of $P'^2$ in \autoref{theorem:strong convergence} and the constant in \eqref{eq:last-step equation} of \autoref{theorem:mean result}, the strong convexity parameter $\mu$ and the smoothness parameter $L$ enter the error only through the constant, specifically in the form $\left(C_1'+C_2'\left(\frac{L}{\mu}\right)^2\right)L$, where $C_1'$ and $C_2'$ are constants. Thus, a larger ratio $\frac{L}{\mu}$ leads to a larger constant and may slow convergence. It is worth emphasizing that both $\mu$ and $L$ are local constants defined over the part of the closed convex set $\WW$ relevant to the loss function. Consequently, the choice of the radius parameter $Q$ for $\WW$ may indirectly affect these constants. In particular, if $\WW$ is taken too large, then the loss domain $[-B,B]$ also expands, which may increase the ratio $\frac{L}{\mu}$ and hence worsen the constant in the error.

\subsection{Robustness to Hyperparameter Perturbation}\label{subsection:3.3}
In the analysis of \autoref{section:mean result}, the choice of the truncation parameter $\theta$ and the step size $\gamma_n$ depends on the regularity parameter $r$. For many spherical data sets arising from physical processes or governed by PDEs, the smoothness of $f^*$ can be inferred from the underlying physical laws. However, in some cases, the exact smoothness of $f^*$ is difficult to obtain. In this subsection, we first review existing methods for estimating the regularity parameter $r$ from the literature, and discuss the effect of hyperparameter perturbation on the convergence rate. We then provide general guidelines for setting hyperparameters when the regularity parameter is unknown. Finally, we propose an algorithm for adaptively selecting the regularity parameter.

Indeed, since $\WW^{r_1}\left(\SS^{d-1}\right)\subset \WW^{r_2}\left(\SS^{d-1}\right)$ for all $r_1\ge r_2$, if $f^* \in \WW^{r_1}(\SS^{d-1})$, then necessarily $f^* \in \WW^{r_2}(\SS^{d-1})$. Therefore, in order to achieve faster convergence rates in the algorithm, one would ideally select the largest possible regularity parameter $r$. When the regularity parameter $r$ is unknown, \cite{li2026local} propose a method for estimating the regularity of the function. Let $\eta_{f^*}(r) = \Vert f^* \Vert_{\WW^{r}(\SS^{d-1})}$. Intuitively, $\eta_{f^*}(r) < \infty$ if $f^* \in \WW^{r}(\SS^{d-1})$, while $\eta_{f^*}(r) = \infty$ if $f^* \notin \WW^{r}(\SS^{d-1})$. This implies that the Sobolev norm of $f^*$ may exhibit a transition point as $r$ increases, at which the norm changes from finite to infinite. The key idea in the literature is to consider the interpolant $f_n$ of $f^*$ and analyze its behavior as $r$ increases. If $f^* \in \WW^r(\SS^{d-1})$, then $\eta_{f_n}(r)$ changes smoothly; however, if $r$ exceeds a critical point such that $f^* \notin \WW^r(\SS^{d-1})$, then $\eta_{f_n}(r)$ increases sharply. This transition corner is used to estimate the regularity parameter $r$, and the so-called L-curve corner is defined via the curvature of the log-Sobolev norm curve
$$r^* = \arg\max_r 
\frac{\left| \big(\log \eta_{f_n}(r)\big)'' \right|}{\left( 1 + \big( (\log \eta_{f_n}(r))' \big)^2 \right)^{3/2}} \,.$$
In our framework, because the algorithm achieves strong convergence in RKHS, and such convergence remains valid under certain regularity parameter mismatch scenarios, the estimator output $\hat{f}_n$ can be used in place of the interpolant as an estimator of $f^*$.

In the analysis of \autoref{section:mean result}, the hyperparameters, including $\theta$ and the step size $\gamma_n$, are determined by the regularity parameter $r$. We now investigate the effect of a mismatch between the chosen regularity parameter and the true regularity of $f^*$ on the convergence rate. We focus on the case $r>\frac{1}{2}$, as the case $0<r\leq\frac{1}{2}$ can be handled analogously. Suppose that $f^*\in\WW^{r}(\SS^{d-1})$, but that a mismatched regularity parameter $r_1>\frac12$ with $r_1\neq r$ is used, with hyperparameters set as $\gamma_n=\gamma_0n^{-\frac{2r_1}{2r_1+1}}\log(n+1)$ and $\theta=\frac{1}{2s(2r_1+1)}$. When $r> r_1$, it follows from $\WW^{r}\left(\SS^{d-1}\right)\subset \WW^{r_1}\left(\SS^{d-1}\right)$ and the analysis in \autoref{section:mean result} that the algorithm naturally achieves an excess risk rate of $\mathcal{O}\left(n^{-\frac{2r_1}{2r_1+1}}\right)$ and a strong convergence rate of $\mathcal{O}\left(n^{-\frac{2r_1-1}{2r_1+1}}\right)$ (up to a logarithmic factor). When $r<r_1$, our analysis in \autoref{subsec:Lemmas and Proofs for Robustness to Hyperparameter Perturbation} shows that the algorithm achieves an excess risk rate of $\mathcal{O}\left(n^{-\frac{2r}{2r_1+1}}\right)$ and a strong convergence rate of $\mathcal{O}\left(n^{-\frac{2r-1}{2r_1+1}}\log(n+1)\right)$. Therefore, even under hyperparameter mismatch, the algorithm still converges effectively. The analysis shows that, when $r_1$ is close to the true value $r$, the algorithm attains a near-optimal rate. In practice, when $r$ is unknown, one may estimate the regularity parameter using the method in \cite{li2026local}.

Furthermore, our convergence analysis in \autoref{subsec:Lemmas and Proofs for Robustness to Hyperparameter Perturbation} provides a general hyperparameter rule. Specifically, for $r\geq\frac{1}{2}$, if one chooses the step size $\gamma_n=\gamma_0n^{-\frac{1}{2}}\left(\log(n+1)\right)^{-1}$ and $\theta>\frac{1}{8sr}$, then the algorithm attains the rate $\mathcal{O}\left(n^{-\frac{1}{2}}\log(n+1)\right)$. In practice, one may choose the step size $\gamma_n=\gamma_0n^{-\frac{1}{2}}$ and set $\theta$ close to, or slightly above, $\frac{1}{4s}$. This hyperparameter rule is adopted in the experiments reported in \autoref{subsection:Robust Regression on the Circle}, where the empirical results suggest that the algorithm attains near-optimal rates. Therefore, we are inclined to believe that the strict hyperparameter restrictions in \autoref{theorem:mean result} may be artifacts of the proof, and that, in practice, the algorithm may attain the optimal rate over a range of choices of the step size $\gamma_n$ and $\theta$.

Next, we discuss the choice of the radius $Q$ in the closed convex set $$\WW=\left\{f\in\HH_K\,|\ \Vert f\Vert_K\leq Q\right\}.$$ For some common convex losses, such as the least-squares loss, logistic loss, and Huber loss, strong convexity and smoothness hold on any bounded closed set. In such cases, $Q$ may be chosen relatively large. On the other hand, if $Q$ is taken too large, the region involved in the analysis may become unnecessarily broad, which can in turn increase the ratio $L/\mu$ between the smoothness parameter $L$ and the strong convexity parameter $\mu$, and thereby increase the constant in the convergence bound. In regression problems, the goal is typically to learn a regression function $f^*$, while our theoretical analysis requires the projection radius $Q$ to be sufficiently large such that $\Vert f^*\Vert_K \leq Q$. In \cite{page2018ivanov}, the authors studied kernel-based least-squares regression with Ivanov regularization and proposed a validation-based selection of the radius of the RKHS ball over a finite uniformly spaced grid, with corresponding convergence guarantees. Motivated by this finite-grid validation strategy, we select the projection radius of T-kernel SGD from a geometric grid $\mathcal{Q}=Q_0\{\frac{1}{3},1,3,9,27,81\}$, where $Q_0>0$ is a reference scale specified below. Suppose that the data are split into a training set $\mathcal{D}_{\mathrm{tr}}=\{(X_i,Y_i)\}_{i=1}^{m}$ and a validation set $\mathcal{D}_{\mathrm{val}}=\{(\widetilde X_j,\widetilde Y_j)\}_{j=1}^{h}$. For each $Q\in\mathcal{Q}$, we run the same T-kernel SGD algorithm on $\mathcal{D}_{\mathrm{tr}}$, changing only the projection radius, and denote the resulting estimator by $\bar f_{m,Q}$. We then select $\widehat Q\in\operatorname*{arg\,min}_{Q\in\mathcal{Q}} \frac{1}{h}\sum_{j=1}^{h}\ell(\bar f_{m,Q}(\widetilde X_j),\widetilde Y_j)$. If $\widehat Q$ lies at the upper boundary of $\mathcal{Q}$, the grid may be enlarged by including larger radii. It remains to specify the reference scale $Q_0$. By the reproducing property, $|f(x)|\leq \kappa\|f\|_K$. Let $s_Y:=\left(m^{-1}\sum_{i=1}^{m}Y_i^2\right)^{1/2}$. A natural choice of the reference scale is $Q_0:=s_Y/\kappa$. Indeed, $\kappa Q_0=s_Y$, and hence this choice provides a natural reference for $Q$ based on $s_Y$.

Finally, we outline the design of an adaptive version of the algorithm. As an initial choice, we set the regularity parameter to $r=1$, and take the step size $\gamma_n=\gamma_0n^{-\frac{2}{3}}$ and $\theta=\frac{1}{6s}$. As the sample size increases, we then estimate the regularity of $f^*$ based on the log-curvature of $\eta_{\hat{f}_n}(r)$, and update the regularity parameter $r$ accordingly. The strong convergence guarantees established earlier under hyperparameter misspecification provide partial theoretical support for this estimation procedure. More specifically, after the $n_k$-th iteration of the algorithm, we compute the log-curvature of $\eta_{\hat{f}_{n_k}}(r)$ and take the value of $r$ corresponding to its maximizer as an estimate of the regularity of $f^*$. We then update the subsequent step-size sequence and truncation parameter $\theta$ according to the new value of $r$, where $n_k=\O(2^{k+10})$. This geometrically spaced inspection strategy preserves the order of the original computational complexity, while allowing the hyperparameters to be adjusted in a timely manner during the iterative procedure. Since $\hat{f}_n\in\HH_{L_n}$, we have $\left(\eta_{\hat{f}_n}(r)\right)^2=\sum_{k=0}^{L_n}\sum_{j=1}^{\dim\HH_k^d}a_k^{-2r}\left(f_{k,j}^{(n)}\right)^2$, and hence both $\eta_{\hat{f}_n}(r)$ and its log-curvature can be computed explicitly. Based on this representation, one may directly search over a prescribed grid for the maximizer of the log-curvature, thereby updating the hyperparameters at each inspection step.

To conclude this subsection, our analysis quantifies the effect of hyperparameter perturbations on the convergence of T-kernel SGD. When the chosen regularity parameter $r_1$ is close to the true regularity $r$, the corresponding hyperparameter choices yield near-optimal convergence rates. When $r\geq 1/2$ but its exact value is unknown, a convenient rule that does not require knowledge of $r$ is to choose $\gamma_n=\gamma_0 n^{-1/2}\bigl(\log(n+1)\bigr)^{-1}$ and $\theta>\frac{1}{4s}$, under which the algorithm attains the rate $\mathcal{O}\bigl(n^{-1/2}\log(n+1)\bigr)$. We further provide practical strategies for selecting the radius $Q$ and estimating the regularity parameter $r$, motivated by \cite{page2018ivanov} and \cite{li2026local}, respectively, and propose an adaptive version of T-kernel SGD that updates the estimate of $r$ during the iteration. Establishing convergence guarantees for the adaptive T-kernel SGD is left for future work.

\section{Numerical Experiments}\label{section:Simulation Study} 

In \autoref{subsection:Robust Regression on the Circle}, we conduct experiments on $\SS^1$ to compare T-kernel SGD with classical kernel SGD and the Nystr\'{o}m method. We also investigate the sensitivity of T-kernel SGD to hyperparameter choices. In \autoref{subsection:Robust Regression on 3-Dimensional Spherical Data}, we conduct experiments on $\SS^2$ to illustrate the theoretical results and compare T-kernel SGD with classical kernel SGD. In \autoref{subsection:Binary Classification of High-Dimensional Non-Spherical Data}, we further evaluate T-kernel SGD on a real-world high-dimensional dataset. Finally, in \autoref{subsec:Robust Regression on GRACE Satellite Data} and \autoref{subsec:Regression with Spherical Diffusion Balance}, we apply the algorithm to real GRACE satellite data and to a problem with latent physical constraints, respectively, to assess its performance in these settings\footnote{The code for the numerical experiments is available on \url{https://github.com/Researcher-Bai-1/T-kernel-SGD-with-general-losses}.}.  

\subsection{Robust Regression on the Circle}\label{subsection:Robust Regression on the Circle}

In this subsection, we validate the theoretical results presented in \autoref{section:mean result} by selecting optimal functions $f^{*}$ that satisfy different regularity conditions. In the experiments, we consider three classical loss functions commonly employed in robust regression: Cauchy, Huber, and Welsch losses. The experimental results show that T-kernel SGD overcomes the saturation effect, attains minimax rates faster than $\O(n^{-1/2})$, and requires less runtime than kernel SGD and the Nystr\"{o}m method to reach a given error level. Finally, we examine the sensitivity of T-kernel SGD to the hyperparameters \(r\), \(s\), \(\theta\), and \(Q\).

In this subsection, we consider the model $Y = f^{*}(X) + \epsilon,$ where $X$ is uniformly distributed on $\mathbb{S}^{1}$, and the noise term $\epsilon$ is also uniformly distributed.
Let $x = (\cos \xi, \sin\xi)$, $x' = (\cos \varphi, \sin \varphi) \in \mathbb{S}^{1}$, and consider the following kernel
for T-kernel SGD:
\begin{equation}
\begin{aligned}
K(x, x') &= K_{0}(x, x') + \sum_{k=1}^{\infty} \frac{1}{(2k)^{2s}} K_{k}(x, x') \overset{\text{(i)}}{=} 1 + \sum_{k=1}^{\infty} \frac{2}{(2k)^{2s}} \cos(k(\mathcal{\xi} - \varphi)) \\
&\overset{\text{(ii)}}{=} 1 + \frac{(-1)^{s+1}\pi^{2s}}{(2s)!} B_{2s}(\{\tfrac{\mathcal{\xi} - \varphi}{2\pi}\}),
\end{aligned}
\label{eq:circle_kernel}
\end{equation}
where $\{\xi\}$ denotes the fractional part of $\xi$, and $B_{2s}$ denotes the $2s$-th Bernoulli polynomial for
$s \in \mathbb{N}$. For the details of equations (i) and (ii), see \cite{dai2013approximation,dieuleveut2016nonparametric}. According to Section 1.6.1 of \cite{dai2013approximation}, $\dim \mathcal{H}_{k}^{2} = 2$ for $k \ge 1$. Consequently, the kernel $K_{k}(x, x')$ on the two-dimensional sphere can be written as $K_{k}(x, x') = Y_{k}^1(x)Y_{k}^1(x') + Y_{k}^2(x)Y_{k}^2(x'),$ and the orthonormal basis functions $Y_{k}^1$ and $Y_{k}^2$ admit simple explicit expressions, corresponding to the first- and second-kind Chebyshev polynomials, respectively. Therefore, each $\hat{f}_{n}$ can be explicitly represented as a truncated series
$\hat{f}_{n} = f_{0}Y_{0,1}+\sum_{k=1}^{L_{n}}\left( f_{k} Y_{k}^1 + f'_{k} Y_{k}^2\right),$
and, when combined with iteration \eqref{eq:T-kernel SGD expression}, only the coefficients of the truncated series need to be updated. Simultaneously, we choose $\mathcal{W}$ to be the closed unit ball of radius $Q = 1$. For T-kernel SGD, we consider three hyperparameter settings. The first corresponds to the theoretical choice satisfying the conditions of \autoref{theorem:mean result}, namely, the step size $\gamma_n=\gamma_0n^{-\frac{2r}{2r+1}}$ and $\theta=\frac{1}{2s(2r+1)}$. In the figure, the dark blue dashed and solid curves represent, respectively, the last-step error and the $\frac12$-suffix averaging error under this setting against the sample size or running time. The second and third settings follow the general hyperparameter choices proposed in \autoref{subsection:3.3}, with step size $\gamma_n=\gamma_0n^{-\frac12}$ and $\theta=1/4$ and $\theta=1/3$, respectively, both using $1/2$-suffix averaging as the output. These two settings are shown in the figure by light blue curves of different colors. For kernel SGD, we adopt a recursion similar to that in \cite{kivinen2004online,smale2006online,ying2008online,guo2019fast}, with the step size
$\gamma_{n} = \gamma_{0} n^{-t}$:
\[
g_{n} = g_{n-1} - \gamma_{n} \, \partial_{u}\ell(g_{n-1}(X_{n}), Y_{n}) K(X_{n}, \cdot).
\]
In the comparative experiments of kernel SGD, we consider three different kernels: the Bernoulli polynomial kernel $\frac{\pi^{2}}{4} B_{2}$ and two widely used universal kernels, namely the Gaussian kernel and the Mat\'{e}rn-$\tfrac{5}{2}$ kernel. 
For the Nystr\"{o}m method, we adopt the hyperparameter setting given in Theorem 7 of \cite{della2024nystrom} and the Mat\'{e}rn-$\tfrac{5}{2}$ kernel. When the regularity parameter is unknown, achieving the rate $\O\left(n^{-\frac12}\right)$ under that theorem requires $\O\left(n^{\frac{1}{2s}}\right)$ Nystr\"{o}m points and at least $\O\left(n^{1+\frac{1}{s}}\right)$ computational time. By contrast, under the hyperparameter setting given in \autoref{subsection:3.3}, T-kernel SGD requires only $\O\left(n^{1+\frac{1}{4s}}\right)$ time, where $s>\frac12$. Therefore, from the perspective of theoretical complexity, T-kernel SGD is more efficient than the Nystr\"{o}m method. Furthermore, the empirical results in \autoref{fig1:example1} and \autoref{fig2:example2} are consistent with this comparison and likewise indicate that T-kernel SGD achieves higher computational efficiency under the corresponding settings.
See \autoref{tab:example1} for the model setup.
\begin{table}[htbp]
\centering
\small
\setlength\tabcolsep{7mm}
\begin{tabular}{|c|c|c|}
\hline
\  &  Example 1 & Example 2\\ \hline
$s$ & 1 & 1\\ 
$r$ & $\frac{7}{4}$ & $\frac{3}{4}$ \\ 
optimal fitting $f^*$ & $\frac{1}{2}B_{4}\left(\frac{\xi}{2\pi}\right)$&$\frac{1}{5}B_{2}\left(\frac{\xi}{2\pi}\right)$\\
kernel SGD step size $\frac{\gamma_n}{\gamma_0}$ & $n^{-7/9}$ & $n^{-3/5}$ \\
noise $\epsilon$ & $U[-0.2,0.2]$ & $U[-0.2,0.2]$\\
Truncation level $L_n$ &$n^{\frac{1}{9}}\&n^{\frac{1}{4}}\&n^{\frac{1}{3}}$ & $n^{\frac{1}{5}}\&n^{\frac{1}{4}}\&n^{\frac{1}{3}}$\\
\hline
\end{tabular}
\caption{Examples}
\label{tab:example1}
\end{table}

\begin{figure}[htbp]
\centering
\includegraphics[width=7.5cm]{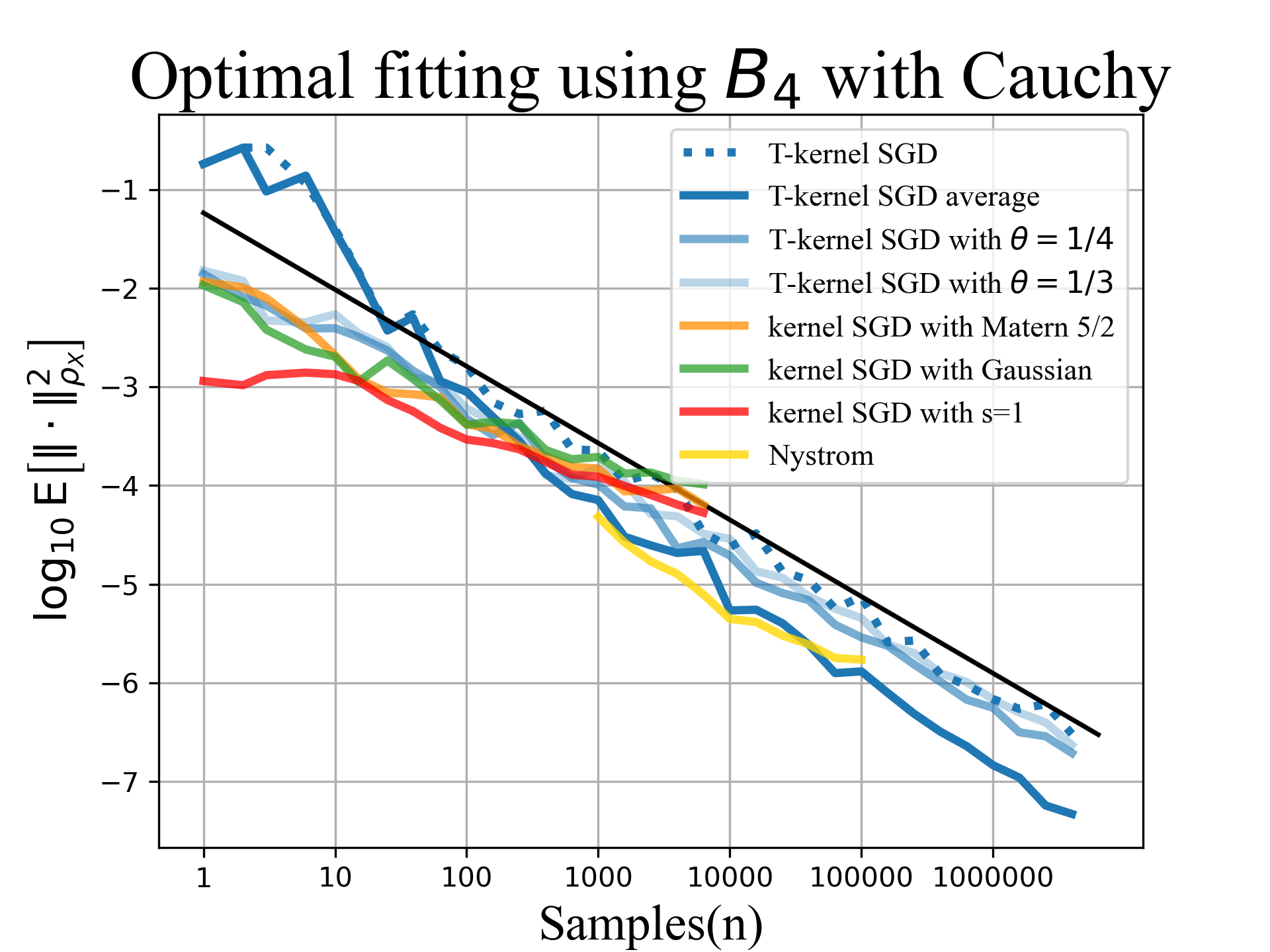}
\includegraphics[width=7.5cm]{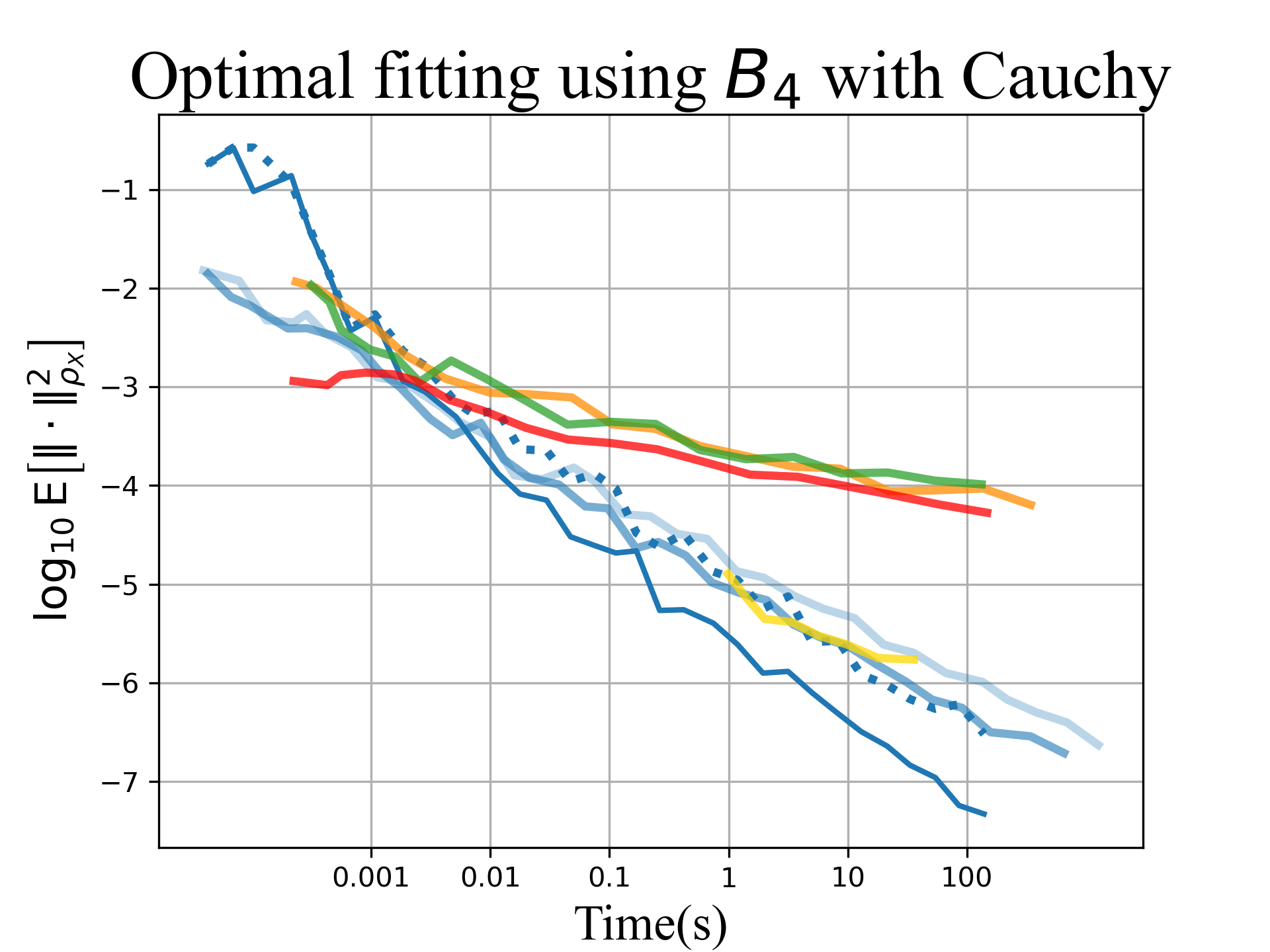}
\includegraphics[width=7.5cm]{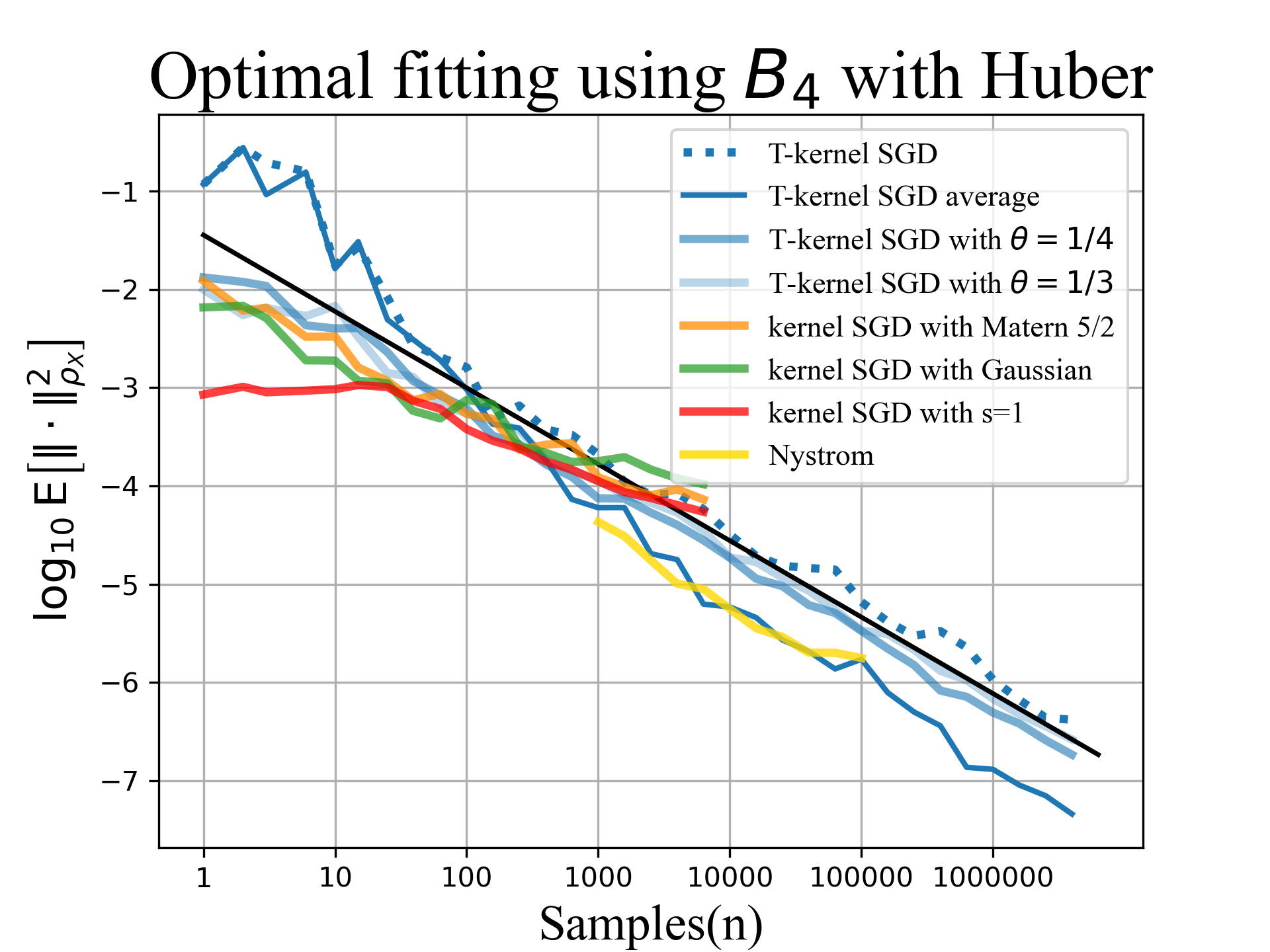}
\includegraphics[width=7.5cm]{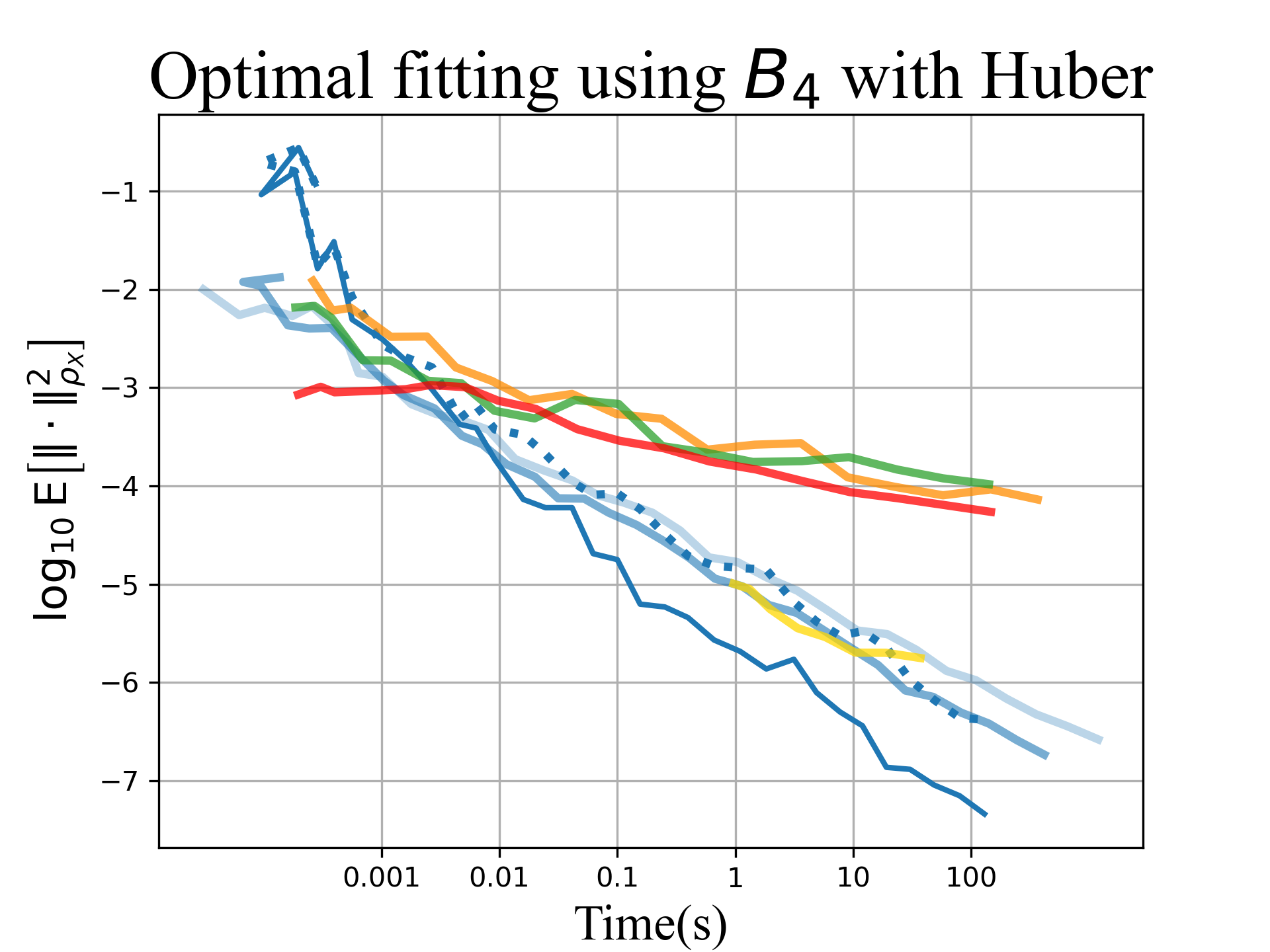}
\includegraphics[width=7.5cm]{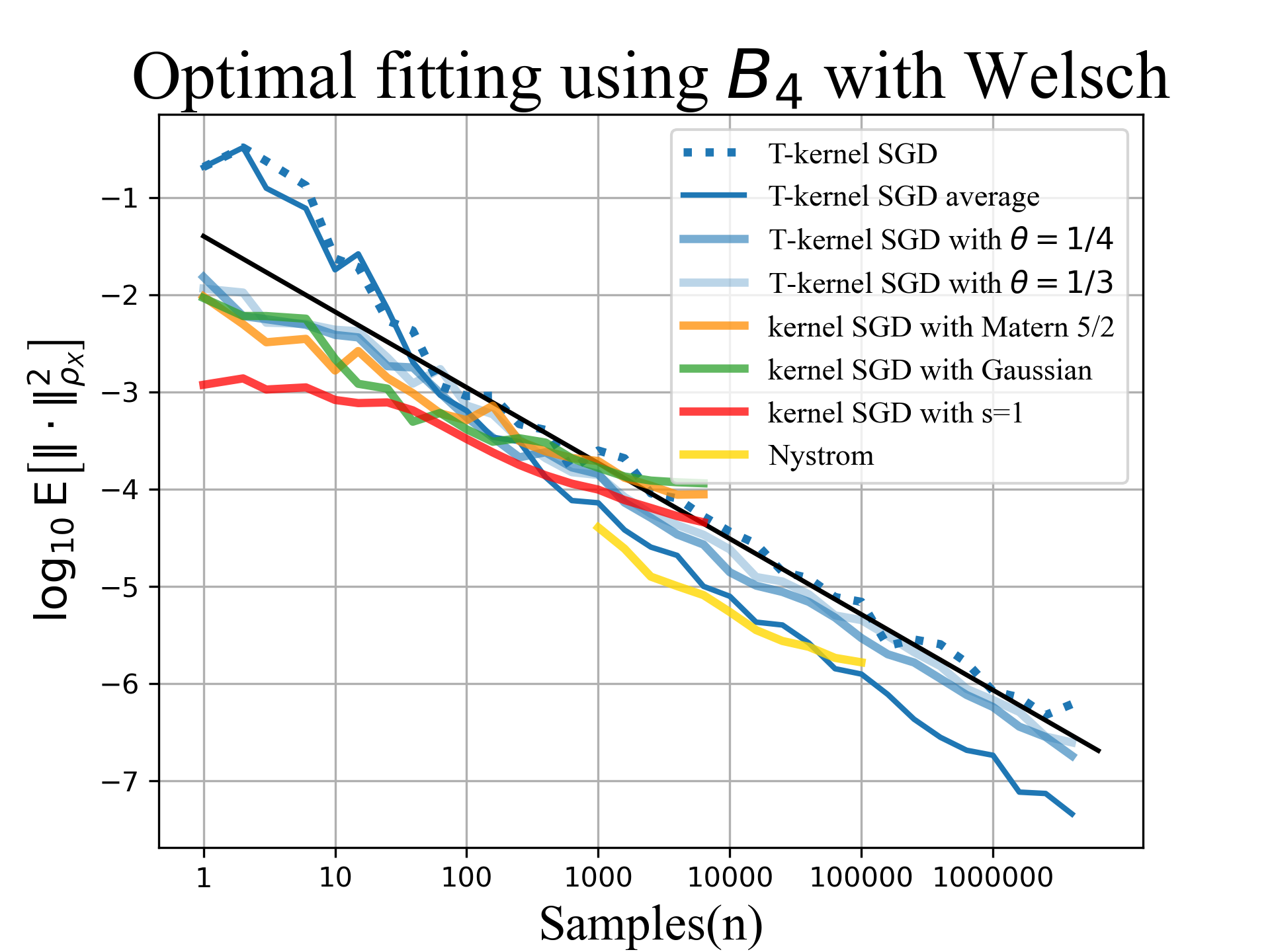}
\includegraphics[width=7.5cm]{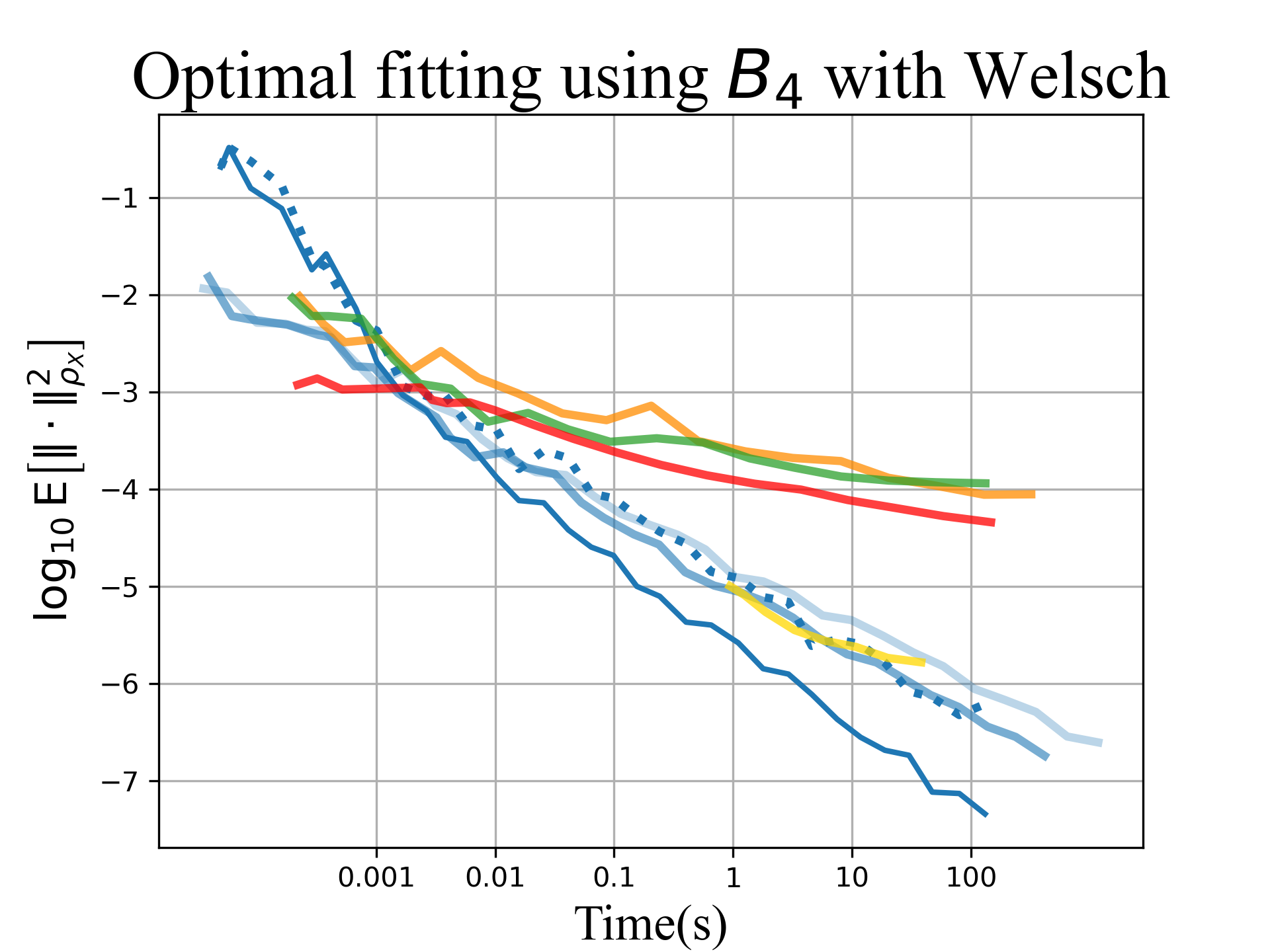}
\caption{The left figure illustrates the convergence of the error with respect to the sample size under three different losses, while the right figure shows the convergence of the error with respect to runtime. The black line indicates the minimax rate, with the slope $-\frac{7}{9}$. Because the left and right figures share a common legend, the legend is omitted from the right figures.}
\label{fig1:example1}
\end{figure}
The comparative experimental results for kernel SGD and T-kernel SGD in Example 1 are presented in \autoref{fig1:example1}. When the target function $f^{*}$ satisfies a higher regularity condition ($r = \tfrac{7}{4} > 1$), T-kernel SGD consistently achieves the theoretically optimal rate, even under non-convex losses such as the Cauchy and Welsch losses. Moreover, for general hyperparameter choices, the experimental results show that the algorithm still attains the optimal rate. Kernel SGD exhibits clear saturation with the Bernoulli polynomial kernel and converges more slowly than the minimax rate. In contrast, T-kernel SGD requires less runtime than kernel SGD to reach a given error level.

\begin{figure}[htbp]
\centering
\includegraphics[width=7.5cm]{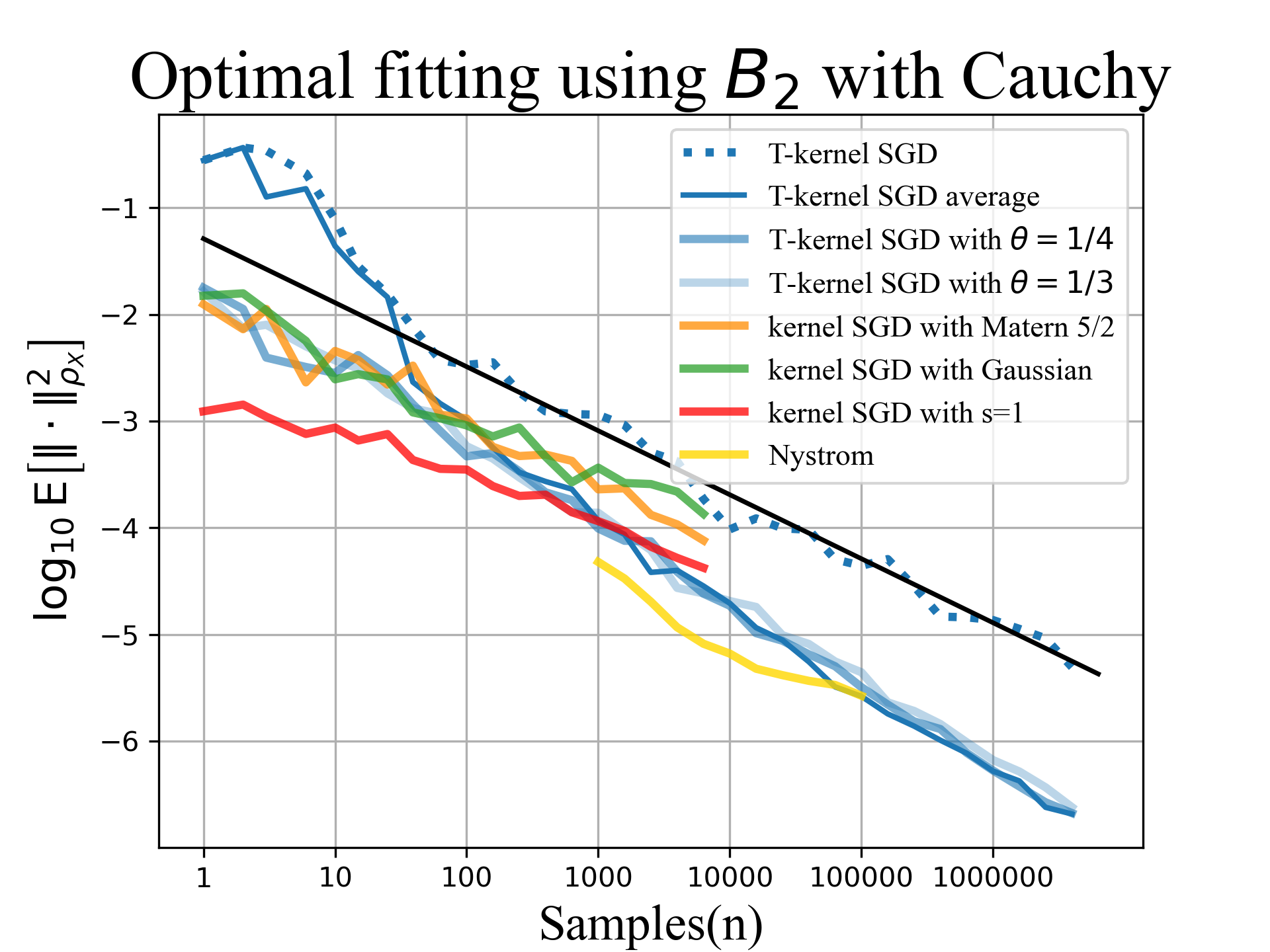}
\includegraphics[width=7.5cm]{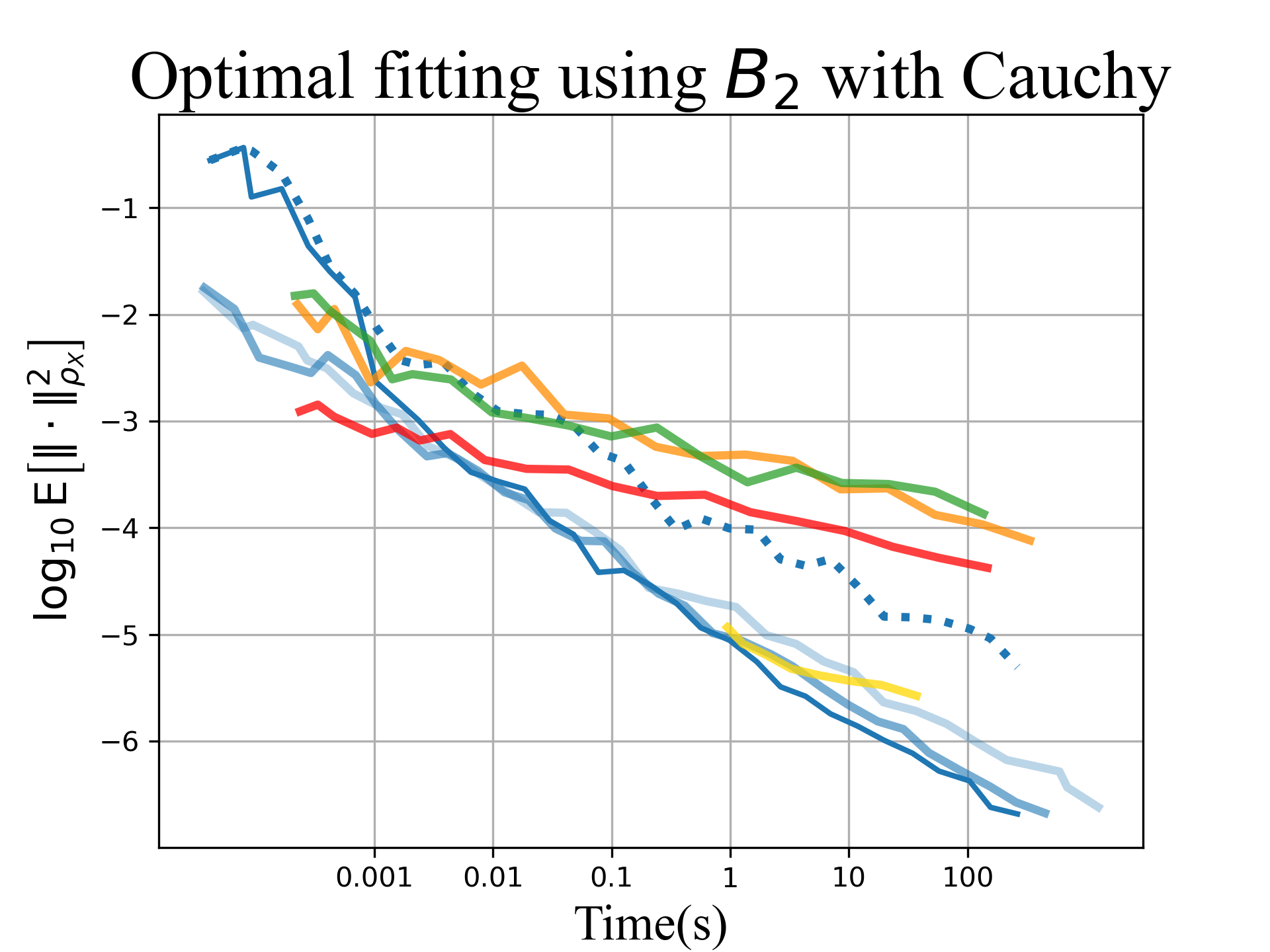}
\includegraphics[width=7.5cm]{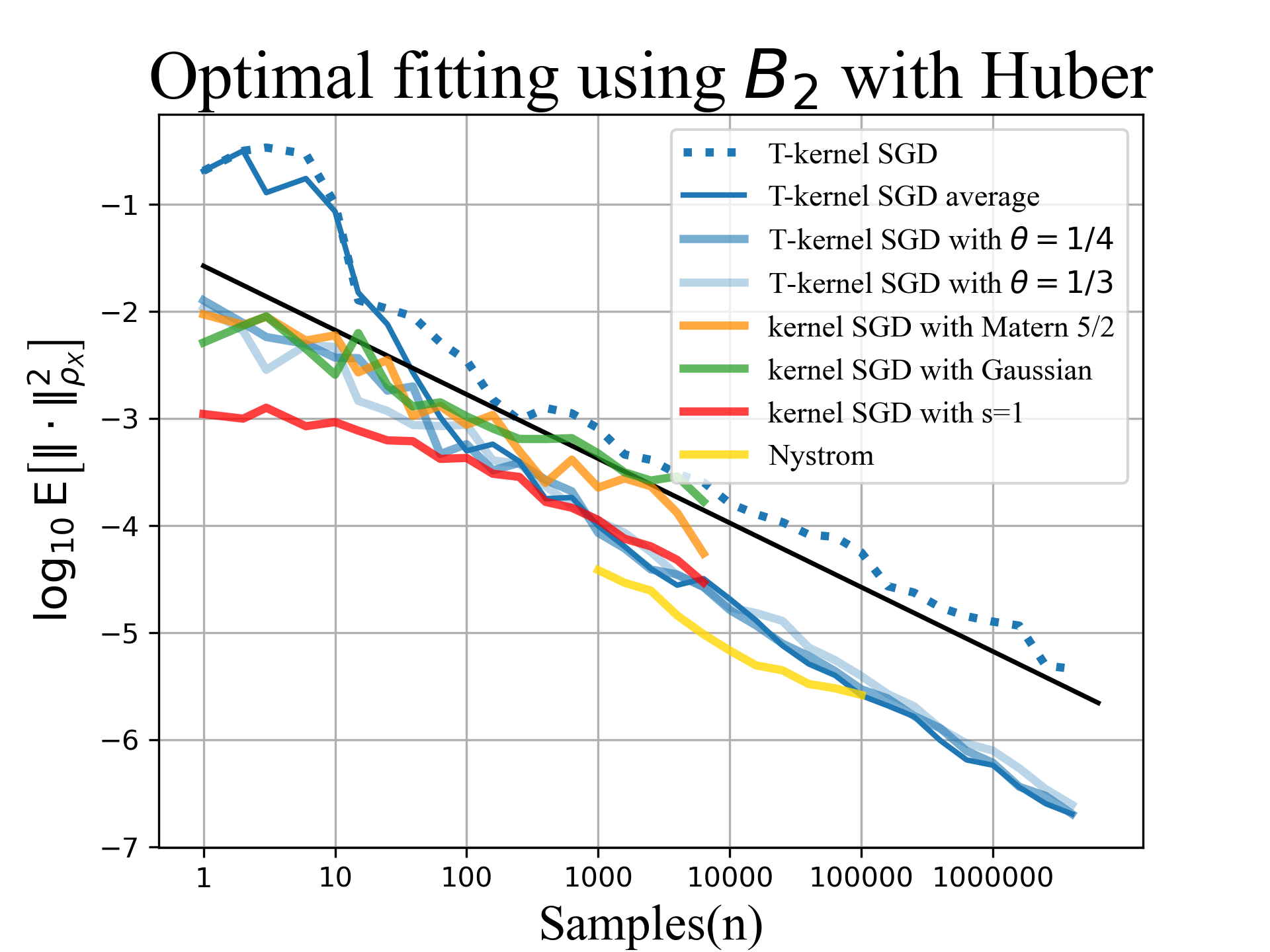}
\includegraphics[width=7.5cm]{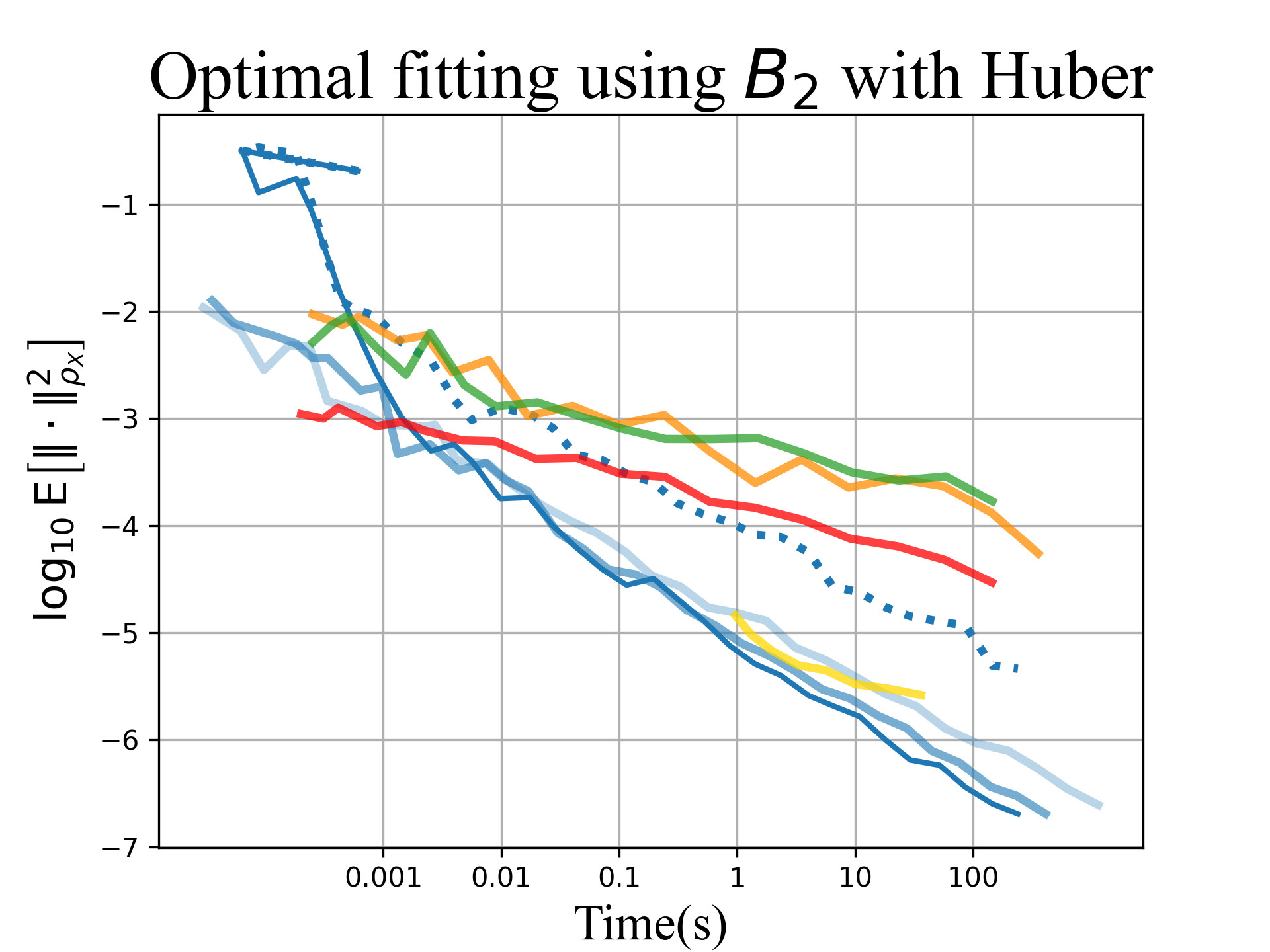}
\includegraphics[width=7.5cm]{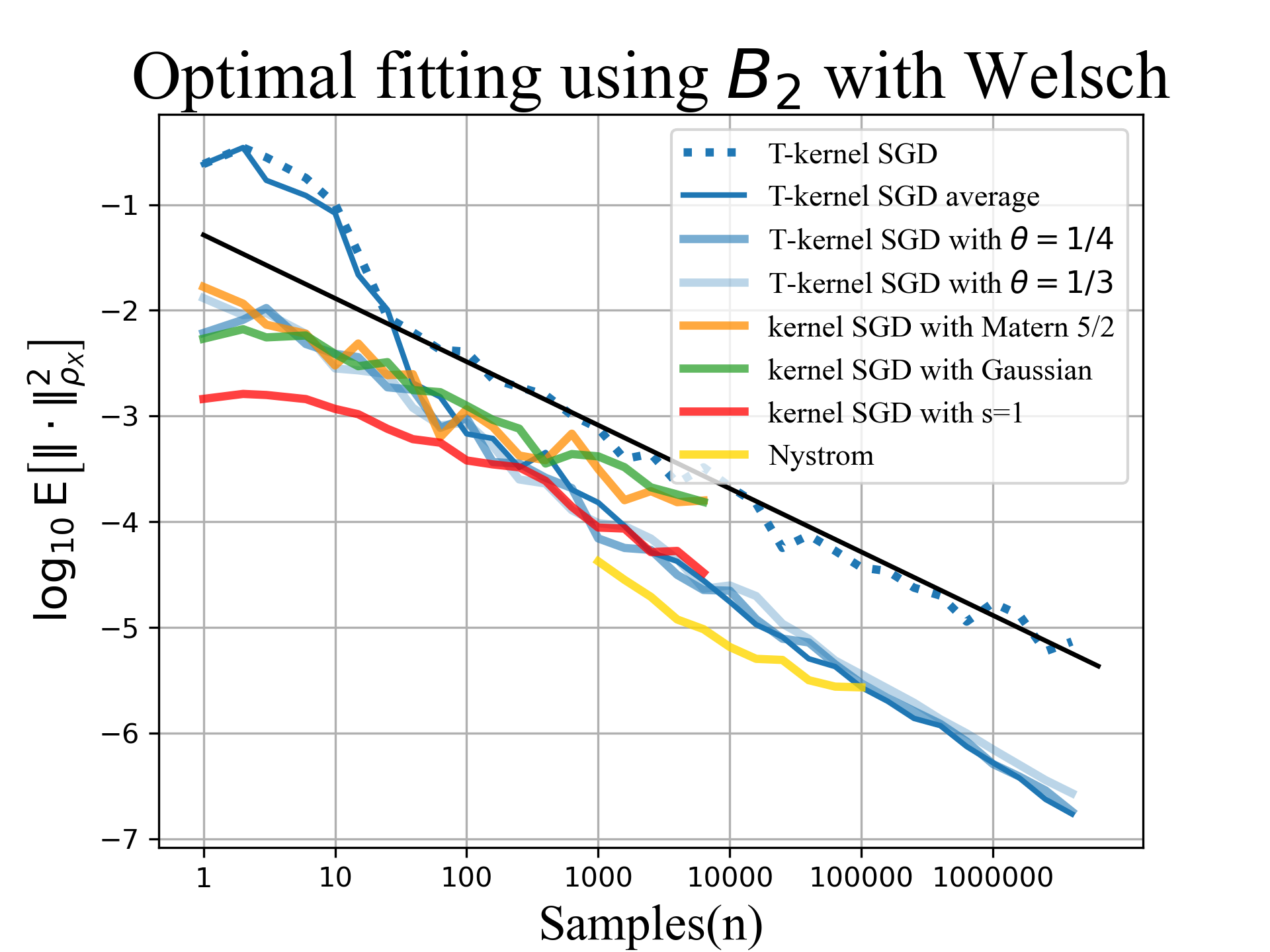}
\includegraphics[width=7.5cm]{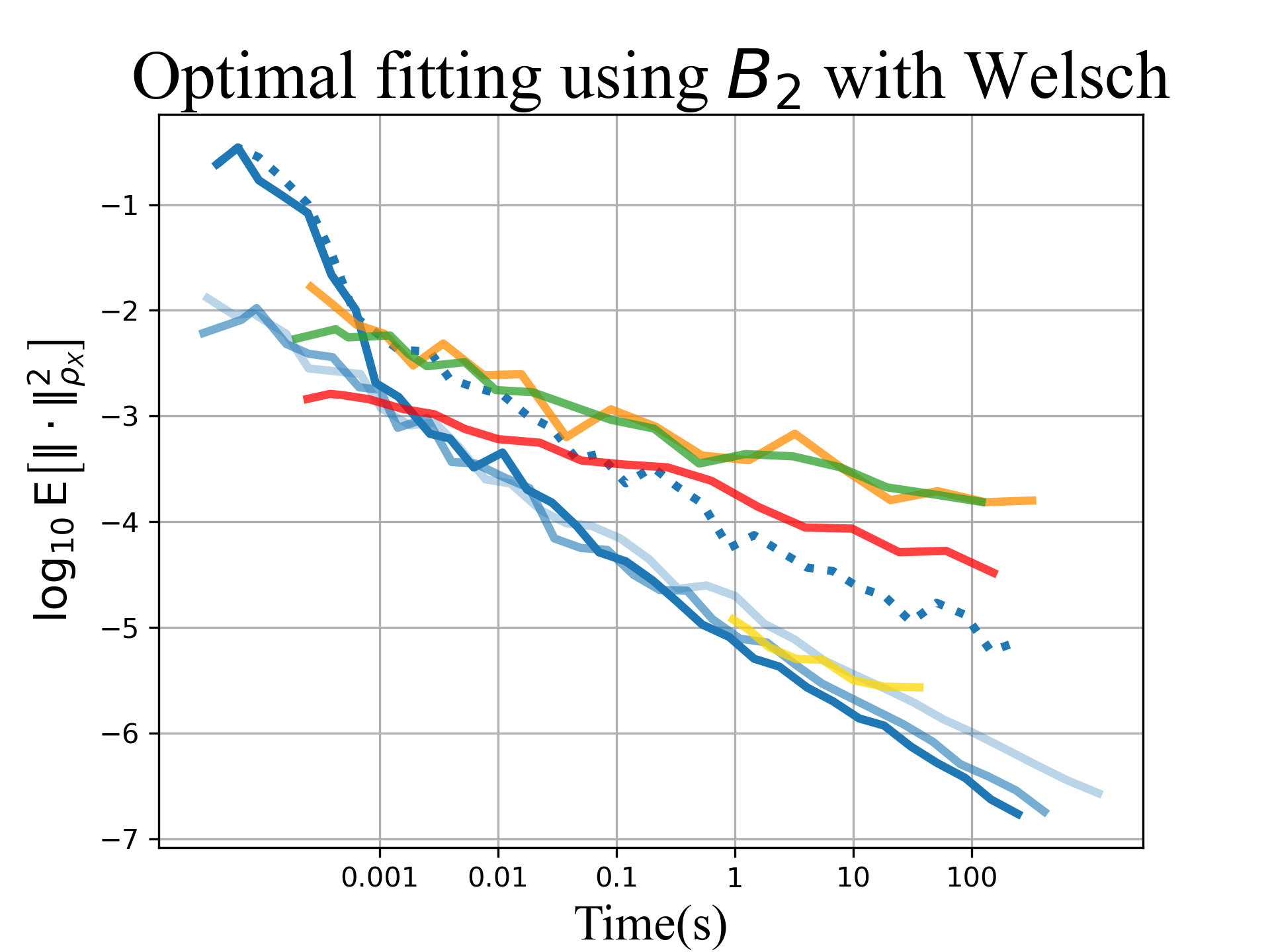}
\caption{ The left figure illustrates the convergence of the error with respect to the sample size under three different losses, while the right figure shows the convergence of the error with respect to runtime. The black line indicates the minimax rate, with the slope $-\frac{3}{5}$. }
\label{fig2:example2}
\end{figure}
The experimental results for Example 2 are shown in \autoref{fig2:example2}, demonstrating the convergence of the algorithm when $f^{*}$ satisfies weaker regularity conditions ($r = \tfrac{3}{4}$). In this case, T-kernel SGD also attains the theoretically predicted rate with shorter runtimes than kernel SGD.

Furthermore, in the experiments, we observed that the projection step is triggered primarily in the early stage of the iteration, that is, when the intermediate iterate $\hat{g}_n$ falls outside the constraint set $\WW$. As the iteration proceeds, however, $\hat{g}_n$ gradually stabilizes within $\WW$. This empirical observation suggests that, under our experimental settings, the projection step mainly serves to stabilize the iterates during the initial phase, whereas its practical effect may become relatively limited once the iterates remain stably inside $\WW$ in later stages.

\subsubsection{Sensitivity analysis}

We next investigate the sensitivity of T-kernel SGD to perturbations in the hyperparameters $(s,r,\theta,Q)$, with particular attention to how deviations from their reference values affect empirical performance. We use the same model as in Example 2, with the target function $\frac{1}{5}B_{2}\left(\frac{\xi}{2\pi}\right)$ and uniformly distributed noise $U[-0.2,0.2]$. To isolate the effect of each hyperparameter, we take the parameter configuration $(\gamma_n,s,r,\theta,Q)=(4n^{-\frac35},1,\frac34,\frac15,1)$ from Example 2, which satisfies the conditions of \autoref{theorem:mean result}, as the reference setting and vary only one hyperparameter at a time. Throughout the experiment, we use the Huber loss and report the performance of the $\frac12$-suffix averaging estimator. For the projection radius $Q$, we additionally examine the empirical performance of the validation-based selection procedure proposed in \autoref{subsection:3.3}. The detailed settings are summarized in Table~\ref{tab:sensitivity}.

\begin{table}[htbp]
\centering
\small
\renewcommand{\arraystretch}{1.15}
\caption{Hyperparameter settings for the sensitivity analysis.}
\label{tab:sensitivity}
\begin{tabular}{c c c p{6.2cm}}
\hline
Varied quantity
& Tested values
& Reference
& Remaining settings
\\
\hline

Chosen regularity $r_1$
&
$\frac12,\frac23,\frac34,1,\frac54$
&
$\frac34$
&
$s=1$, $Q=1$,
$t=\frac{2r_1}{2r_1+1}$, and
$\theta=\frac{1}{2s(2r_1+1)}$
\\

Capacity parameter $s$
&
$\frac34,\frac78,1,\frac98,\frac54$
&
$1$
&
$t=\frac35$, $\theta=\frac15$, and $Q=1$
\\

Truncation parameter $\theta$
&
$\frac18,\frac16,\frac15,\frac14,\frac13$
&
$\frac15$
&
$s=1$, $t=\frac35$, and $Q=1$
\\

Projection radius $Q$
&
$\frac12,\frac34,1,\frac54,\widehat Q$
&
$1$
&
$s=1$, $t=\frac35$, and $\theta=\frac15$
\\

\hline
\end{tabular}
\end{table}
For the validation-based selection of the projection radius $\widehat Q$, we generate mutually independent training and validation samples, each of size $10{,}000$. Using the training sample, we compute $s_Y$ and set the reference scale to $Q_0=s_Y/\kappa$, where $\kappa^2=\sup_{x\in\mathbb S^1}K(x,x)=1+\frac{\pi^2}{12}$. For each candidate radius in $\mathcal Q=Q_0\left\{\frac13,1,3,9,27,81\right\}$, we train T-kernel SGD on the training sample and evaluate its empirical risk on the validation sample. We then select $\widehat Q$ by minimizing the validation risk over $\mathcal Q$. In the experiment, the true regularity parameter is $r=\frac34$, whereas $r_1$ denotes the chosen regularity parameter used to determine $\theta$ and $\gamma_n$; varying $r_1$ allows us to examine how a mismatch between the chosen and true regularity parameters affects the empirical performance of T-kernel SGD. We also consider different values of the capacity parameter $s$, corresponding to RKHSs with different capacities, to investigate their effects on the empirical performance of the algorithm.
\begin{figure}[htbp]
\centering
\includegraphics[width=7.5cm]
{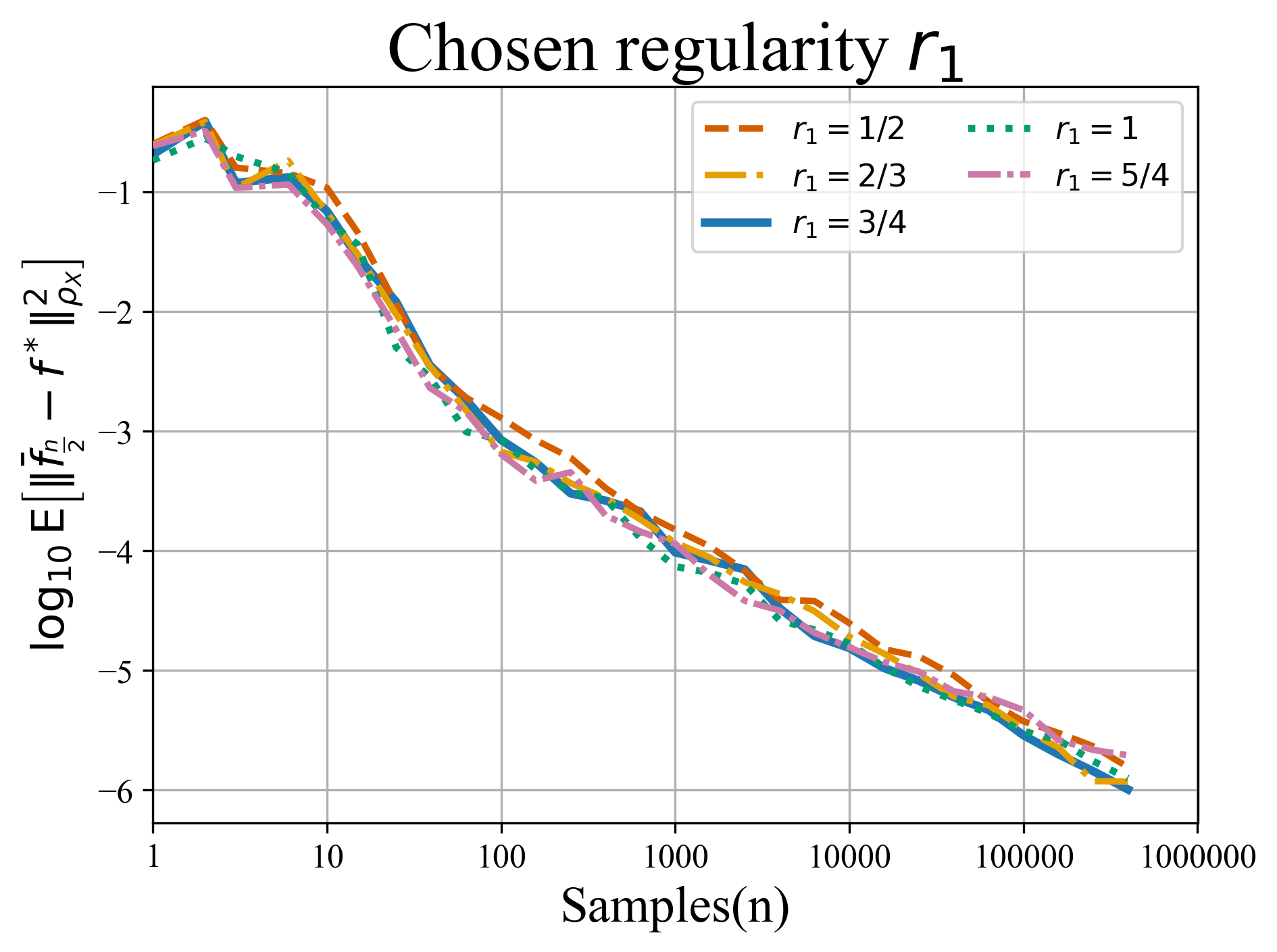}
\hfill
\includegraphics[width=7.5cm]
{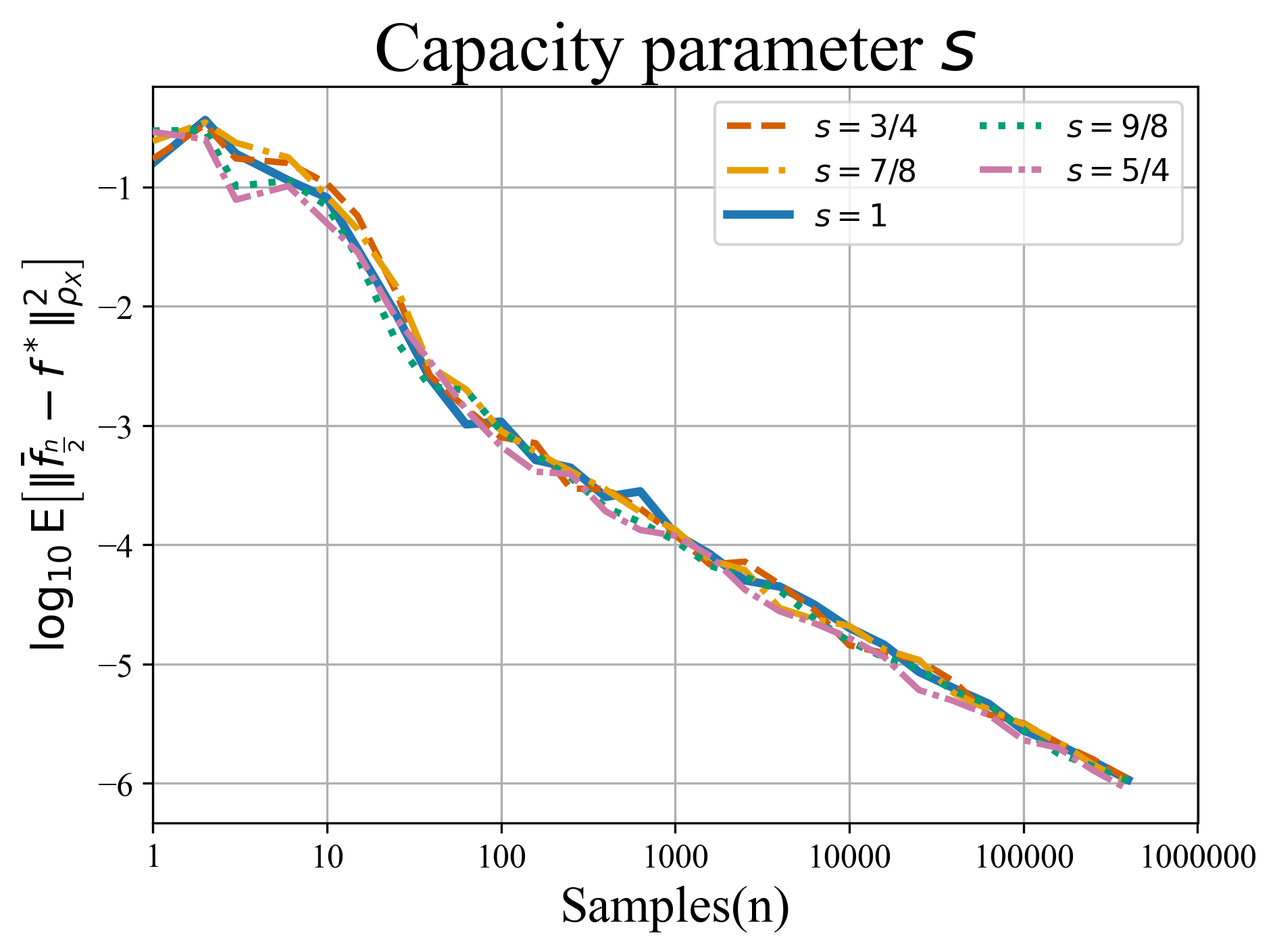}

\vspace{0.3cm}

\includegraphics[width=7.5cm]
{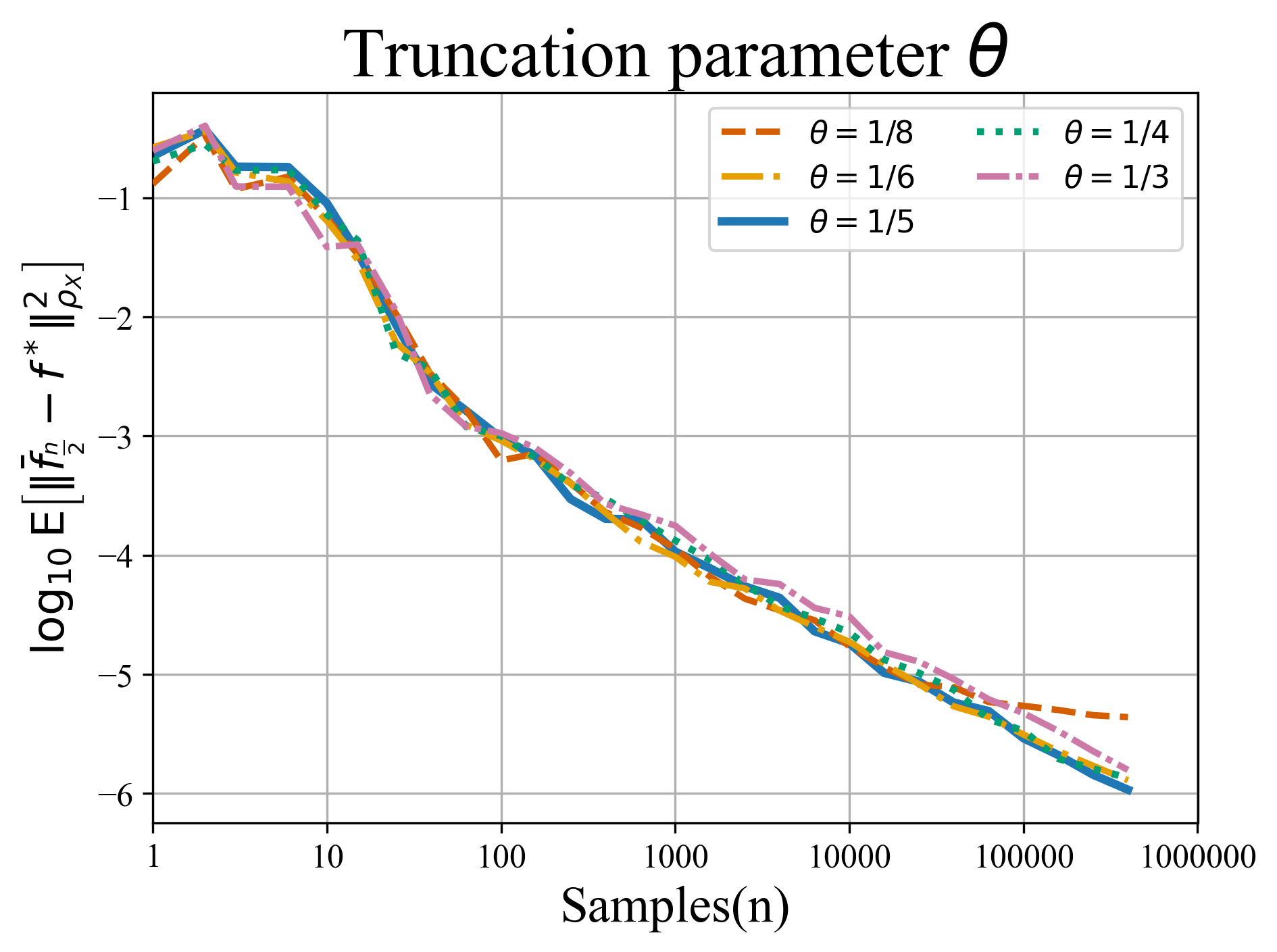}
\hfill
\includegraphics[width=7.5cm]
{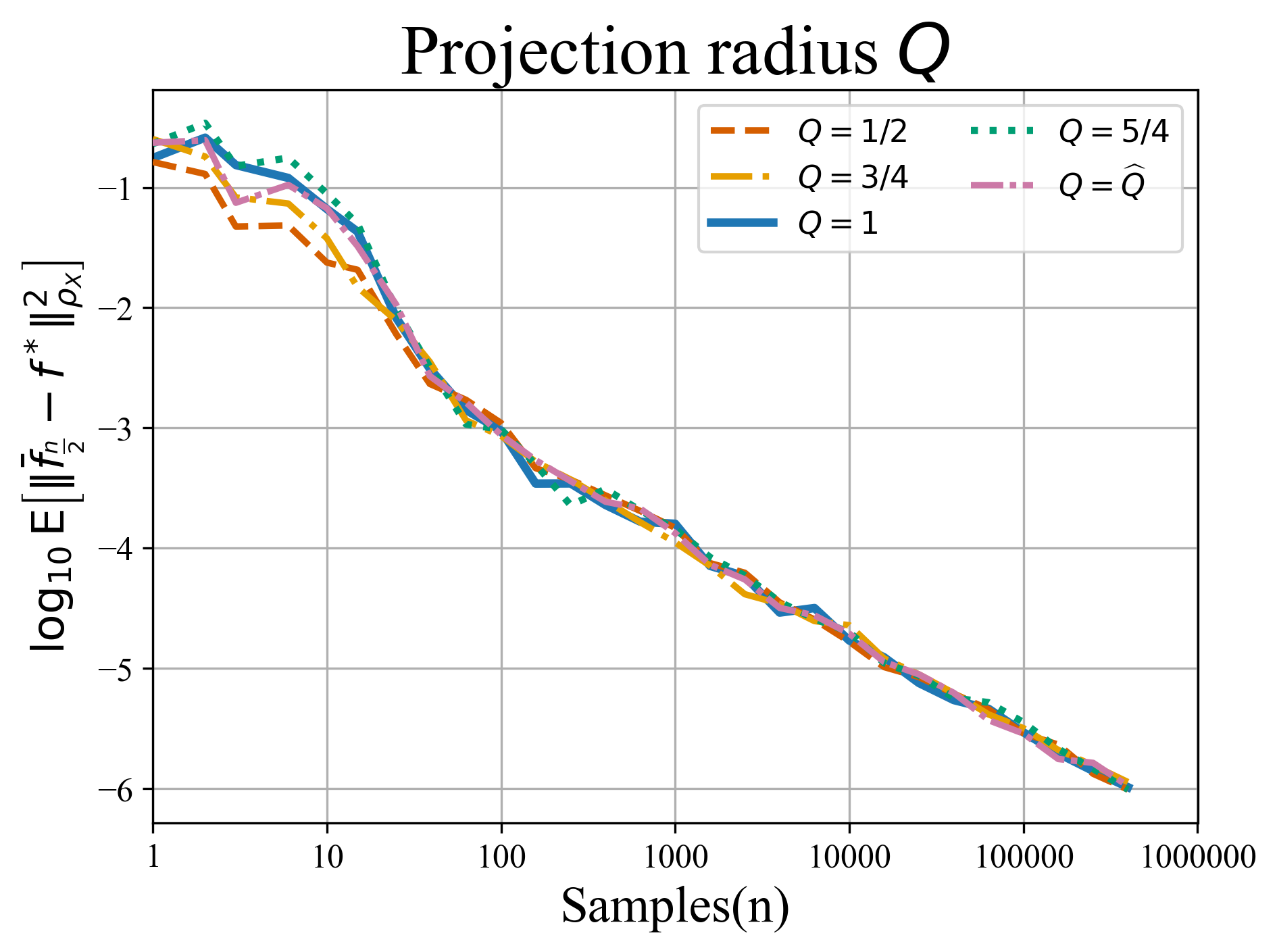}
\caption{
The upper-left, upper-right, lower-left, and lower-right figures present the experimental results for the chosen regularity parameter $r_1$, capacity parameter $s$, truncation parameter $\theta$, and projection radius $Q$, respectively. Each curve reports the error averaged over $20$ independent repetitions, and the solid blue curve represents the empirical performance under the reference parameter configuration.
}
\label{fig:sensitivity}
\end{figure}

The experimental results are presented in \autoref{fig:sensitivity}. Overall, moderate perturbations of the hyperparameters affect the empirical performance of T-kernel SGD but do not lead to noticeable deterioration in its convergence rate. More specifically, when the chosen regularity parameter $r_1$ is close to the true regularity parameter $r$, the algorithm generally achieves smaller errors, consistent with the analysis in \autoref{subsection:3.3}. A similar pattern is observed for the truncation parameter $\theta$. By contrast, the empirical performance appears relatively insensitive to perturbations in the capacity parameter $s$ and projection radius $Q$.

\subsection{Robust Regression on 3-Dimensional Spherical Data}\label{subsection:Robust Regression on 3-Dimensional Spherical Data}

To validate the theoretical analysis, we design experiments on the three-dimensional sphere $\mathbb{S}^{2}$ using the Cauchy, Huber, and Welsch losses. Here, we consider the explanatory variable $X$ uniformly distributed on the sphere, and the response $Y = f^{*}(X) + \epsilon$ with additive Gaussian noise $\epsilon \sim \mathcal{N}(0, 0.2^{2})$. The function $f^{*}$ is defined as 
$f^{*} = \frac{1}{5} \sum_{k=0}^{10} \left( \dim \Pi^{3}_{k} \right)^{-0.501 - 2s r}
        \sum_{j=1}^{2k+1} Y_{k,j},$
where $s = 1$, $r = 1$ and $\dim \Pi^{3}_{k} = (k+1)^2.$ In T-kernel SGD, we set $Q = 1$, $\frac{\gamma_{n}}{\gamma_{0}} = n^{-\frac{2r}{2r+1}}$, and $\theta = \frac{1}{2s(2r+1)}$, and use both the last iterate and the $\frac12$-suffix average as outputs, in accordance with \autoref{theorem:mean result}. In kernel SGD, we consider the Gaussian kernel, the Mat\'{e}rn-$\frac{5}{2}$ kernel, and the Mat\'{e}rn-$\tfrac{3}{2}$ kernel, given by
\begin{align*}
&K_{Gaussian}(\tau)=\exp\left(-\frac{\tau^2}{2}\right),\quad K_{Matern}^{5/2}(\tau) = \left(1+\sqrt{5}\tau+\frac{5\tau^2}{3}\right)\exp\left(-\sqrt{5}\tau\right),\\
&K^{3/2}_{\mathrm{Matern}}(\tau) =\left( 1 + \sqrt{3}\,\tau \right) \exp\left( -\sqrt{3}\,\tau \right),
\end{align*}
where $\tau = \Vert x - x'\Vert$. We further set the step size in kernel SGD as $\gamma_{n} = \gamma_{0} n^{-\frac{2r}{2r+1}}$. The experimental results in \autoref{fig3:example3} show that T-kernel SGD attains the theoretically predicted optimal rate and requires less runtime than kernel SGD to reach a given error level.

\begin{figure}[htbp]
\centering
\includegraphics[width=7.5cm]{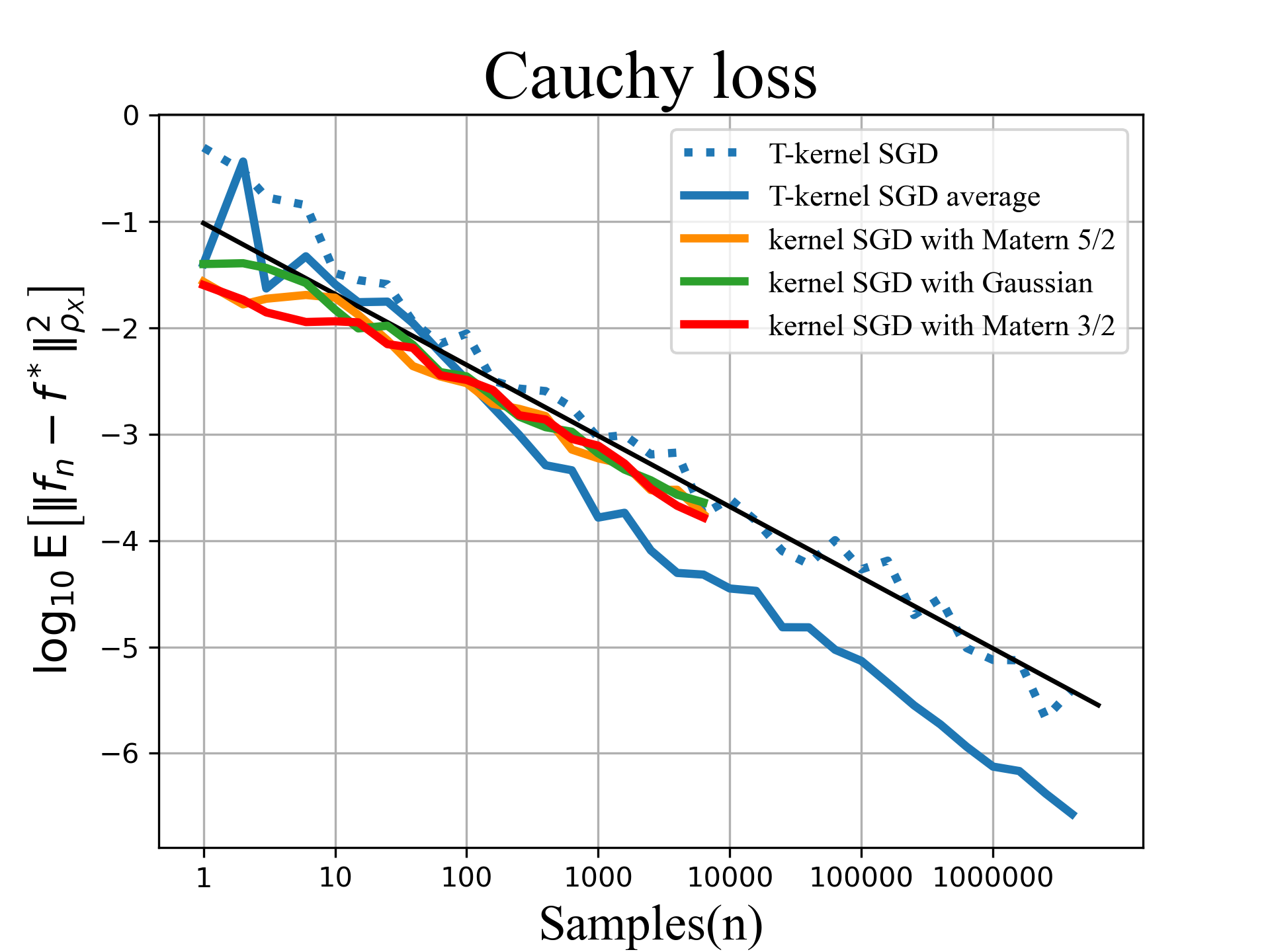}
\includegraphics[width=7.5cm]{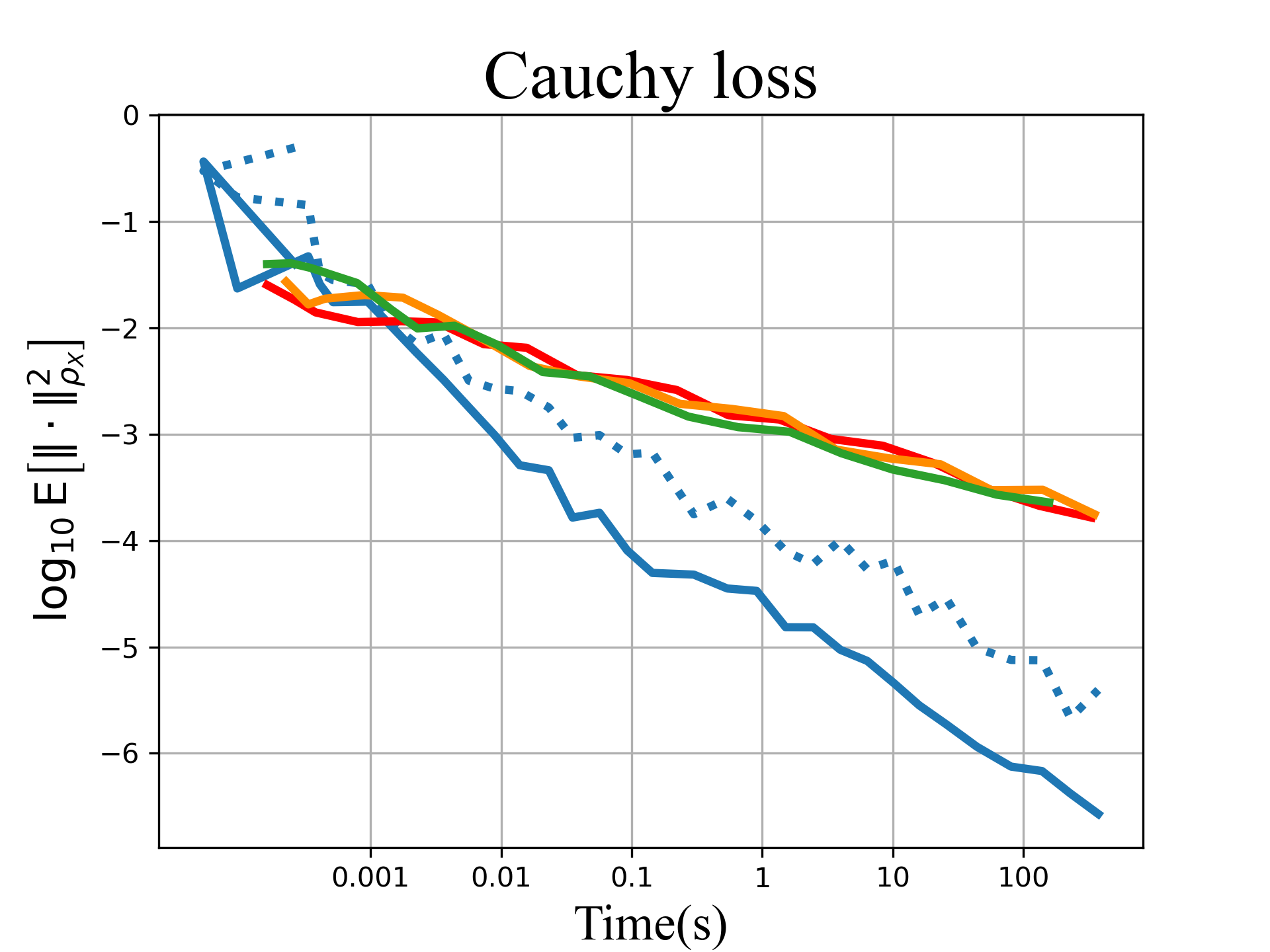}
\includegraphics[width=7.5cm]{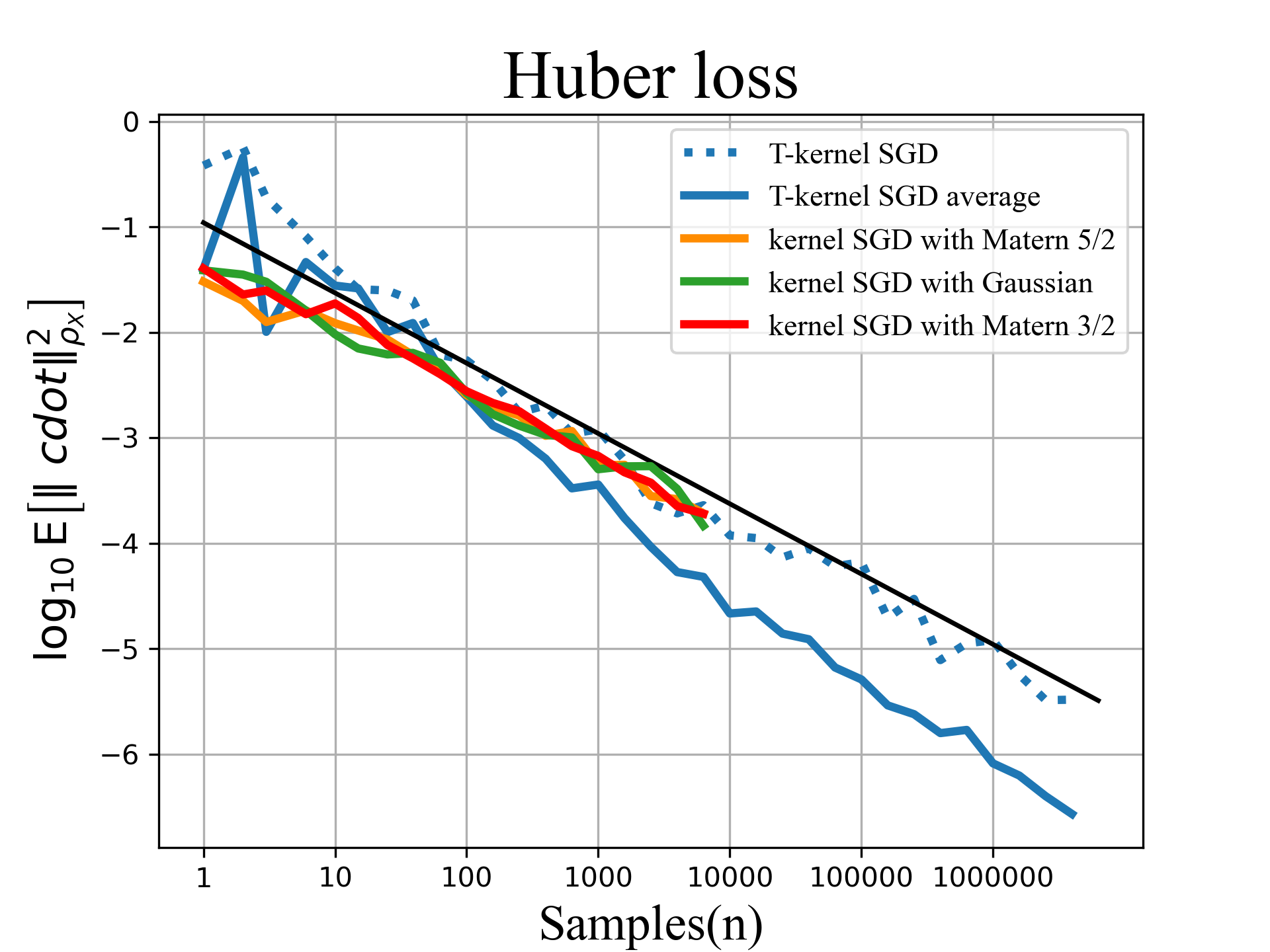}
\includegraphics[width=7.5cm]{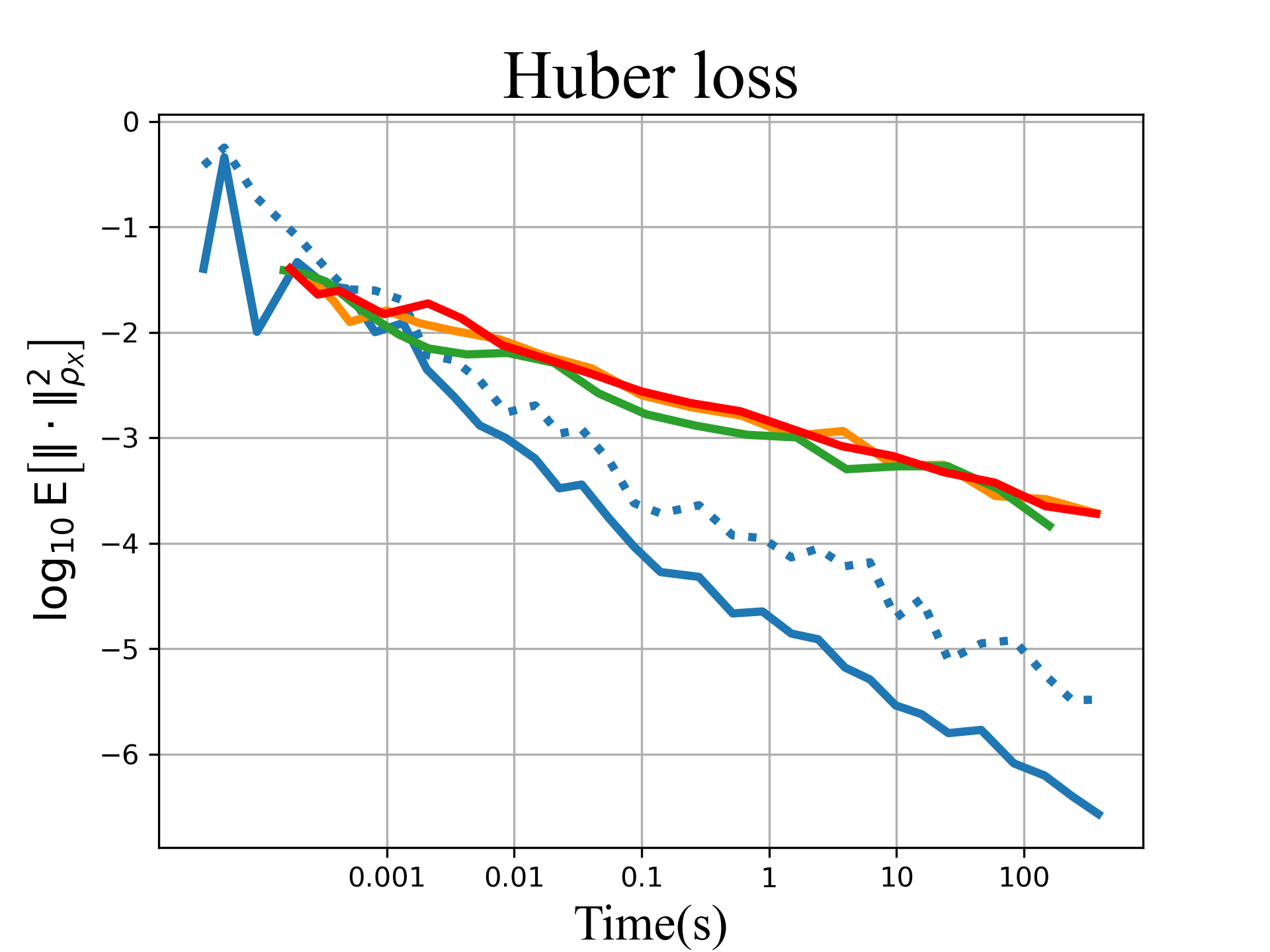}
\includegraphics[width=7.5cm]{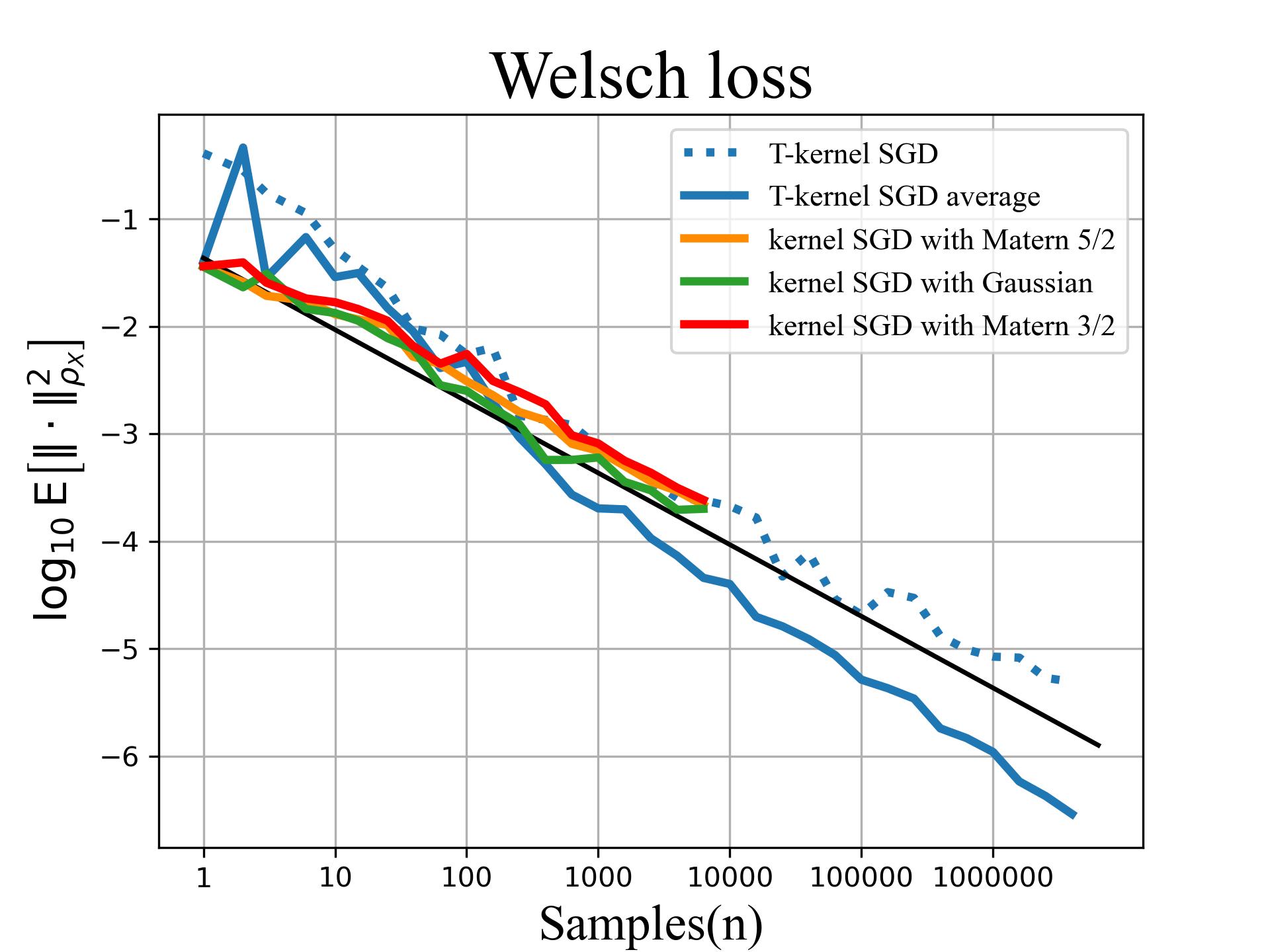}
\includegraphics[width=7.5cm]{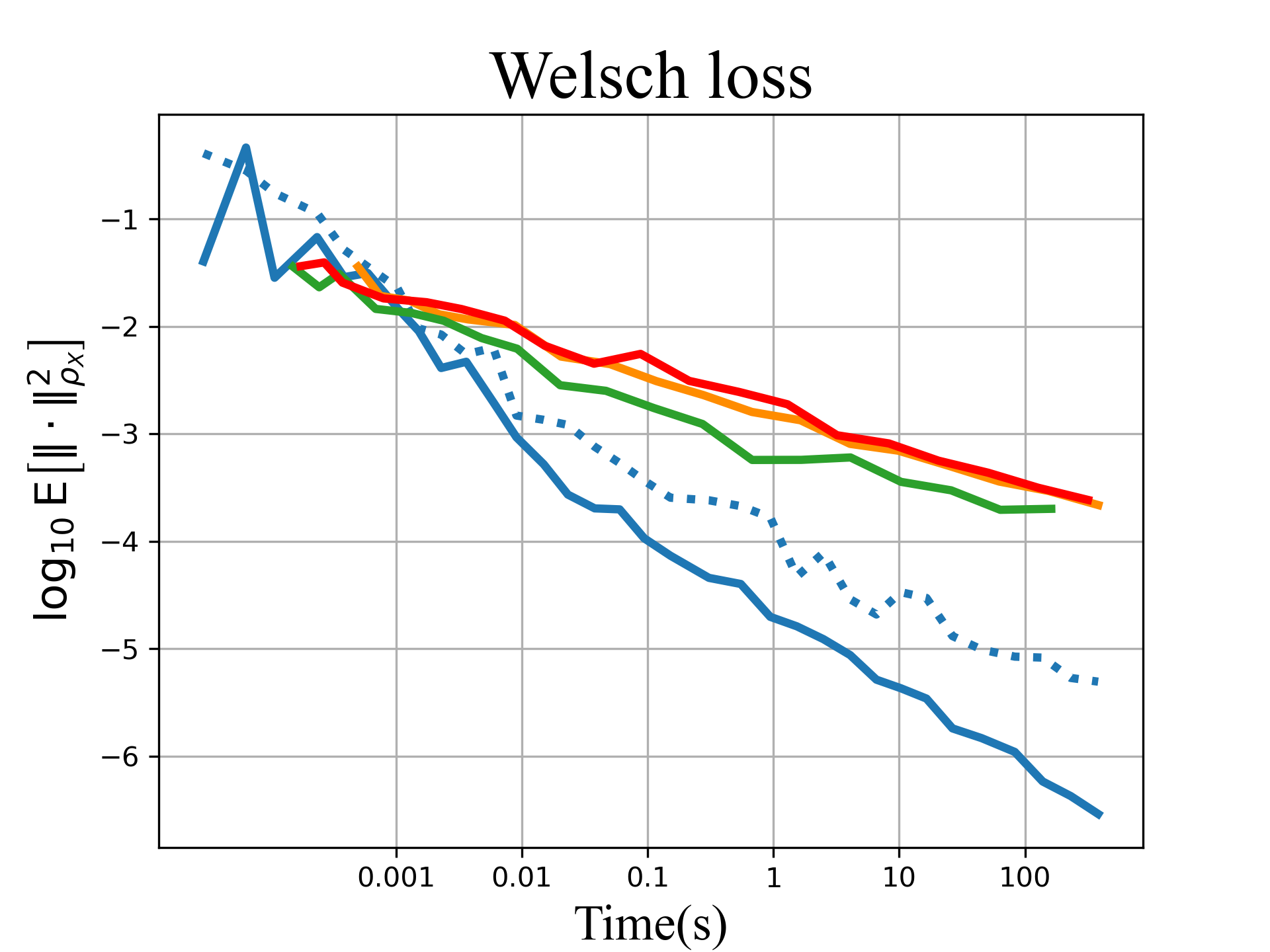}
\caption{ Error convergence versus sample size (left figure) and runtime (right figure) for three loss functions. The black line indicates the minimax rate with slope $-2/3$. The legend is shown only in the left figure because it applies to both figures.}
\label{fig3:example3}
\end{figure}

\subsection{Binary Classification of High-Dimensional MNIST Dataset}\label{subsection:Binary Classification of High-Dimensional Non-Spherical Data}

In this subsection, we illustrate the application of T-kernel SGD to a real-world nonspherical data set. Specifically, we consider the binary classification problem of distinguishing odd from even digits in the MNIST dataset using the logistic loss. The 784-dimensional MNIST dataset is a standard benchmark for evaluating machine learning algorithms. In our experiment, the output space $\mathcal{Y} = \{-1,1\}$ represents odd and even digits, respectively. We also compare the performance of T-kernel SGD with classical kernel SGD. 

In T-kernel SGD, we map each input $x$ to the spherical point $F(x/5)$ using the inverse spherical-polar projection \cite{jost2017riemannian}, defined by
\begin{align*}
F : \mathbb{R}^{d}_{+} \to \mathbb{S}^{d}, \quad x \to F(x) = \frac{1}{4 + x_{1}^{2} + \cdots + x_{d}^{2}}
\left( 4x_{1}, \ldots, 4x_{d}, (4 - x_{1}^{2} - \cdots - x_{d}^{2}) \right).
\end{align*}
We use $K^{T}_{L_{n}}(x, x') = \sum_{k=0}^{L_{n}} \left( \dim \Pi^{d+1}_{k} \right)^{-2s}K_{k}(x, x')$ as the truncated kernel in the iteration, with step size $\gamma_{n} = 0.6 n^{-0.05}$ and hyperparameters $\theta = 0.68$ and $s = 0.505$. For this real-world classification problem, the RKHS norm of the minimizer $f^{*}$ is unknown, so we choose $Q = 1000$ to be sufficiently large and use both Polyak averaging and the last iterate as outputs. In the comparison experiment with kernel SGD, we adopt the Gaussian kernel $K(x, x') = \exp\left(-\frac{\|x - x'\|^{2}}{2\sigma^{2}}\right)$ and set $\sigma = 20$ to account for the high dimensionality of the data. To improve the robustness of kernel SGD, we apply Polyak averaging as in \cite{dieuleveut2016nonparametric} and use a constant step size $\gamma_{n} = 0.1$.
\begin{figure}[htbp]
\centering
\includegraphics[width=6.5cm]{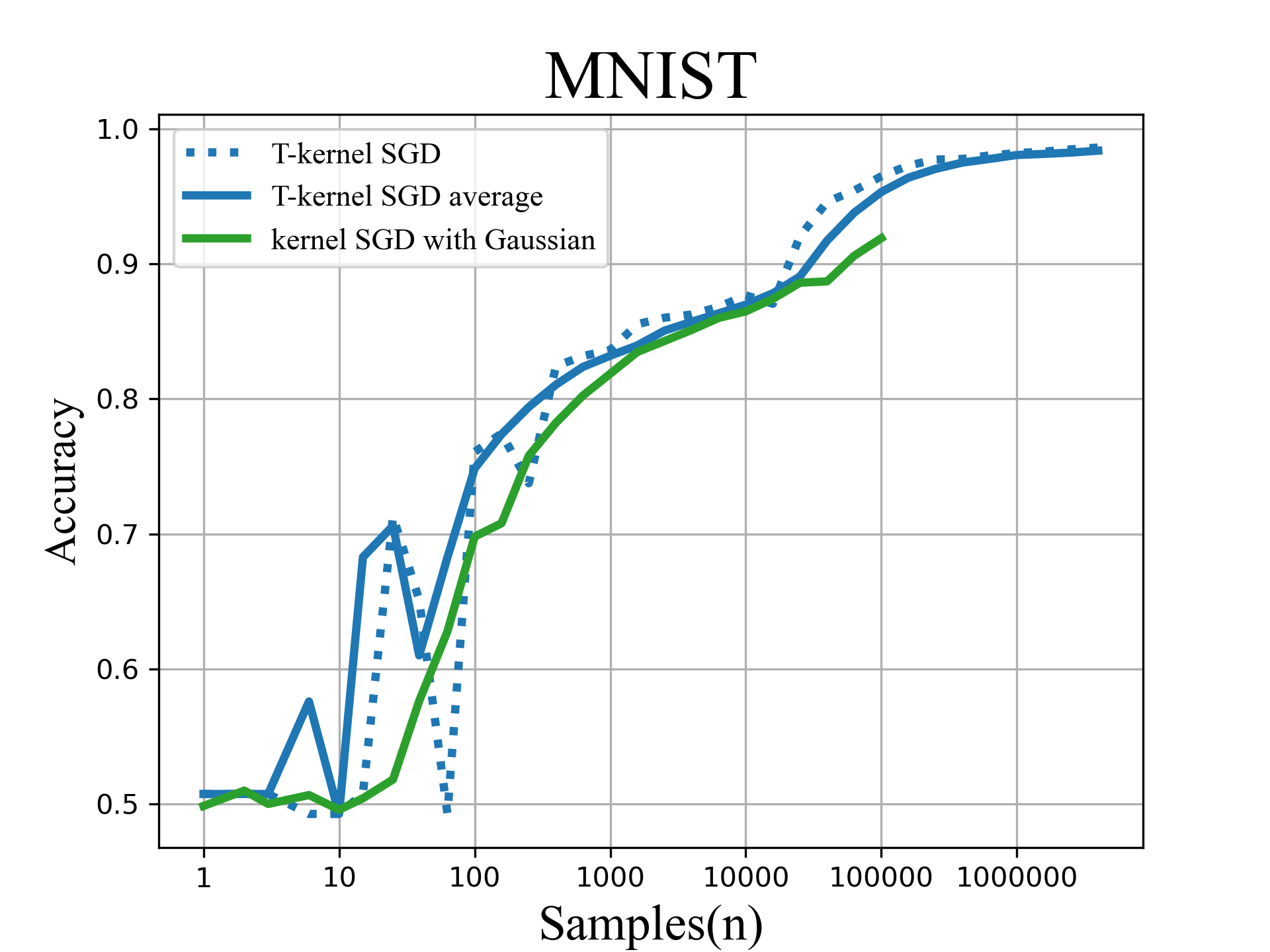}
\includegraphics[width=6.5cm]{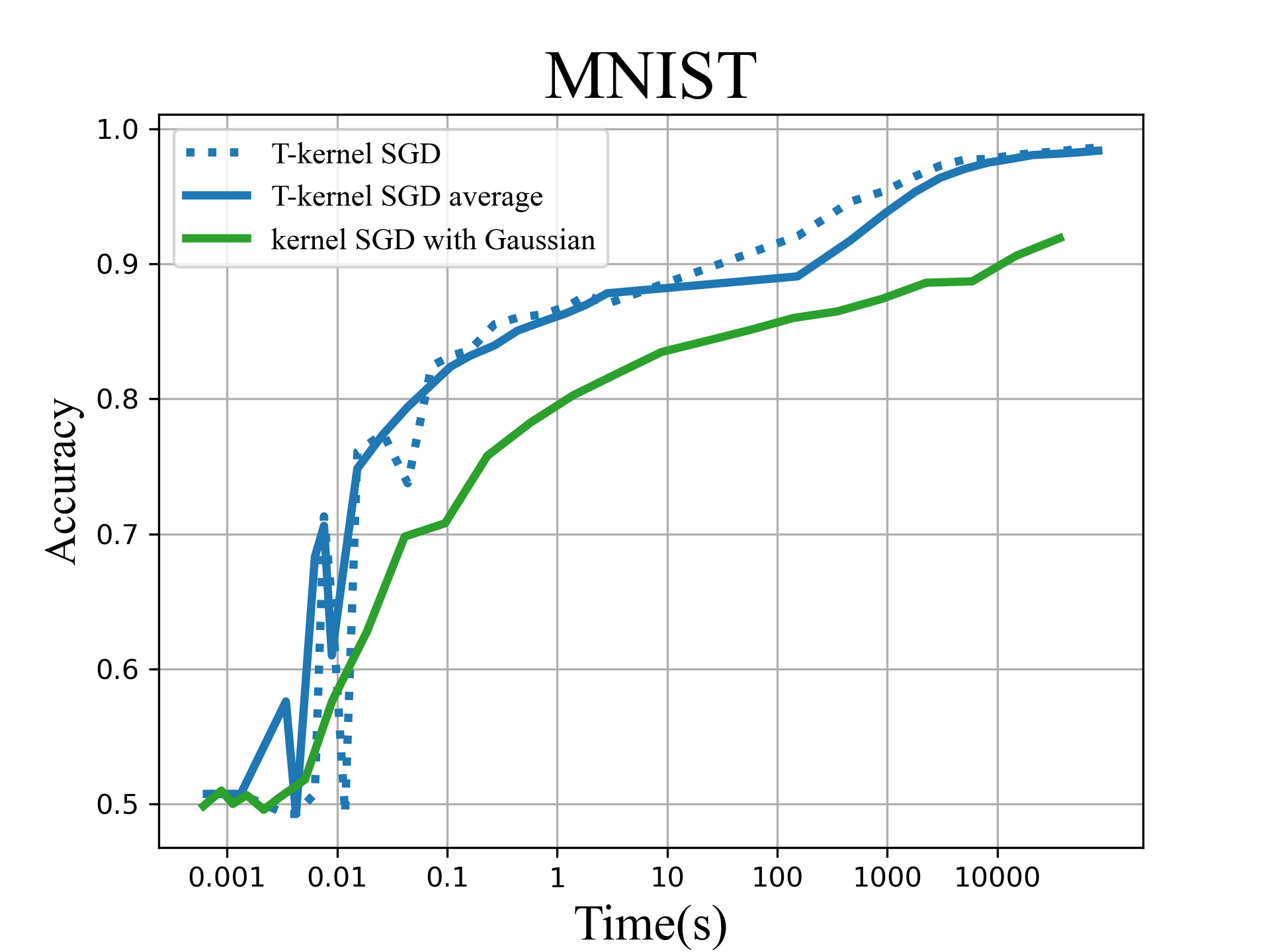}
\caption{The figures show accuracy versus sample size and runtime, respectively.}
\label{fig4:example4}
\end{figure}
We augment the original MNIST dataset with Gaussian white noise. As shown by the sample- and time-to-accuracy plots in \autoref{fig4:example4}, T-kernel SGD attains higher test accuracy than kernel SGD at comparable sample sizes and runtimes. With the prescribed $\theta$ and $n\leq10^6$, the truncation level satisfies $L_n\leq2$, giving a hypothesis space of dimension $309{,}290$. Despite the 784-dimensional inputs, the results suggest that the minimizer can be well approximated by functions in a low-degree spherical harmonic space, consistent with relatively strong regularity. Thus, T-kernel SGD may be feasible for certain high-dimensional problems with sufficiently regular minimizers. 

\subsection{Robust Regression on GRACE Satellite Data}\label{subsec:Robust Regression on GRACE Satellite Data}

In this section, we use real Earth observational data from the GRACE satellite mission\footnote{The GRACE satellite data used in this study are available from: \url{https://search.earthdata.nasa.gov/search/granules?p=C2491772131-POCLOUD&pg[0][v]=f&pg[0][gsk]=start_date&q=GRACE\%20level-2}} to evaluate the performance of T-kernel SGD in a practical task. By precisely measuring the Earth's time-varying gravity field, the GRACE mission provides direct information on large-scale mass transport processes in the Earth system and enables a global characterization of the spatiotemporal variations in terrestrial water storage, glacier and ice-sheet mass, ocean mass, and certain solid-Earth processes. We consider GRACE satellite data from four different months, namely, 2003.1, 2003.4, 2003.7, and 2003.10, and fit these data using T-kernel SGD.

We first preprocess the data by removing the constant and first-order terms, and rescale the data by a factor of 1000 to improve the numerical scale for subsequent computation and presentation. We then add Gaussian noise with distribution $\mathcal{N}(0, 0.2^{2})$ to the outputs of the real data. We choose the Cauchy loss and adopt the general hyperparameter rule proposed in \autoref{subsection:3.3}, setting $s=0.51$, step size $\gamma_n=1.5n^{-\frac12}$, truncation parameter $\theta=0.28$, and radius $Q=20$ for the closed convex set $\WW$. In the experiment, we examine the convergence of the Polyak average $\bar{f}_n$ through both its prediction error $\EE\left[\left\Vert\bar{f}_{n}-f^*\right\Vert_{\rho_X}^2\right]$ and the error $\EE\left[\left\Vert\Delta_{\SS^{2}}^{\frac12} \bar{f}_n-\Delta_{\SS^{2}}^{\frac12} f^*\right\Vert_{\rho_X}^2\right]$, where $\Delta_{\SS^{2}}$ denotes the Laplace-Beltrami operator on the sphere. This allows us to assess the convergence of the estimator in the Sobolev-space sense.
\begin{figure}[htbp]
\centering
\includegraphics[width=7.5cm]{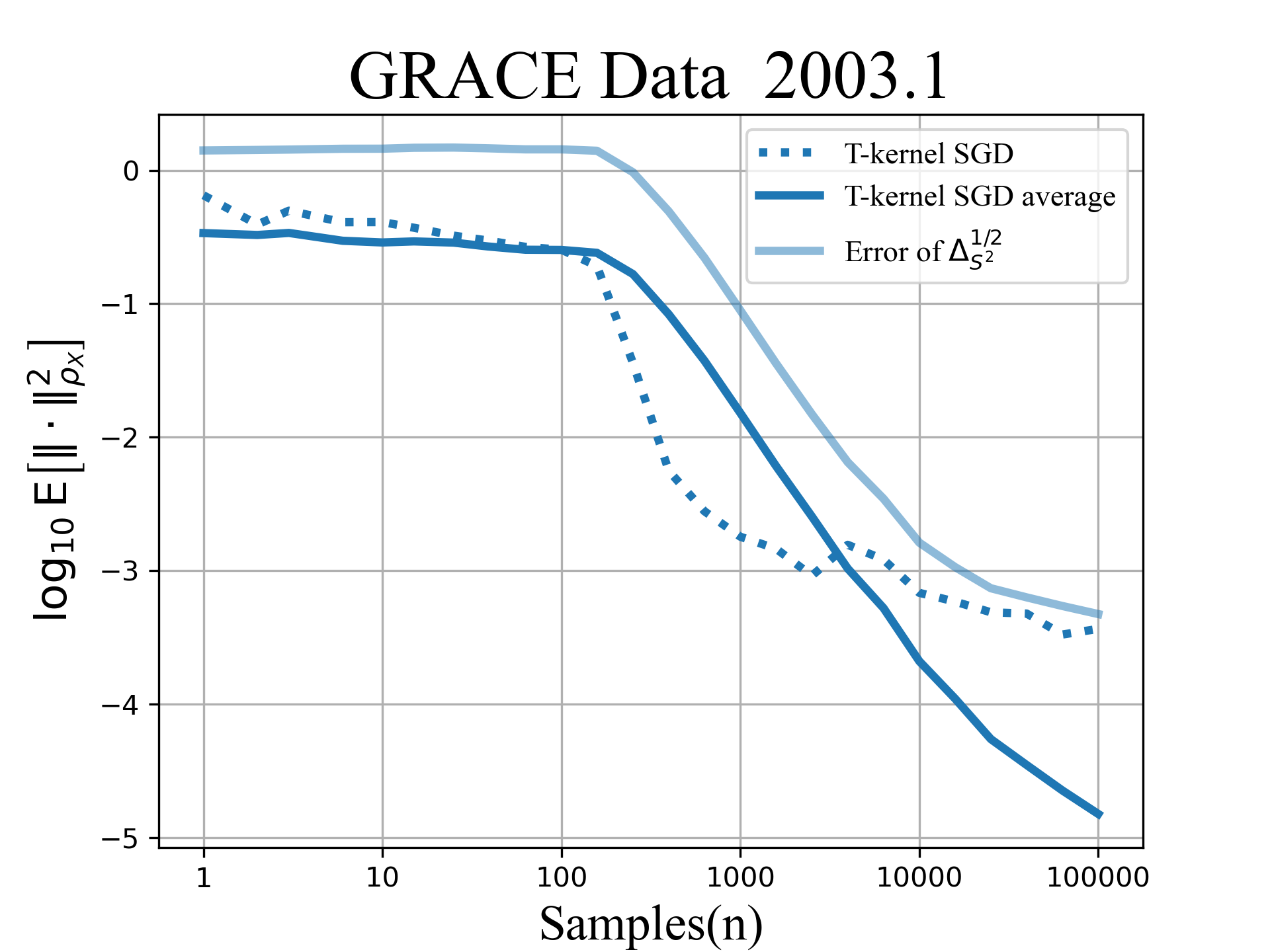}
\includegraphics[width=7.5cm]{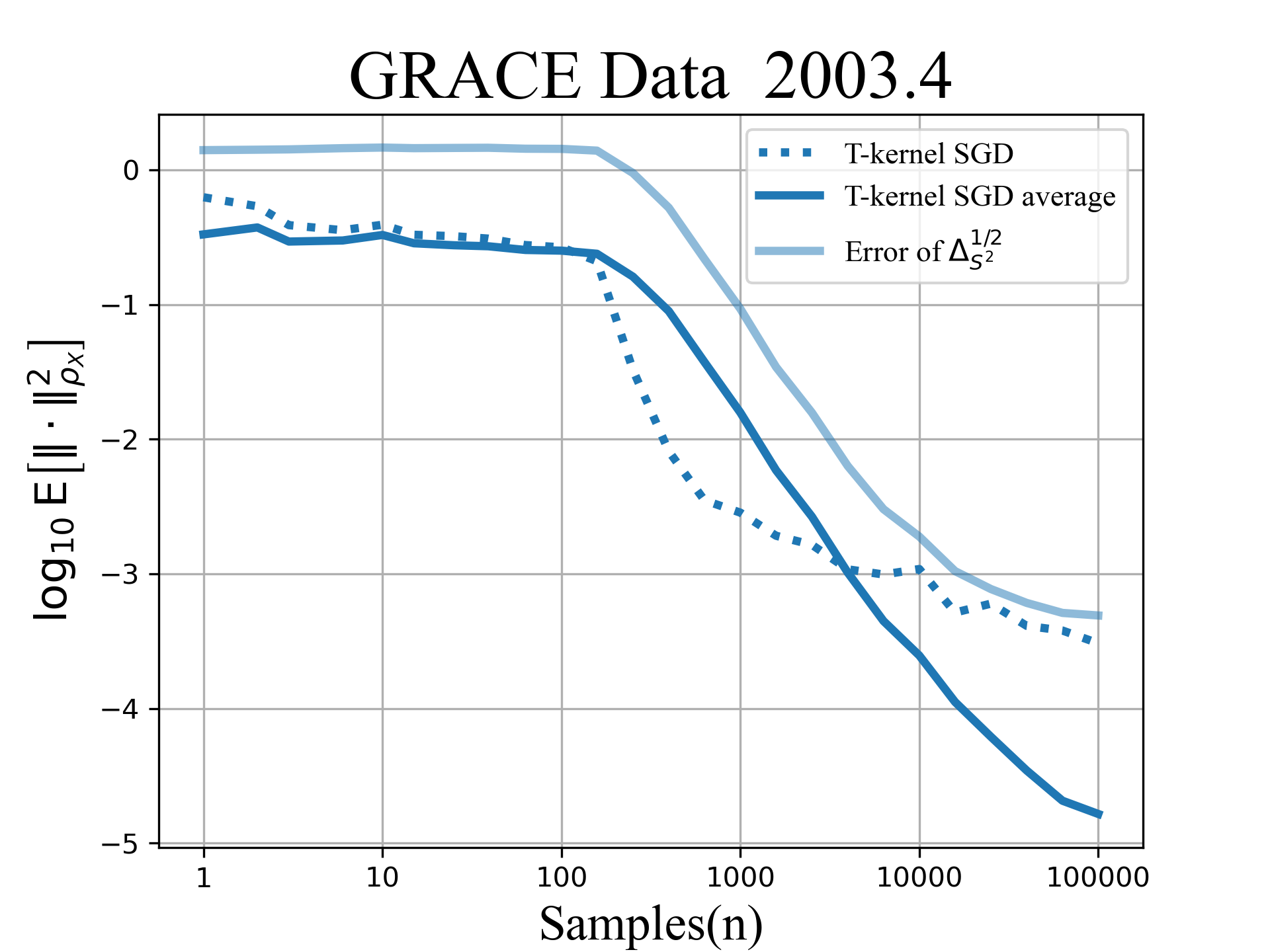}
\includegraphics[width=7.5cm]{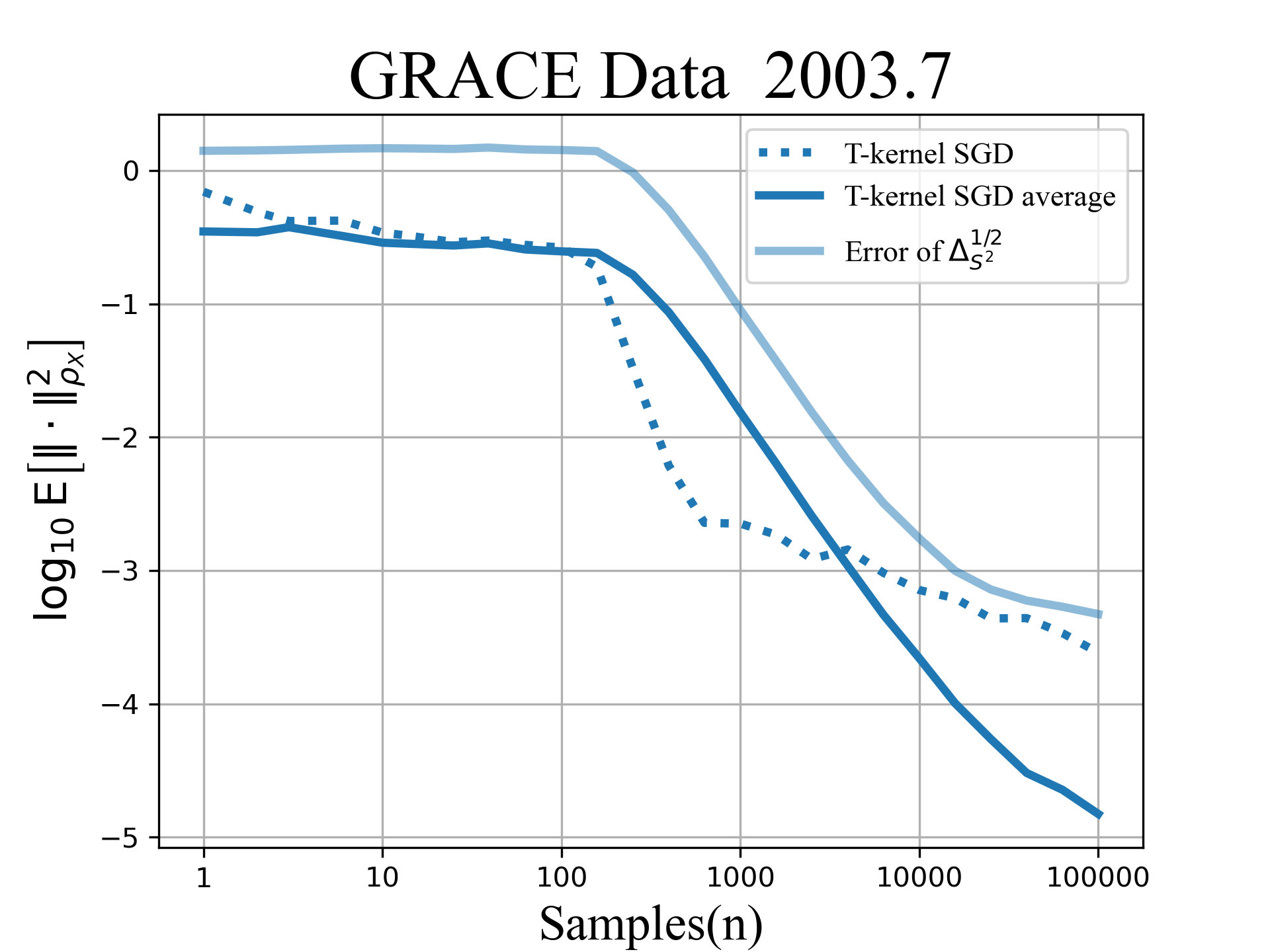}
\includegraphics[width=7.5cm]{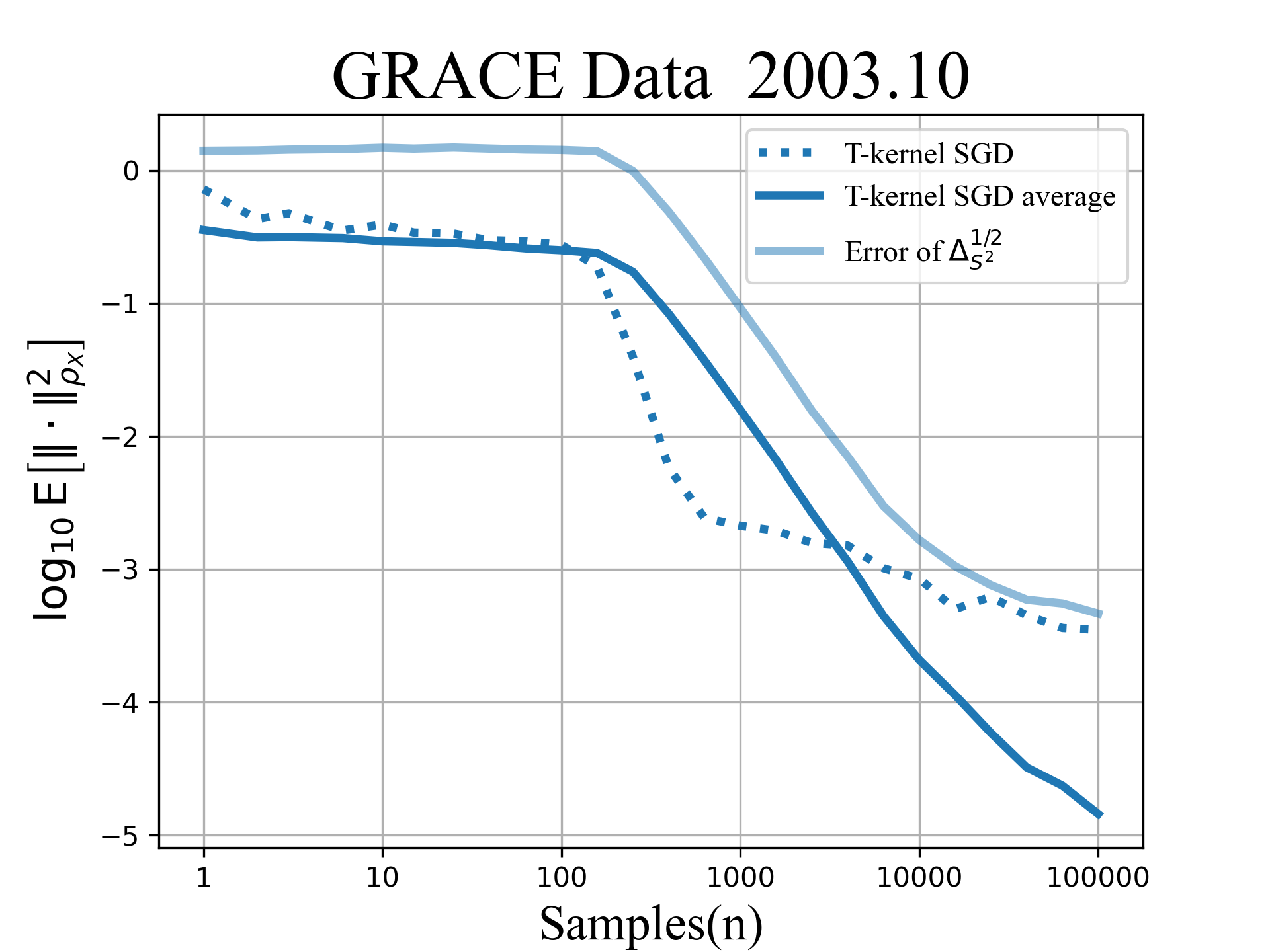}
\caption{ The above figures illustrate the convergence of the error with respect to the sample size. }
\label{figure}
\end{figure}
The empirical results in \autoref{figure} show that, for the real GRACE satellite data, the algorithm achieves not only effective convergence in the overall risk, but also convergence in the Sobolev-space sense.

\subsection{Regression with Spherical Diffusion Balance}\label{subsec:Regression with Spherical Diffusion Balance}

To illustrate the ability of T-kernel SGD to adapt to latent local differential constraints, in this section we consider a representative example in which the target function $f^*$ satisfies the Laplace–Beltrami-type constraint
$$\Delta_{\SS^{2}}f^*(x)=q(x).$$
Such a constraint may be viewed as a representative form of local differential constraints on the sphere, and it is often used as an approximate description in models of pollutant diffusion or simplified climate evolution. Through this example, we aim to demonstrate that, when the function $f^*$ satisfies an underlying differential constraint, the strong convergence property of the algorithm may still support recovery of the corresponding differential constraints, even though the constraint is not explicitly incorporated into the training procedure and learning is based solely on noisy observations. Motivated by this perspective, in addition to the population risk, we also examine the convergence of the corresponding constraint error in order to assess the ability of the algorithm to adapt to such latent local differential constraints.

Here we consider the regression model $Y=f^*(X)+\epsilon$, where $\epsilon$ follows a Gaussian distribution $\mathcal{N}(0,0.2^2)$. The target function is given by $f^*=\sum_{k=1}^{14}(k(k+1))^{-2.5}\sum_{j=1}^{\dim\HH_k^d}Y_{k,j}$ and satisfies the Laplace–Beltrami-type constraint
$$\Delta_{\SS^{2}}f^*(x)=\sum_{k=1}^{14}(k(k+1))^{-1.5}\sum_{j=1}^{\dim\HH_k^d}Y_{k,j}.$$
In implementing the T-kernel SGD algorithm, we not only adopt the theoretically prescribed choice of step size and truncation parameter corresponding to $r=1$ and $s=1$, but also consider the general hyperparameter rule proposed in \autoref{subsection:3.3}. Under this rule, the two choices of $(\gamma_n,\theta)$ are given by $\left(2n^{-\frac12},\frac16\right)$ and $\left(2n^{-\frac12},\frac14\right)$, respectively. For each parameter setting, we evaluate the Polyak average $\bar{f}_n$ in terms of both its least-squares error $\EE\left[\left\Vert\bar{f}_n-f^*\right\Vert_{\rho_X}^2\right]$ and its Laplace-Beltrami constraint error
$\EE\left[\left\Vert\Delta_{\SS^{2}}\bar{f}_n-\Delta_{\SS^{2}}f^*\right\Vert_{\rho_X}^2\right]$.

As shown in \autoref{figure A.5}, under different parameter settings, the population risk of the algorithm not only attains the theoretically predicted optimal rate, but the corresponding constraint violation also exhibits a convergence rate consistent with the strong convergence, even though the Laplace-Beltrami-type constraint is not explicitly incorporated into the algorithm. This suggests that T-kernel SGD is capable of adapting to latent local differential constraints satisfied by the function $f^*$ while relying solely on noisy observational data.

\begin{figure}[htbp]
\centering
\includegraphics[width=7.5cm]{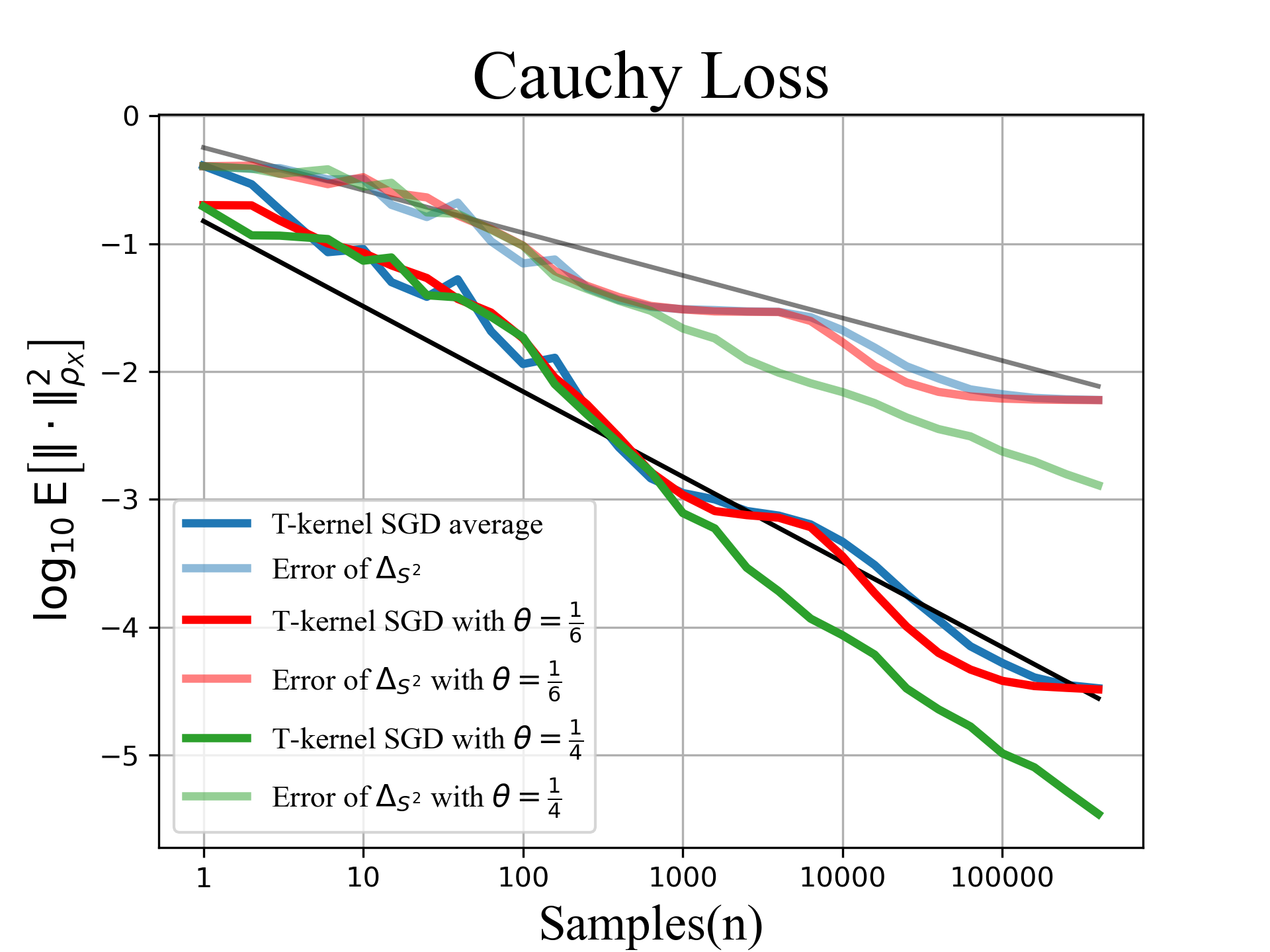}
\includegraphics[width=7.5cm]{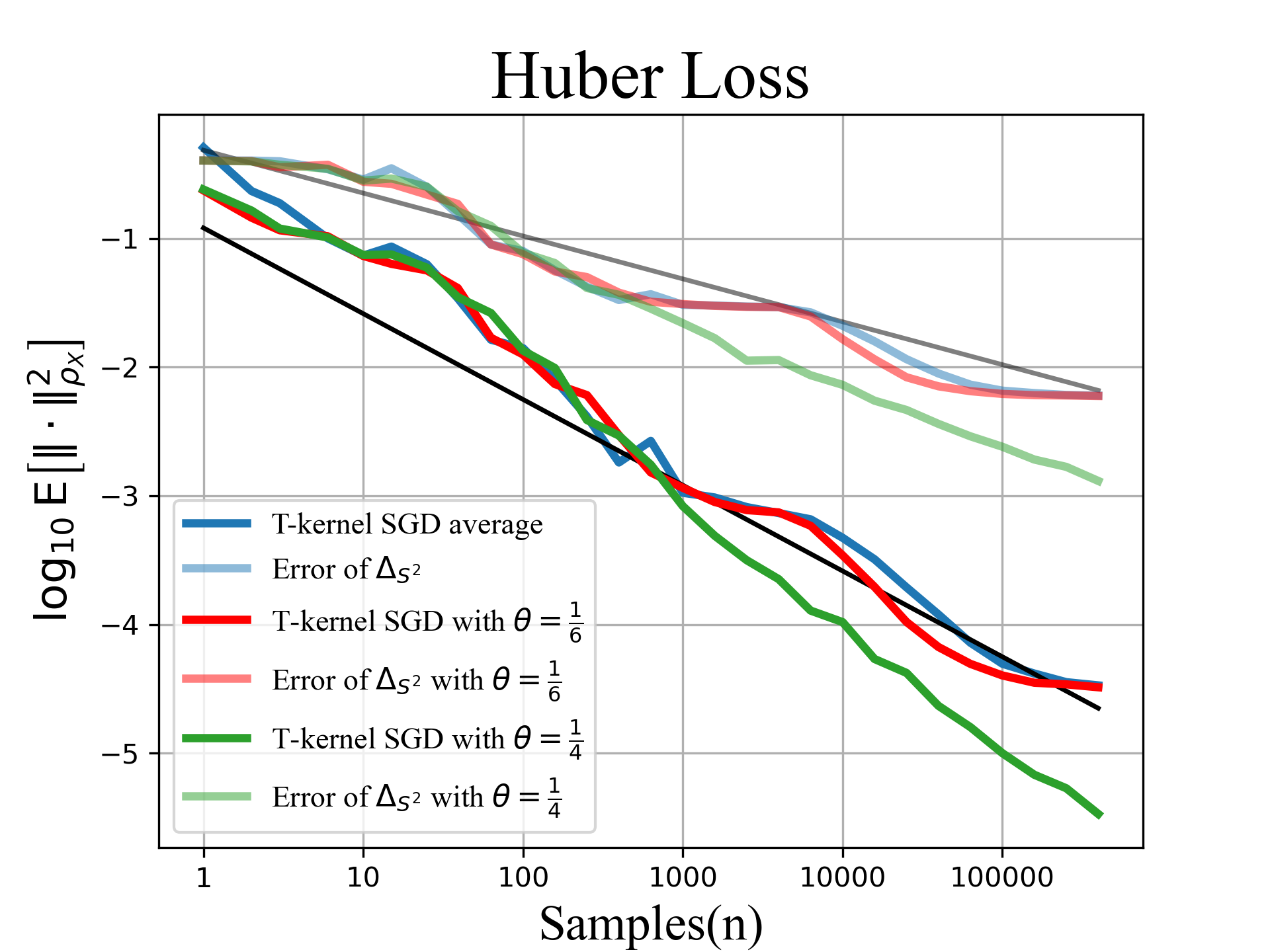}
\caption{ In the above figures, the dark curves correspond to the population risk error, the light curves correspond to the Laplace-Beltrami-type constraint error, and the black lines indicate the theoretically optimal rate.}
\label{figure A.5}
\end{figure}

\section{Conclusion}
This paper proposes a T-kernel SGD algorithm for general losses. By employing the adaptive regularization mechanism, the method attains theoretically optimal rates for both excess risk and strong convergence, while substantially reducing computational complexity and achieving optimal storage complexity. These results indicate that the algorithm provides an effective balance between statistical and computational efficiency in large-scale online nonparametric learning, while also offering theoretical support for latent physical constraints. Numerical experiments support the theory and provide empirical evidence of the computational advantages of the method. When $r$ is unknown, we recommend the $r$-independent schedule and validation-based selection of $Q$ described in \autoref{subsection:3.3}.

\bibliographystyle{plain}
\bibliography{bibliography.bib}

\appendix
\numberwithin{theorem}{section} 
\numberwithin{lemma}{section}
\numberwithin{proposition}{section}

\section{Appendix}

\subsection{Truncated Kernel Stochastic Gradient Descent with General Bases}\label{subsec:Truncated Kernel Stochastic Gradient Descent with General Basis}

At the end of \autoref{subsection:T-kernel SGD}, we discussed how to extend the algorithm to more general domains by mapping non-spherical data on $\Omega$ to the sphere via a diffeomorphism. In the present subsection,  we illustrate how the T-kernel SGD algorithm can be generalized from the kernel induced by spherical radial basis functions to kernels induced by general orthonormal bases. Let $\lambda$ be the Lebesgue measure on a compact domain $\Omega$, and let $\{\phi_j\}_{j\geq1}$ be an orthonormal basis of $\LL^2_\lambda(\Omega)$. Assume that each $\phi_j$ is continuous and that $\{\phi_j\}_{j\geq 1}$ is uniformly bounded. Uniformly bounded orthonormal systems are classical in the theory of orthogonal series \cite{kashin1989orthogonal}, and many standard bases satisfy both conditions, including the Fourier basis on $[0,1]^d$ and a broad class of eigenfunction bases arising from regular Sturm-Liouville problems \cite{fulton1994eigenvalue}. These assumptions exclude, for example, the discontinuous Haar system and the Franklin system \cite{ciesielski1963properties}, which is not uniformly bounded. Continuity and uniform boundedness are sufficient for the kernel construction and the subsequent convergence analysis. Extending these results to broader classes of orthonormal systems or, more generally, to Riesz bases is left for future work. Define the kernel function $K(x,x')=\sum_{j=1}^\infty j^{-2s}\phi_j(x)\phi_j(x')$ for $s>\frac12$. Since $\sum_{j=1}^\infty j^{-2s}<\infty$ and $\{\phi_j\}_{j\geq1}$ is uniformly bounded, the kernel $K(x,x')$ is symmetric, continuous, and positive definite, and therefore defines a Mercer kernel on $\Omega$. The corresponding reproducing kernel Hilbert space is given by 
$$\HH_K=\left\{\sum_{j=1}^\infty b_j\phi_j \,\big\vert\,\sum_{j=1}^\infty j^{2s}b_j^2<\infty\right\}$$ 
with norm $\left\Vert \sum_{j=1}^\infty b_j\phi_j\right\Vert_{\HH_K}^2:=\sum_{j=1}^\infty j^{2s}b_j^2$. Here we restrict attention to losses satisfying local strong convexity and smoothness, namely, \autoref{assum:extence}, \autoref{assum:Lsmooth}, and \autoref{assum:mustorngconvex} hold. As before, we introduce the closed convex set $\WW:=\left\{f\in\HH_K |\ \Vert f\Vert_{\HH_K}\leq Q\right\}$ with $0 < Q\cdot\sup_{x\in\Omega}|K(x,x)|^{1/2}<B$. We choose a nested sequence of finite-dimensional subspaces $\{\HH_{L_n}\}$ by setting $\HH_{L_n}=\text{span}\{\phi_1,\dots,\phi_{L_n}\}$ and $L_n=\left\lceil n^\theta \right\rceil$, where $\theta>0$ and $\left\lceil n^\theta \right\rceil$ denotes the smallest integer greater than or equal to $n^\theta$. With these preparations, we now define the truncated kernel stochastic gradient descent algorithm under a general orthonormal basis. Starting from the initialization $\hat{f}_0 = 0$, we recursively define a sequence of iterates as
\begin{equation}\label{eq:T-kernel SGD extension}
\begin{aligned}
\hat{f}_n:=&P_{\WW}\left(\hat{f}_{n-1}-\gamma_n\partial_u\ell(\hat{f}_{n-1}(X_n),Y_n)\sum_{j=1}^{L_n}j^{-2s}\phi_j(X_n)\phi_j\right).
\end{aligned}
\end{equation}
We  also employ the $\alpha$-suffix averaging scheme $\bar{f}_{\alpha n} := \frac{1}{\alpha n} \sum_{i=(1-\alpha)n}^{n-1} \hat{f}_i$ as the output of the estimator. Similarly, we define the function space
$$ \WW^{r}\left(\Omega\right)=\left\{f=\sum_{j=1}^\infty b_j\phi_j \,\Bigg\vert\,\sum_{j=1}^\infty j^{4sr}b_j^2<\infty\right\}, \quad r\geq\frac12,\ s>\frac12.$$
This space plays the role of measuring the regularity of $f^*$ in the general orthonormal basis setting. Under \autoref{assum:independnt} and the analogue of \autoref{assum:Radon-Nikodym derivative} for $\rho_X$ relative to $\lambda$, and upon replacing the original regularity assumption in \autoref{assum:regularity condition} by $f^*\in \WW^{r}\left(\Omega\right)$, the arguments used in the proofs of \autoref{theorem:mean result} and \autoref{theorem:strong convergence} extend with only minor modifications. Consequently, the truncated kernel SGD associated with a general orthonormal basis attains the theoretically optimal rate up to a logarithmic factor, namely
\begin{equation*}
\begin{aligned}
\EE\left[\ee\left(\hat{f}_n\right)-\ee\left(f^*\right)\right]&\leq \mathcal{O}\left(n^{-\frac{2r}{2r+1}}\left(\log(n+1)\right)^2\right),\\
\EE\left[\ee\left(\bar{f}_{\alpha n}\right)-\ee\left(f^*\right)\right]&\leq \O\left(n^{-\frac{2r}{2r+1}}\log(n+1)\right).
\end{aligned}
\end{equation*}
The corresponding strong convergence rate is also optimal up to a logarithmic factor
\begin{align*}
&\EE\left[\left\Vert\hat{f}_n-f^*\right\Vert_{K}^2\right]\leq\mathcal{O}\left(n^{-\frac{2r-1}{2r+1}}\left(\log(n+1)\right)^2\right).
\end{align*}
Previously, \cite{zhang2022sieve} established the optimality of sieve-SGD under a general orthonormal basis in the least-squares setting. Here, we show that the corresponding truncated kernel SGD under a general orthonormal basis continues to attain optimal rates for general loss functions. This further demonstrates that our framework is not restricted to spherical kernels, but extends naturally to reproducing kernel Hilbert spaces induced by continuous and uniformly bounded orthonormal basis functions.

\subsection{Preliminaries}\label{appendix pre}
In this section, we present the explicit expressions of spherical harmonics required for the algorithmic implementation, along with several auxiliary lemmas and their proofs used in the main text.

\subsubsection{Orthonormal Basis of the Spherical Harmonic Space}\label{subsection:Orthonormal Basis of the Spherical Harmonic Space}

For the $d$-dimensional unit sphere $\SS^{d-1}$, we consider the spherical harmonic space $\HH_k^d$ with $k \geq 0$. Let $\alpha=(\alpha_1,\dots,\alpha_{d-1}) \in \NN^{d-1}$ be a multi-index satisfying $|\alpha|=\alpha_1+\dots+\alpha_{d-1}=k$. We define
$\lambda_j=\tfrac{d-j-1}{2}+\sum_{i=j+1}^{d-1}\alpha_i$. For a point $x=(x_1,\dots,x_d)\in\SS^{d-1}$, an orthonormal basis of $\HH_k^d$ is given by
\begin{align*}
Y_{\alpha,0} &= h_{\alpha,0}\cdot g(x)\prod_{j=1}^{d-2}\left(x_1^2+\dots+x_{d-j+1}^2\right)^{\alpha_j/2}C_{\alpha_j}^{\lambda_j}\left(\frac{x_{d-j+1}}{\sqrt{x_1^2+\dots+x_{d-j+1}^2}} \right),\  \text{where   }\alpha_{d-1}\geq0,\\
Y_{\alpha,1} &= h_{\alpha,1}\cdot g(x)\prod_{j=1}^{d-2}\left(x_1^2+\dots+x_{d-j+1}^2\right)^{\alpha_j/2}C_{\alpha_j}^{\lambda_j}\left(\frac{x_{d-j+1}}{\sqrt{x_1^2+\dots+x_{d-j+1}^2}} \right),\  \text{where   }\alpha_{d-1}\geq1.
\end{align*}
Here, $h_{\alpha,i}$ with $i=0,1$ are normalization constants, and $C_k^\lambda(u)$ denotes the Gegenbauer polynomial, which satisfies $C_0^\lambda(u)=1$, $C_1^\lambda(u)=2\lambda u$, and the following three-term recurrence relation:
$$C_{k+1}^\lambda(u)=\frac{2(k+\lambda)}{k+1}uC_k^\lambda(u)-\frac{k+2\lambda-1}{k+1}C_{k-1}^\lambda(u).$$
For further properties of the Gegenbauer polynomials, we refer the reader to Appendix B.2 of \cite{dai2013approximation}. For $Y_{\alpha,0}$, the function $g(x)$ corresponds to the real part of $(x_2+\sqrt{-1}\cdot x_1)^{\alpha_{d-1}}$, whereas for $Y_{\alpha,1}$, $g(x)$ corresponds to the imaginary part of $(x_2+\sqrt{-1}\cdot x_1)^{\alpha_{d-1}}$. The normalization constant $h_{\alpha,i}$ satisfies:
\begin{align*}
h_{\alpha,i}^{-2}=
          \left\{
          \begin{aligned}
              &\frac{\pi}{\Omega_{d-1}}\prod_{j=1}^{d-2}\frac{\pi 2^{1-2\lambda_j}\Gamma(\alpha_j+2\lambda_j)}{\alpha_j!(\lambda_j+\alpha_j)\left(\Gamma(\lambda_j)\right)^2},\ \ \text{if}\ \ \alpha_{d-1}>0, \\
              & 
              \frac{2\pi}{\Omega_{d-1}}\prod_{j=1}^{d-2}\frac{\pi 2^{1-2\lambda_j}\Gamma(\alpha_j+2\lambda_j)}{\alpha_j!(\lambda_j+\alpha_j)\left(\Gamma(\lambda_j)\right)^2},\ \ \text{if}\ \ \alpha_{d-1}=0.
          \end{aligned}
          \right.
\end{align*}
Here, $\Gamma(u)$ represents the Gamma function. For a more detailed discussion of the orthonormal basis, we refer the reader to \cite{dai2013approximation}.

Next, we discuss the computational complexity of the basis functions. Since $0 \leq \alpha_j \leq k$ and the quantities $\{x_1^2, x_1^2+x_2^2, \dots, x_1^2+\dots+x_d^2\}$ can be computed recursively, evaluating $\left(x_1^2+\dots+x_{d-j+1}^2\right)^{\alpha_j/2}$ requires at most $\O(k)$ computational time. Moreover, using the three-term recurrence relation of the Gegenbauer polynomials, computing $C_{\alpha_j}^{\lambda_j}\left(u\right)$ also requires at most $\O(k)$ computational time. Therefore, the computation of a basis function $Y_{\alpha,i} \in \HH_k^d$ requires at most $\O(dk)$ computational time. Although the computational complexity of each basis function appears to increase with the data dimension, in high-dimensional settings (e.g., $d>100$) we typically only use second- or third-order polynomials. In such cases, the expressions of the orthonormal basis can be considerably simplified. The orthonormal bases for $\HH_0^d$, $\HH_1^d$, and $\HH_2^d$ are given below
\begin{align*}
\HH_0^d=&\text{span}\{1\},\\
\HH_1^d=&\text{span}\{\sqrt{d}x_i\}_{1\leq i\leq d},\\
\HH_2^d=&\text{span}\Bigg(\left\{\sqrt{d(d+2)}x_ix_j\right\}_{1\leq i<j\leq d}\cup\left\{\frac{\sqrt{d(d+2)}}{2}(x_1^2-x_2^2)\right\}\\
&\quad\quad\cup\left\{ h_j\cdot \left( x_1^2 + \dots + x_{d-j+1}^2\right)C_{2}^{\lambda_j}\left(\frac{x_{d-j+1}}{\sqrt{ x_1^2 + \dots + x_{d-j+1}^2} }\right)\right\}_{1\leq j\leq d-2}\Bigg).
\end{align*}
From the perspective of the orthonormal bases of $\HH_0^d$, $\HH_1^d$, and $\HH_2^d$, the $\mathcal{O}(dk)$ computational time should be interpreted as an analytical upper bound, and the dimension factor $d$ may not be necessary in practice. Indeed, since the coefficients of the polynomials $C_2^{\lambda_j}(u) = 2\lambda_j(\lambda_j + 1)u^2 - \lambda_j$ are explicitly known, evaluating each basis function in $\HH_0^d$, $\HH_1^d$, and $\HH_2^d$ requires at most 1, 2, and 10 operations, respectively, regardless of the dimension $d$. The expression of the constant $h_j$ is given by:
$$h_j^{-2}=\pi^{1/2}\frac{(d-j)\left[(d-j)^2-1\right]}{d(d+2)}\frac{\Gamma(\frac{d-j+1}{2})}{\Gamma(\frac{d-j}{2})}\frac{2^{1-(d-j-1)}\Gamma(d-j-1)}{\left(\Gamma\left(\frac{d-j-1}{2}\right)\right)^2}.$$
Since the Gamma function becomes computationally challenging in high dimensions, we consider simplifying the above expression using Poincar\'{e}-type expansions (see 5.11(i) in \cite{NIST2010Olver}) and the ratio of two Gamma functions (see \cite{frenzen1992error}).
\begin{align*}
2^{1-2\lambda}\frac{\Gamma(2\lambda)}{\left(\Gamma(\lambda)\right)^2}&\sim\frac{\lambda^{1/2}}{\pi^{1/2}}\frac{\Gamma^*(2\lambda)}{\left(\Gamma^*(\lambda)\right)^2},\quad \text{where}\ \ \Gamma^*(\lambda)\sim\sum_{k=0}^\infty \frac{g_k}{\lambda^k},\\
\frac{\Gamma\left(\lambda+\frac{1}{2}\right)}{\Gamma(\lambda)}&\sim\sum_{j=0}^\infty \frac{\left(2j-1-\frac{1}{2}\right)\left(2j-2-\frac{1}{2}\right)\dots\left(-\frac{1}{2}\right)}{(2j)!}B_{2j}^{(3/2)}\left(\frac{3}{4}\right)\left(\lambda-\frac{1}{4}\right)^{1/2-2j}.
\end{align*}
Here, $B_{2j}^{(3/2)}\left(\frac{3}{4}\right)$ represents the generalized Bernoulli polynomials, as detailed in \cite{frenzen1992error}.
\subsubsection{Lemmas}

\begin{proposition}\label{kernel uniformly convergence}
If $a_k>0$ and the finite limit $l:=\lim_{k\to\infty}a_k(\dim\Pi_k^d)^{2s}$ exists for some $s>\frac12$, then the spherical radial basis function $$K(x,x')=\sum_{k=0}^\infty a_kK_k(x,x')$$ defined in \eqref{eq:def kernel K} converges uniformly and is uniformly bounded.
\end{proposition}
\begin{proof}
By Corollary 1.2.7 in \cite{dai2013approximation}, we have $|K_k(x, x')| \leq \dim \mathcal{H}_k^d$ for $x,x'\in\SS^{d-1}$. Furthermore, according to Corollaries 1.1.5 and 1.1.4 in \cite{dai2013approximation}, we obtain
\begin{equation*}
\begin{aligned}
&\dim{\HH_k^{d}}=\dim{\PP_k^{d}}-\dim{\PP_{k-2}^{d}}=\binom{k+d-1}{d-1}-\binom{k+d-3}{d-1},\\
&\dim{\Pi_k^{d}}=\dim{\PP_k^{d}}+\dim{\PP_{k-1}^{d}}=\binom{k+d-1}{d-1}+\binom{k+d-2}{d-1}.
\end{aligned}
\end{equation*}
For $k\geq1$, $\dim \Pi_k^d$ satisfies the following relation 
$$\dim \Pi_k^d \geq \binom{k+d-1}{d-1}=\frac{(k+d-1)\dots(k+1)}{(d-1)!}\geq\frac{k^{d-1}}{(d-1)!}.$$
Since $\lim_{k \to \infty} a_k \cdot \left(\dim \Pi_k^d\right)^{2s} = l$, it follows that there exists a constant $M > 0$ such that $0 < a_k < M \left(\dim \Pi_k^d\right)^{-2s}$. For $x,x'\in\SS^{d-1}$, we obtain
$$|K(x,x')|\leq\sum_{k=0}^\infty a_k|K_k(x,x')|\leq a_0 + M\left((d-1)!\right)^{2s}\sum_{k=1}^\infty k^{-2s(d-1)}\dim{\HH_k^{d}}.$$
If $d = 2$, then $\dim \HH_k^d = 2$. In this case, when $s > \frac{1}{2}$, uniform convergence follows directly from the Weierstrass M-test, since
$$|K(x,x')|\leq a_0+M\left((d-1)!\right)^{2s}\sum_{k=1}^\infty 2k^{-2s}.$$
If $d \geq3$, then
\begin{equation*}
\begin{aligned}
&\dim{\HH_k^{d}}=\binom{k+d-1}{d-1}-\binom{k+d-3}{d-1}\\
=&\frac{(k+d-1)(k+d-2)-k(k-1)}{(d-1)!}(k+d-3)\dots(k+1)\\
=&\frac{(d-1)(2k+d-2)(k+d-3)\dots(k+1)}{(d-1)!}\leq 2(k+1)^{d-2}.
\end{aligned}
\end{equation*}
In this case, when $s > \frac{1}{2}$, uniform convergence follows directly from the Weierstrass M-test, since
$$|K(x,x')|\leq a_0+M\left((d-1)!\right)^{2s}\sum_{k=1}^\infty  2(k+1)^{d-2}k^{-2s(d-1)}\leq a_0 + 2^{d-1}M\left((d-1)!\right)^{2s}\sum_{k=1}^\infty k^{-2s}.$$
The proposition then follows.
\end{proof}

Before proving results related to the Fr\'{e}chet derivative of the population risk, we first introduce a necessary preliminary.
\begin{proposition}\label{prop1}
If \autoref{assum:extence} and \autoref{assum:Lsmooth} hold, then the loss function $\ell(u,v)$ satisfies the following condition uniformly with respect to its second argument $v$: for every $\epsilon>0$, there exists $\delta_\epsilon>0$ such that, for all $(u,v)\in(-B,B)\times\mathcal{Y}$ and all $u'$ satisfying $0<|u'|<\delta_\epsilon$ and $u+u'\in[-B,B]$, we have
\begin{equation}\label{eq:bound guarantee}
\left|\frac{\ell(u+u',v)-\ell(u,v)}{u'}-\partial_u\ell(u,v)\right|<\epsilon .
\end{equation}
\end{proposition}

\begin{proof}
For any $\epsilon > 0$, choose $\delta_\epsilon = \frac{\epsilon}{L} > 0$. Then, for any $(u_1, v), (u_2, v) \in [-B, B] \times \mathcal{Y}$ such that  $|u_1-u_2|<\delta_\epsilon$, we have $|\partial_u\ell(u_1,v)-\partial_u\ell(u_2,v)|\leq L|u_1-u_2|<\epsilon$. For any fixed $v \in \mathcal{Y}$, by the Lagrange mean value theorem, if $|u'| < \delta_\epsilon$ and $u, u + u' \in [-B, B]$, then there exists $\eta \in (0, 1)$ such that
$$\ell(u+u',v)-\ell(u,v)=\partial_u\ell(u+\eta u',v)u',$$
and hence
$$\left|\frac{\ell(u+u',v)-\ell(u,v)}{u'}-\partial_u\ell(u,v)\right|\leq \left|\partial_u\ell(u+\eta u',v)-\partial_u\ell(u,v)\right|\leq L|\eta u'|<\epsilon.$$

\end{proof}

\begin{lemma}\label{lemma:A.1}
If \autoref{assum:extence} and \autoref{assum:Lsmooth} hold and we choose $f\in\WW$, then the Fr\'{e}chet derivative of the population risk $\ee(f)$ can be expressed as follows:
$$\nabla\ee(f)\big|_{\HH_K}=\EE\left[ \partial_u\ell(f(X),Y)K(X,\cdot)\right].$$
\end{lemma}
\begin{proof}
By \autoref{prop1}, for any $\epsilon>0$, choose $h\in\HH_K$ such that $\Vert h\Vert_\infty\leq \Vert h\Vert_K\kappa<\delta_\epsilon$ and $\Vert h\Vert_K\leq B\frac{1}{\kappa} - \Vert f\Vert_K$. Then, for any $Y\in\YY$, we have
\begin{equation}\label{eq:A.1.1}
                \begin{aligned}
&\left|\ell(f(X)+h(X),Y)-\ell(f(X),Y)-\partial_u\ell(f(X),Y)h(X)\right|\leq\epsilon |h(X)|\\
\leq&\epsilon\Vert h\Vert_K\Vert K(X,\cdot)\Vert_K\leq\epsilon\Vert h\Vert_K\kappa.
                \end{aligned} 
            \end{equation}
Taking expectations on both sides of \eqref{eq:A.1.1} and applying Jensen's inequality, we obtain
\begin{equation*}
\begin{aligned}
&\left|\ee(f+h)-\ee(f)-\EE\left[\partial_u\ell(f(X),Y)h(X)\right]\right|\\
=&\left|\EE\left[\ell(f(X)+h(X),Y)-\ell(f(X),Y)-\partial_u\ell(f(X),Y)h(X)\right]\right|\\
\leq&\EE\left[\left|\ell(f(X)+h(X),Y)-\ell(f(X),Y)-\partial_u\ell(f(X),Y)h(X)\right|\right]\\
\leq&\epsilon\Vert h\Vert_K\kappa.
\end{aligned}
\end{equation*}
Using the reproducing property, one can obtain
\begin{equation*}
\begin{aligned}
&\ee(f+h)-\ee(f)-\EE\left[\lan\partial_u\ell(f(X),Y)K(X,\cdot),h\ran_K\right]\\
=&\ee(f+h)-\ee(f)-\lan\EE\left[\partial_u\ell(f(X),Y)K(X,\cdot)\right],h\ran_K=o(\Vert h\Vert_K).
\end{aligned}
\end{equation*}
Finally, by using the definition of the Fr\'{e}chet derivative \cite{ciarlet2013linear}, we complete the proof by obtaining
$$\nabla\ee(f)\big|_{\HH_K}=\EE\left[ \partial_u\ell(f(X),Y)K(X,\cdot)\right].$$
\end{proof}

\begin{lemma}\label{lemma:A.2}
If \autoref{assum:extence} and \autoref{assum:Lsmooth} hold and we choose $f \in \WW \cap \HH_{L_n}$ with $L_n \in \NN$, then the Fr\'{e}chet derivative of the population risk $\ee(f)$ in the RKHS $(\HH_{L_n}, \langle \cdot, \cdot \rangle_K)$ is given by
$$\nabla\ee(f)\big|_{\HH_{L_n}}=\EE\left[ \partial_u\ell(f(X),Y)K_{L_n}^T(X,\cdot)\right].$$
We also have
\begin{equation*}
P_{\HH_{L_n}}\left(\partial_u\ell(f(X_n),Y_n)K(X_n,\cdot)\right)=\partial_u\ell(f(X_n),Y_n)K_{L_n}^T(X_n,\cdot).
\end{equation*}
\end{lemma}
\begin{proof}
As in the proof of \autoref{lemma:A.1}, by \autoref{prop1}, for any $\epsilon > 0$, choose $h \in \HH_{L_n}$ such that $\Vert h\Vert_\infty \leq \Vert h\Vert_K \kappa < \delta_\epsilon$ and $\Vert h\Vert_K \leq B\frac{1}{\kappa} - \Vert f\Vert_K$. Then, for any $Y \in \YY$, we have
\begin{equation*}
\left|\ell(f(X)+h(X),Y)-\ell(f(X),Y)-\partial_u\ell(f(X),Y)h(X)\right|\leq\epsilon |h(X)|\leq\epsilon\Vert h\Vert_K\kappa.
\end{equation*}
Similarly, we have
\begin{align*}
&\left|\ee(f+h)-\ee(f)-\EE\left[\partial_u\ell(f(X),Y)h(X)\right]\right|\\
\leq&\EE\left[\left|\ell(f(X)+h(X),Y)-\ell(f(X),Y)-\partial_u\ell(f(X),Y)h(X)\right|\right]\\
\leq&\epsilon\Vert h\Vert_K\kappa.
\end{align*}
Using the reproducing property, one can obtain
\begin{align*}
\ee(f+h)-\ee(f)-\lan\EE\left[\partial_u\ell(f(X),Y)K_{L_n}^T(X,\cdot)\right],h\ran_K=o(\Vert h\Vert_K).
\end{align*}
By using the definition of the Fr\'{e}chet derivative \cite{ciarlet2013linear}, we have
$$\nabla\ee(f)\big|_{\HH_{L_n}}=\EE\left[ \partial_u\ell(f(X),Y)K_{L_n}^T(X,\cdot)\right].$$
By the definition of the kernel function $K(x, x')$, we have
$$K(X_n,\cdot)=\sum_{k=0}^\infty a_kK_k(X_n,\cdot)=K_{L_n}^T(X_n,\cdot)+\sum_{k=L_n+1}^\infty a_kK_k(X_n,\cdot).$$
Since each $K_k(X_n, \cdot) \in \mathcal{H}_k^d \subset \mathcal{H}_{L_n}^\perp$ for $k \geq L_n + 1$, it follows that $\sum_{k=L_n+1}^\infty a_kK_k(X_n,\cdot)\in \mathcal{H}_{L_n}^\perp$. Therefore, the conclusion holds by the uniqueness of the orthogonal decomposition.
\end{proof}

\begin{lemma}\label{lemma:A.3}
If \autoref{assum:extence} and \autoref{assum:Lsmooth} hold and the optimal function $f^*$ is an interior point of $\WW$, i.e., $\Vert f^* \Vert_K < Q$, then for any $f \in \WW$, we have
$$\ee(f)-\ee(f^*)\leq \frac{L}{2}\left\Vert f\circ F-f^*\circ F\right\Vert_{\rho_X}^2,$$
where $L$ is the Lipschitz constant defined in \autoref{assum:Lsmooth}.
\end{lemma}

\begin{proof}
Fix any $v \in \YY$, and define a function $l(u) = \ell(u, v)$ on $[-B, B]$. Then $l(u)$ is $L$-smooth and satisfies
$$|l'(u_1)-l'(u_2)|=\left|\partial_u\ell(u_1,v)-\partial_u\ell(u_2,v)\right|\leq L|u_1-u_2|$$
for $u_1,u_2\in[-B,B]$. Then $l(u)$ satisfies the quadratic upper bound in Theorem 2.1.5 of \cite{nesterov2013introductory}, i.e.
\begin{equation}\label{eq:A.3}
                \begin{aligned}
l(u_2)&\leq l(u_1)+l'(u_1)(u_2-u_1)+\frac{L}{2}(u_1-u_2)^2,\\
\Rightarrow\ \ \ell(u_2,v)&\leq\ell(u_1,v)+\partial_u\ell(u_1,v)(u_2-u_1)+\frac{L}{2}(u_1-u_2)^2.
                \end{aligned} 
            \end{equation}
            
In addition, by substituting $f\in\WW$ and $f^*$ into \eqref{eq:A.3} and taking expectations on both sides, we obtain
\begin{equation*}
\begin{aligned}
&\EE\left[\ell(f\circ F(X),Y)\right]-\EE\left[\ell(f^*\circ F(X),Y)\right]\\
\leq&\EE\left[\partial_u\ell(f^*\circ F(X),Y)(f\circ F(X)-f^*\circ F(X))\right]+\frac{L}{2}\EE\left[(f\circ F(X)-f^*\circ F(X))^2\right]\\
=&\lan\EE\left[\partial_u\ell(f^*\circ F(X),Y)K(F(X),\cdot)\right], f-f^*\ran_K+\frac{L}{2}\Vert f\circ F-f^*\circ F\Vert_{\rho_X}^2\\
=&\lan\nabla\ee(f^*)|_{\HH_K},f-f^*\ran_K+\frac{L}{2}\Vert f\circ F-f^*\circ F\Vert_{\rho_X}^2\\
\overset{\text{(i)}}{=}&\frac{L}{2}\Vert f\circ F-f^*\circ F\Vert_{\rho_X}^2.
\end{aligned}
\end{equation*}
Since $f^*$ is an interior point and by Theorem 7.1-5 in \cite{ciarlet2013linear}, we have $\nabla\ee(f^*)|_{\HH_K}=0$, which justifies equality (i). Therefore, we complete the proof by obtaining
$$\ee(f)-\ee(f^*)\leq \frac{L}{2}\left\Vert f\circ F-f^*\circ F\right\Vert_{\rho_X}^2.$$
\end{proof}

\begin{lemma}\label{lemma:A.4}
Let $\WW$ be defined as in \eqref{eq:WW defintion}, and denote by $P_\WW: \HH_K \rightarrow \WW$ the projection operator onto $\WW$. Then, for any $f \in \HH_{L_n}$, we have $P_\WW(f) \in \HH_{L_n} \cap \WW$.
\end{lemma}

\begin{proof}
Note that the orthogonal complement of $\HH_{L_n}$ in $\HH_K$ is $\HH_{L_n}^\perp$. Hence, the projection $P_\WW(f)$ admits an orthogonal decomposition of the form $P_\WW(f) = f_1 + f_2$ with $f_1 \in \HH_{L_n}$ and $f_2 \in \HH_{L_n}^\perp$. If $f \in \HH_{L_n}$ and $f_2 \neq 0$, then one has
\begin{equation}\label{eq:A.4}
                \begin{aligned}
\min_{g\in\WW}\Vert g-f\Vert_K^2&=\Vert P_\WW(f)-f\Vert_K^2\\
&=\Vert (f_1+f_2)-f\Vert_K^2=\Vert f_1-f\Vert_K^2+\Vert f_2\Vert_K^2>\Vert f_1-f\Vert_K^2.
                \end{aligned} 
            \end{equation}
Since $\Vert f_1\Vert_K\leq\Vert P_\WW(f)\Vert_K\leq Q$, it follows that $f_1\in\WW$. This implies that \eqref{eq:A.4} contradicts the definition of the projection operator $P_\WW$, and hence $f_2=0$, which further implies $P_\WW(f)\in \HH_{L_n}$.

\end{proof}

\begin{lemma}\label{lemma:A.5}
If \autoref{assum:extence}, \autoref{assum:Lsmooth} and \autoref{assum:mustorngconvex} hold, then $\ee$ is convex on $\WW$. For $f,g\in\WW$, we have
$$\ee(g)-\ee(f)-\lan \nabla\ee(f)\big|_{\HH_K},g-f\ran_K\geq\frac{\mu}{2}\Vert g\circ F-f\circ F\Vert_{\rho_X}^2.$$
\end{lemma}
\begin{proof}
For any $f, g \in \WW$, the local $\mu$-strong convexity of $\ell(u,v)$ implies that
\begin{equation*}
\begin{aligned}
&\ \ \ \EE\left[\ell(g\circ F(X),Y)-\ell(f\circ F(X),Y)-\partial_u\ell(f\circ F(X),Y)(g\circ F(X)-f\circ F(X))\right]\\
&\geq\frac{\mu}{2}\EE\left[(g\circ F(X)-f\circ F(X))^2\right]\\
\Rightarrow\ \ &\ee(g)-\ee(f)-\lan \EE\left[\partial_u\ell(f\circ F(X),Y)K(F(X),\cdot)\right],g-f\ran_K\geq\frac{\mu}{2}\Vert g\circ F-f\circ F\Vert_{\rho_X}^2\\
\Rightarrow\ \ &\ee(g)-\ee(f)-\lan \nabla\ee(f)\big|_{\HH_K},g-f\ran_K\geq\frac{\mu}{2}\Vert g\circ F-f\circ F\Vert_{\rho_X}^2\geq0.
\end{aligned}
\end{equation*}
Thus, $\mathcal{E}(f)$ is convex by Section 7.12-1 in \cite{ciarlet2013linear}, and the proof is complete.
\end{proof}

\begin{lemma}\label{lemma:A.6}
If $f\in\HH_{L_n}$ and is represented as $f=\sum_{k=0}^{L_n}\sum_{j=1}^{\dim\HH_k^d}f_{k,j}Y_{k,j}$, then we have
\begin{equation}
P_\WW(f)=
          \left\{
          \begin{aligned}
              &\frac{Q}{\Vert f\Vert_K}f=\frac{Q}{\left(\sum_{k=0}^{L_n}\sum_{j=1}^{\dim\HH_k^d}a_k^{-1}f_{k,j}^2\right)^{\frac{1}{2}}}f,\ \text{if}\ \Vert f\Vert_K>Q \\
              & f,\ \text{if}\ \Vert f\Vert_K\leq Q. 
          \end{aligned}
          \right.
      \end{equation}

\end{lemma}
\begin{proof}
By the definition of the projection operator $P_\WW$, we have
$$P_\WW(f)=\arg\min_{g\in\WW}\Vert f-g\Vert_K^2.$$
Furthermore, by \autoref{lemma:A.4}, we know that for any $f\in\HH_{L_n}$, the projection $P_\WW(f)\in\HH_{L_n}$. Hence, the problem reduces to
\begin{equation}\label{eq:A.5}
                \begin{aligned}
\min_{g\in\HH_{L_n}}\ \ &\frac{1}{2}\Vert f-g\Vert_K^2\\
\text{s.t.}\ \ &\frac{1}{2}\Vert g\Vert_K^2\leq \frac{1}{2}Q^2.
                \end{aligned} 
            \end{equation}
Using the generalized Fourier expansions of $f$ and $g$ with respect to the orthonormal basis $\{Y_{k,j}\}$, we can transform \eqref{eq:A.5} into the finite-dimensional convex optimization problem given in \eqref{eq:A.6}. If we assume $g=\sum_{k=0}^{L_n}\sum_{j=1}^{\dim\HH_k^d}g_{k,j}Y_{k,j}$, then
\begin{equation}\label{eq:A.6}
                \begin{aligned}
\min_{g\in\HH_{L_n}}\ \ &\frac{1}{2}\Vert f-g\Vert_K^2=\frac{1}{2}\sum_{k=0}^{L_n}a_k^{-1}\sum_{j=1}^{\dim\HH_k^d}\left(g_{k,j}-f_{k,j}\right)^2\\
\text{s.t.}\ \ &\frac{1}{2}\Vert g\Vert_K^2=\frac{1}{2}\sum_{k=0}^{L_n}a_k^{-1}\sum_{j=1}^{\dim\HH_k^d}\left(g_{k,j}\right)^2\leq \frac{1}{2}Q^2.
                \end{aligned} 
            \end{equation}
For $\lambda \geq 0$, the Lagrangian corresponding to \eqref{eq:A.6} is given by
\begin{equation*}
                \begin{aligned}
L(g,\lambda)=&\frac{1}{2}\Vert f-g\Vert_K^2+\frac{\lambda}{2}\left(\Vert g\Vert_K^2-Q^2\right)\\
=&\frac{1}{2}\sum_{k=0}^{L_n}a_k^{-1}\sum_{j=1}^{\dim\HH_k^d}\left(g_{k,j}-f_{k,j}\right)^2+\frac{\lambda}{2}\left(
\sum_{k=0}^{L_n}a_k^{-1}\sum_{j=1}^{\dim\HH_k^d}\left(g_{k,j}\right)^2-Q^2\right).
                \end{aligned} 
            \end{equation*}
The KKT condition can be obtained as follows
\begin{equation*}
          \left\{
          \begin{aligned}
              &\frac{\partial L}{\partial g_{k,j}}=a_k^{-1}\left(g_{k,j}-f_{k,j}\right)+\lambda a_k^{-1}g_{k,j}=0, \\
              &\lambda\left(\sum_{k=0}^{L_n}a_k^{-1}\sum_{j=1}^{\dim\HH_k^d}\left(g_{k,j}\right)^2-Q^2\right)=0, \\
              &\frac{1}{2}\sum_{k=0}^{L_n}a_k^{-1}\sum_{j=1}^{\dim\HH_k^d}\left(g_{k,j}\right)^2\leq \frac{1}{2}Q^2.
          \end{aligned}
          \right.
      \end{equation*}
Eventually, we conclude that if $\Vert f \Vert_K \leq Q$, then $P_\WW(f) = f$; otherwise, if $\Vert f \Vert_K > Q$,
 \begin{equation}\label{eq:A.7}
                \begin{aligned}
P_\WW(f)&=\sum_{k=0}^{L_n}\sum_{j=1}^{\dim\HH_k^d}\frac{f_{k,j}}{1+\lambda}Y_{k,j},\\
\frac{1}{1+\lambda}&=\frac{Q}{\left(\sum_{k=0}^{L_n}\sum_{j=1}^{\dim\HH_k^d}a_k^{-1}f_{k,j}^2\right)^{\frac{1}{2}}}.
                \end{aligned} 
            \end{equation}     
Since the function $\frac{1}{2}\Vert f-g\Vert_K^2$ is strongly convex, the KKT point in \eqref{eq:A.7} corresponds to the unique optimal solution. This completes the proof.
\end{proof}

\begin{lemma}\label{lemma:A.7}
If \autoref{assum:Radon-Nikodym derivative} holds, then
$$b_\rho\Omega_{d-1}\Vert f\Vert_\omega^2\leq \Vert f\circ F\Vert_{\rho_X}^2\leq B_\rho\Omega_{d-1}\Vert f\Vert_\omega^2,\ \ \forall f\in\LLd.$$
\end{lemma}
\begin{proof}
By \autoref{assum:Radon-Nikodym derivative} and the definition of the pushforward distribution $\rho_F=F_\#\rho_X$, for every $f\in\LLd$,
\begin{align*}
\int_\Omega |f\circ F|^2d\rho_X=\int_{\SS^{d-1}} |f|^2d\rho_F=\int_{\SS^{d-1}} |f|^2\frac{d\rho_F}{d\omega}d\omega\leq B_\rho\Omega_{d-1}\frac{1}{\Omega_{d-1}}\int_{\SS^{d-1}} |f|^2d\omega.
\end{align*}
The lower bound follows analogously.
\end{proof}

\begin{lemma}\label{lemma:A.8}
Consider the regression model $Y = f^*(X) - \epsilon$, where $\epsilon$ denotes noise and $f^*\in\LL_{\rho_X}^2(\Omega)$. Assume that $\epsilon$ is independent of $X$ with density function $p(u)$ satisfying $p(-u)=p(u)$ and $p'(u)\leq 0$ for $u\ge 0$. Suppose further that the loss function $\ell(u)\ge 0$ is continuously differentiable, symmetric in the sense that $\ell(-u)=\ell(u)$, and satisfies $\ell'(u)\ge 0$ for $u\ge 0$. If, for every $\delta\in\RR$, both $\int\ell'(\delta+u)p(u)du$ and $\int\ell(\delta+u)p(u)du$ exist and are finite, then
$$f^*\in\arg\min_{f\in\LL_{\rho_X}^2(\Omega)}\ee(f) := \arg\min_{f\in\LL_{\rho_X}^2(\Omega)}\EE\left[\ell(f(X)-Y)\right].$$
\end{lemma}
\begin{remark}
The assumption on the noise in the lemma covers several common noise distributions, including Gaussian noise. The conditions on the loss function essentially require it to be symmetric and nondecreasing on $[0,\infty)$, thereby covering a broad class of robust regression losses, including non-convex losses such as the Cauchy loss and the Welsch loss.
\end{remark}
\begin{proof}
We have
$$\min_{f\in\LL_{\rho_X}^2(\Omega)}\ee(f)=\min_{f\in\LL_{\rho_X}^2(\Omega)}\EE_X\left[\EE_\epsilon[\ell(f(X)-Y)|X]\right].$$
Moreover, for any $f\in\LL_{\rho_X}^2(\Omega)$, we have $\EE_\epsilon[\ell(f(X)-f^*(X)+\epsilon)|X]\geq \inf_{\delta\in\RR}\EE_\epsilon[\ell(\delta+\epsilon)]$; it therefore suffices to show that $\inf_{\delta\in\RR}\EE_\epsilon[\ell(\delta+\epsilon)]=\EE_\epsilon[\ell(0+\epsilon)]$. We define $$\phi(\delta):=\EE_\epsilon[\ell(\delta+\epsilon)]=\int_{-\infty}^\infty\ell(\delta+u)p(u)du.$$
 By the symmetry of $\ell$ and the noise distribution, $\phi(\delta)$ is also symmetric about $\delta =0$. For $\delta>0$, we have
 \begin{align*}
\phi'(\delta)&=\int_{-\infty}^\infty\ell'(\delta+u)p(u)du=\int_{-\infty}^{-\delta}\ell'(\delta+u)p(u)du+\int_{-\delta}^\infty\ell'(\delta+u)p(u)du\\
&=\int_{-\infty}^{0}\ell'(u)p(u-\delta)du+\int_{0}^\infty\ell'(u)p(u-\delta)du=\int_{0}^\infty\ell'(u)\left(p(u-\delta)-p(u+\delta)\right)du\\
&\overset{\text{(i)}}{\geq} \int_{0}^\infty\ell'(u)\cdot 0du\geq 0,
\end{align*}
where (i) follows from the identity $p(0-\delta)=p(0+\delta)$ and the monotonicity of $p$ on $[0,\infty)$. It follows from the monotonicity of $\phi(\delta)$ on $[0,\infty)$ and its symmetry that $0$ is a minimizer of $\phi(\delta)$. The proof is complete.
\end{proof}

\subsection{Proof of \autoref{theorem:strong convergence} (Strong Convergence)}\label{sec:strong convergence prove}

In contrast to the order of presentation in the main text, we begin by proving the strong convergence guarantee of the T-kernel SGD. We then present the proof of \autoref{theorem:strong convergence} directly. First, let $f_{L_n}$ be the projection of $f^*$ onto $(\HH_{L_n},\Vert\cdot\Vert_K)$.
\begin{align}
&\left\Vert \hat{f}_n-f_{L_n}\right\Vert_K^2\notag\\
=&\left\Vert P_\WW\left(\hat{f}_{n-1}-\gamma_n\partial_u\ell\left(\hat{f}_{n-1}\circ F(X_n),Y_n\right)K_{L_n}^T(F(X_n),\cdot)\right)-f_{L_n}\right\Vert_K^2\notag\\
\overset{\text{(i)}}{\leq}&\left\Vert \hat{f}_{n-1}-\gamma_n\partial_u\ell\left(\hat{f}_{n-1}\circ F(X_n),Y_n\right)K_{L_n}^T(F(X_n),\cdot)-f_{L_n}\right\Vert_K^2\notag\\
=&\left\Vert \hat{f}_{n-1}-f_{L_n}\right\Vert_K^2-2\gamma_n\lan\partial_u\ell\left(\hat{f}_{n-1}\circ F(X_n),Y_n\right)K_{L_n}^T(F(X_n),\cdot), \hat{f}_{n-1}-f_{L_n}\ran_K\notag\\
&+\gamma_n^2\left|\partial_u\ell\left(\hat{f}_{n-1}\circ F(X_n),Y_n\right)\right|^2\Vert K_{L_n}^T(F(X_n),\cdot)\Vert_K^2.\notag
\end{align}
In (i), we use the result $\Vert P_\WW(f) - P_\WW(g) \Vert_K \leq \Vert f - g \Vert_K$ for $f, g \in \HH_K$, as stated in Section 4.3-1 of \cite{ciarlet2013linear}, where $\WW$ is a closed convex subset of $\HH_K$, and $f_{L_n} \in \WW$. Using  $\Vert \hat{f}_{n-1}\Vert_\infty\leq\Vert \hat{f}_{n-1}\Vert_K\kappa<B$, and the bound 
$$\sup_{x\in\SS^{d-1}}\Vert K_{L_n}^T(x,\cdot)\Vert_{K}=\sup_{x\in\SS^{d-1}}\sqrt{ K_{L_n}^T(x,x)}\leq\sup_{x\in\SS^{d-1}}\sqrt{ K(x,x)}=\kappa,$$ 
we obtain
$$\left|\partial_u\ell\left(\hat{f}_{n-1}\circ F(X_n),Y_n\right)\right|^2\Vert K_{L_n}^T(F(X_n),\cdot)\Vert_K^2\leq M^2\kappa^2:=M_1^2,$$
where we define $M_1^2 := M^2 \kappa^2$. Therefore, one has
\begin{equation}\label{eq:B.1}
                \begin{aligned}
&\left\Vert \hat{f}_n-f_{L_n}\right\Vert_K^2-\left\Vert \hat{f}_{n-1}-f_{L_n}\right\Vert_K^2\\
\leq&-2\gamma_n\lan\partial_u\ell\left(\hat{f}_{n-1}\circ F(X_n),Y_n\right)K_{L_n}^T(F(X_n),\cdot), \hat{f}_{n-1}-f_{L_n}\ran_K+\gamma_n^2M_1^2.
                \end{aligned} 
            \end{equation}
Since $\WW \cap \HH_{L_n}$ is a bounded and closed subset of the finite-dimensional space $\HH_{L_n}$, it is compact. Moreover, since $\ee(f)$ is continuous on $\WW$, it attains its minimum on the compact set $\WW \cap \HH_{L_n}$. That is, there exists $f_{L_n}^* = \arg\min_{f \in \WW \cap \HH_{L_n}} \ee(f)$.
Taking expectations on both sides of \eqref{eq:B.1}, we obtain
\begin{equation}\label{eq:B.2}
                \begin{aligned}
&\EE\left[\left\Vert \hat{f}_n-f_{L_n}\right\Vert_K^2\right]\\
\leq&\EE\left[\left\Vert \hat{f}_{n-1}-f_{L_n}\right\Vert_K^2\right]\\
&-2\gamma_n\EE\left[\lan\partial_u\ell\left(\hat{f}_{n-1}\circ F(X_n),Y_n\right)K_{L_n}^T(F(X_n),\cdot), \hat{f}_{n-1}-f_{L_n}\ran_K\right]+\gamma_n^2M_1^2\\
=&\EE\left[\left\Vert \hat{f}_{n-1}-f_{L_n}\right\Vert_K^2\right]\\
&-2\gamma_n\EE\left[\lan\partial_u\ell\left(\hat{f}_{n-1}\circ F(X_n),Y_n\right)K_{L_n}^T(F(X_n),\cdot), \hat{f}_{n-1}-f_{L_n}^*\ran_K\right]\\
&-2\gamma_n\EE\left[\lan\partial_u\ell\left(\hat{f}_{n-1}\circ F(X_n),Y_n\right)K_{L_n}^T(F(X_n),\cdot), f_{L_n}^*-f_{L_n}\ran_K\right]+\gamma_n^2M_1^2.
                \end{aligned} 
            \end{equation}
Next, we apply \autoref{lemma:B.2} and \autoref{lemma:B.3} to derive
\begin{equation}\label{eq:B.3}
\begin{aligned}
&\EE\left[\lan\partial_u\ell\left(\hat{f}_{n-1}\circ F(X_n),Y_n\right)K_{L_n}^T(F(X_n),\cdot), \hat{f}_{n-1}-f_{L_n}^*\ran_K\right]\\
\geq& \frac{\mu}{2} \EE\left[\left\Vert \hat{f}_{n-1}\circ F-f_{L_n}^*\circ F\right\Vert_{\rho_X}^2\right]
\end{aligned}
\end{equation}
and 
\begin{equation}\label{eq:B.4}
                \begin{aligned}
&-\EE\left[\lan\partial_u\ell\left(\hat{f}_{n-1}\circ F(X_n),Y_n\right)K_{L_n}^T(F(X_n),\cdot), f_{L_n}^*-f_{L_n}\ran_K\right]\\
\leq& L\cdot\EE\left[\Vert \hat{f}_{n-1}\circ F-f_{L_n}^*\circ F\Vert_{\rho_X}\cdot\Vert f_{L_n}^*\circ F-f_{L_n}\circ F\Vert_{\rho_X}\right]+\frac{L}{2}\Vert f_{L_n}\circ F-f^*\circ F\Vert_{\rho_X}^2.
                \end{aligned} 
            \end{equation}
We combine \eqref{eq:B.3} and \eqref{eq:B.4} to continue \eqref{eq:B.2},
\begin{equation}\label{eq:B.5}
                \begin{aligned}
&\EE\left[\left\Vert \hat{f}_n-f_{L_n}\right\Vert_K^2\right]\\
\leq&\EE\left[\left\Vert \hat{f}_{n-1}-f_{L_n}\right\Vert_K^2\right]-\gamma_n\mu \EE\left[\left\Vert \hat{f}_{n-1}\circ F-f_{L_n}^*\circ F\right\Vert_{\rho_X}^2\right]+\gamma_nL\Vert f_{L_n}\circ F-f^*\circ F\Vert_{\rho_X}^2\\
&+2\gamma_nL\cdot\EE\left[\Vert \hat{f}_{n-1}\circ F-f_{L_n}^*\circ F\Vert_{\rho_X}\cdot\Vert f_{L_n}^*\circ F-f_{L_n}\circ F\Vert_{\rho_X}\right]+\gamma_n^2M_1^2\\
=&\EE\left[\left\Vert \hat{f}_{n-1}-f_{L_n}\right\Vert_K^2\right]-\frac{\gamma_n\mu}{2}\EE\left[\left\Vert \hat{f}_{n-1}\circ F-f_{L_n}^*\circ F\right\Vert_{\rho_X}^2\right]+\gamma_nL\Vert f_{L_n}\circ F-f^*\circ F\Vert_{\rho_X}^2\\
&+2\gamma_nL\cdot\EE\bigg[\left\Vert \hat{f}_{n-1}\circ F-f_{L_n}^*\circ F\right\Vert_{\rho_X}\cdot\bigg(\left\Vert f_{L_n}^*\circ F-f_{L_n}\circ F\right\Vert_{\rho_X}\\
&-\frac{\mu}{4L}\left\Vert \hat{f}_{n-1}\circ F-f_{L_n}^*\circ F\right\Vert_{\rho_X}\bigg)\bigg]+\gamma_n^2M_1^2.
                \end{aligned} 
            \end{equation}
By \autoref{lemma:B.4}, we have
\begin{equation}\label{eq:B.6}
\begin{aligned}
&\EE\left[\left\Vert \hat{f}_{n-1}\circ F-f_{L_n}^*\circ F\right\Vert_{\rho_X}\left(\left\Vert f_{L_n}^*\circ F-f_{L_n}\circ F\right\Vert_{\rho_X}-\frac{\mu}{4L}\left\Vert \hat{f}_{n-1}\circ F-f_{L_n}^*\circ F\right\Vert_{\rho_X}\right)\right]\\
&\leq\frac{4L}{\mu}\left\Vert f_{L_n}\circ F-f_{L_n}^*\circ F\right\Vert_{\rho_X}^2.
\end{aligned}
            \end{equation}
Combining \eqref{eq:B.6} with the preceding steps to continue from \eqref{eq:B.5} yields
\begin{equation}\label{eq:B.7}
                \begin{aligned}
&\EE\left[\left\Vert \hat{f}_n-f_{L_n}\right\Vert_K^2\right]\\
\leq&\EE\left[\left\Vert \hat{f}_{n-1}-f_{L_n}\right\Vert_K^2\right]-\frac{\gamma_n\mu}{2}\EE\left[\left\Vert \hat{f}_{n-1}\circ F-f_{L_n}^*\circ F\right\Vert_{\rho_X}^2\right]\\
&+\gamma_nL\Vert f_{L_n}\circ F-f^*\circ F\Vert_{\rho_X}^2+\frac{8L^2}{\mu}\gamma_n\left\Vert f_{L_n}\circ F-f_{L_n}^*\circ F\right\Vert_{\rho_X}^2+\gamma_n^2M_1^2.
                \end{aligned} 
            \end{equation}
We use the following inequality in conjunction with \eqref{eq:B.7}
\begin{equation*}
\begin{aligned}
\EE\left[\left\Vert \hat{f}_{n-1}\circ F-f_{L_n}\circ F\right\Vert_{\rho_X}^2\right]\leq2\EE\left[\left\Vert \hat{f}_{n-1}\circ F-f_{L_n}^*\circ F\right\Vert_{\rho_X}^2\right]+2\EE\left[\left\Vert f_{L_n}^*\circ F-f_{L_n}\circ F\right\Vert_{\rho_X}^2\right]\\
\Rightarrow -\EE\left[\left\Vert \hat{f}_{n-1}\circ F-f_{L_n}^*\circ F\right\Vert_{\rho_X}^2\right]\leq\left\Vert f_{L_n}^*\circ F-f_{L_n}\circ F\right\Vert_{\rho_X}^2-\frac{1}{2}\EE\left[\left\Vert \hat{f}_{n-1}\circ F-f_{L_n}\circ F\right\Vert_{\rho_X}^2\right]
,
\end{aligned}
\end{equation*}
to obtain
\begin{equation}\label{eq:B.8}
                \begin{aligned}
&\EE\left[\left\Vert \hat{f}_n-f_{L_n}\right\Vert_K^2\right]\\
\leq&\EE\left[\left\Vert \hat{f}_{n-1}-f_{L_n}\right\Vert_K^2\right]-\frac{\gamma_n\mu}{4}\EE\left[\left\Vert \hat{f}_{n-1}\circ F-f_{L_n}\circ F\right\Vert_{\rho_X}^2\right]\\
&+\frac{\gamma_n\mu}{2}\left\Vert f_{L_n}^*\circ F-f_{L_n}\circ F\right\Vert_{\rho_X}^2+\gamma_nL\Vert f_{L_n}\circ F-f^*\circ F\Vert_{\rho_X}^2\\
&+\frac{8L^2}{\mu}\gamma_n\left\Vert f_{L_n}\circ F-f_{L_n}^*\circ F\right\Vert_{\rho_X}^2+\gamma_n^2M_1^2.
                \end{aligned} 
            \end{equation}
We note that the orthogonal complement of $\HH_{L_n}$ in $\HH_K$ is $\HH_{L_n}^\perp$. Since $\hat{f}_n-f_{L_n}\in \HH_{L_n}$ and $f_{L_{n+1}}-f_{L_n}=(f_{L_{n+1}}-f^*)-(f_{L_n}-f^*)\in \HH_{L_n}^\perp$, it follows that $\hat{f}_n-f_{L_n}$ is orthogonal to $f_{L_{n+1}} - f_{L_n}$. Therefore, we obtain            
$$\EE\left[\left\Vert \hat{f}_n-f_{L_{n+1}}\right\Vert_K^2\right] = \EE\left[\left\Vert \hat{f}_n-f_{L_n}\right\Vert_K^2\right] + \left\Vert f_{L_{n+1}}-f_{L_n}\right\Vert_K^2.$$
Substituting the above equation back into \eqref{eq:B.8}, one can obtain
\begin{equation}\label{eq:B.9}
\begin{aligned}
&\EE\left[\left\Vert \hat{f}_n-f_{L_{n+1}}\right\Vert_K^2\right]\\
\leq&\EE\left[\left\Vert \hat{f}_{n-1}-f_{L_n}\right\Vert_K^2\right]-\frac{\gamma_n\mu}{4}\EE\left[\left\Vert \hat{f}_{n-1}\circ F-f_{L_n}\circ F\right\Vert_{\rho_X}^2\right]\\
&+\gamma_n\left(\frac{\mu}{2}+\frac{8L^2}{\mu}\right)\left\Vert f_{L_n}^*\circ F-f_{L_n}\circ F\right\Vert_{\rho_X}^2+\gamma_nL\Vert f_{L_n}\circ F-f^*\circ F\Vert_{\rho_X}^2\\
&+\gamma_n^2M_1^2+\left\Vert f_{L_{n+1}}-f_{L_n}\right\Vert_K^2.
\end{aligned} 
\end{equation}
In \autoref{lemma:B.5}, we show that if $f\in \HH_{L_n}$,
\begin{align}\label{eq:B.10}
\Vert f\circ F\Vert_{\rho_X}^2\geq\frac{A^2_2}{A_1}\frac{b_{\rho}\Omega_{d-1}}{(2d)^{2s}}\,n^{-2\theta s}\Vert f\Vert_K^2.
\end{align}
In \autoref{lemma:B.6}, we establish the following inequalities
\begin{equation}\label{eq:B.11}
                \begin{aligned}
\Vert f_{L_n}^*\circ F-f_{L_n}\circ F\Vert_{\rho_X}^2&\leq\frac{L}{\mu}B_\rho\Omega_{d-1} A_1^{2r}\Vert f^*\Vert_{\WW^r}^2\left(n+1\right)^{-4\theta sr},\\
\Vert f_{L_n}\circ F-f^*\circ F\Vert_{\rho_X}^2&\leq  B_\rho\Omega_{d-1} A_1^{2r}\Vert f^*\Vert_{\WW^r}^2\left(n+1\right)^{-4\theta sr}.
                \end{aligned} 
            \end{equation}
In combination with \eqref{eq:B.10} and \eqref{eq:B.11}, we continue \eqref{eq:B.9} to obtain
\begin{equation}\label{eq:B.12}
                \begin{aligned}
&\EE\left[\left\Vert \hat{f}_n-f_{L_{n+1}}\right\Vert_K^2\right]\\
\leq&\left(1-\frac{A^2_2}{A_1}\frac{b_{\rho}\mu\Omega_{d-1}}{4(2d)^{2s}}\gamma_n n^{-2\theta s}\right)\EE\left[\left\Vert \hat{f}_{n-1}-f_{L_n}\right\Vert_K^2\right]\\
&+\gamma_n\left(\frac{\mu}{2}+\frac{8L^2}{\mu}\right)\frac{L}{\mu}B_\rho\Omega_{d-1} A_1^{2r}\Vert f^*\Vert_{\WW^r}^2\left(n+1\right)^{-4\theta sr}\\
&+\gamma_nLB_\rho\Omega_{d-1} A_1^{2r}\Vert f^*\Vert_{\WW^r}^2\left(n+1\right)^{-4\theta sr}+\gamma_n^2M_1^2+\left\Vert f_{L_{n+1}}-f_{L_n}\right\Vert_K^2.
                \end{aligned} 
            \end{equation}
We choose $t = \frac{2r}{2r+1}$, set the step size as $\gamma_n = \gamma_0 n^{-t} \log(n+1)$, and also set $\theta = \frac{1}{2s(2r+1)}$. Under this hyperparameter setting, we obtain the following two identities: $t = 4\theta sr$ and $t + 2\theta s = 1$, as well as the inequality $(n+1)^{-4\theta sr} \leq \frac{\gamma_n}{\gamma_0 \log(2)}$. We set the initial step size as $\gamma_0 = c\frac{A_14(2d)^{2s}}{A^2_2b_{\rho}\mu\Omega_{d-1}}$, where the constant $c$ satisfies $\frac{1}{\log(2)} \leq c \leq \frac{2}{\log(3)}$. Substituting the above constants and inequalities into \eqref{eq:B.12}, we obtain
\begin{equation*}
\begin{aligned}
&\EE\left[\left\Vert \hat{f}_n-f_{L_{n+1}}\right\Vert_K^2\right]\notag\\
\leq&\left(1-c\frac{\log(n+1)}{n}\right)\EE\left[\left\Vert \hat{f}_{n-1}-f_{L_n}\right\Vert_K^2\right]+\left\Vert f_{L_{n+1}}-f_{L_n}\right\Vert_K^2\\
&+\gamma_n^2\left[\left(\left(\frac{\mu}{2}+\frac{8L^2}{\mu}\right)\frac{L}{\mu}+L\right)B_\rho\Omega_{d-1} A_1^{2r}\Vert f^*\Vert_{\WW^r}^2\frac{1}{\gamma_0\log(2)}+M_1^2\right]\\
\overset{\text{(i)}}{\leq}&\left(1-c\frac{\log(n+1)}{n}\right)\EE\left[\left\Vert \hat{f}_{n-1}-f_{L_n}\right\Vert_K^2\right]+\left\Vert f_{L_{n+1}}-f_{L_n}\right\Vert_K^2\\
&+n^{-2t}\left(\log(n+1)\right)^2P^2.
\end{aligned} 
\end{equation*}
In (i), we define the quantity $P^2=\gamma_0^2\left[\left(\left(\frac{\mu}{2}+\frac{8L^2}{\mu}\right)\frac{L}{\mu}+L\right)B_\rho\Omega_{d-1} A_1^{2r}\Vert f^*\Vert_{\WW^r}^2\frac{1}{\gamma_0\log(2)}+M_1^2\right]$. 

Consider the function $h(u)=\frac{\log(u+1)}{u}$, which is monotonically decreasing for $u\geq2$. In particular, we have $\left(1-c\frac{\log(n+1)}{n}\right)\geq0$ for $n\geq2$. Based on the recursive relation for $\hat{f}_n$, we have
\begin{equation}\label{eq:B.13}
                \begin{aligned}
&\EE\left[\left\Vert \hat{f}_n-f_{L_{n+1}}\right\Vert_K^2\right]\\
\leq&\left(c\log(2)-1\right)\prod_{l=2}^n\left(1-c\frac{\log(l+1)}{l}\right)\left\Vert \hat{f}_{0}-f_{L_1}\right\Vert_K^2\\
&+\sum_{k=1}^n\prod_{l=k+1}^n\left(1-c\frac{\log(l+1)}{l}\right)\left\Vert f_{L_k}-f_{L_{k+1}}\right\Vert_K^2\\
&+\sum_{k=1}^n\prod_{l=k+1}^n\left(1-c\frac{\log(l+1)}{l}\right)k^{-2t}\left(\log(k+1)\right)^2P^2.
                \end{aligned} 
            \end{equation}
Here, we apply \autoref{lemma:B.7} and \autoref{lemma:B.8} to \eqref{eq:B.13} and obtain
\begin{equation*}
\begin{aligned}
&\EE\left[\left\Vert \hat{f}_n-f_{L_{n+1}}\right\Vert_K^2\right]\\
\leq&\left(2Q^2+2A_1^{2r-1}\Vert f^*\Vert_{\WW^r}^2\right)(n+1)^{-\frac{2r-1}{2r+1}}+(4r+2)P^2(\log(n+1))^2(n+1)^{-\frac{2r-1}{2r+1}}.
\end{aligned}
\end{equation*}
Using the third inequality in \autoref{lemma:B.6}, we complete the proof of \autoref{theorem:strong convergence},
\begin{equation*}
\begin{aligned}
&\EE\left[\left\Vert \hat{f}_n-f^*\right\Vert_K^2\right]=\EE\left[\left\Vert \hat{f}_n-f_{L_{n+1}}\right\Vert_K^2\right]+\left\Vert f_{L_{n+1}}-f^*\right\Vert_{K}^2\\
\leq&\left(2Q^2+3A_1^{2r-1}\Vert f^*\Vert_{\WW^r}^2\right)(n+1)^{-\frac{2r-1}{2r+1}}+(4r+2)P^2(\log(n+1))^2(n+1)^{-\frac{2r-1}{2r+1}}.
\end{aligned}
\end{equation*}

\subsubsection{Technical Results}

\begin{lemma}\label{lemma:B.2}
If the assumptions in \autoref{theorem:strong convergence} hold and the quantity
$$\EE\left[\lan\partial_u\ell\left(\hat{f}_{n-1}\circ F(X_n),Y_n\right)K_{L_n}^T(F(X_n),\cdot), \hat{f}_{n-1}-f_{L_n}^*\ran_K\right]$$ 
is defined as in \eqref{eq:B.2}, then we have
\begin{align*}
\EE\left[\lan\partial_u\ell\left(\hat{f}_{n-1}\circ F(X_n),Y_n\right)K_{L_n}^T(F(X_n),\cdot), \hat{f}_{n-1}-f_{L_n}^*\ran_K\right]\geq \frac{\mu}{2} \EE\left[\left\Vert \hat{f}_{n-1}\circ F-f_{L_n}^*\circ F\right\Vert_{\rho_X}^2\right].
\end{align*}
\end{lemma}
\begin{proof}
By the local strong convexity of the loss function in \autoref{assum:mustorngconvex}, we have
\begin{equation}\label{eq:B.14}
                \begin{aligned}
\ell(f_{L_n}^*\circ F(X_n),Y_n)\geq&\ell(\hat{f}_{n-1}\circ F(X_n),Y_n)+\frac{\mu}{2}(f_{L_n}^*\circ F(X_n)-\hat{f}_{n-1}\circ F(X_n))^2\\
&+\partial_u\ell(\hat{f}_{n-1}\circ F(X_n),Y_n)(f_{L_n}^*\circ F(X_n)-\hat{f}_{n-1}\circ F(X_n)).
                \end{aligned} 
            \end{equation}
Taking expectation on both sides of \eqref{eq:B.14}, one can obtain
\begin{align*}
&\EE[\ee(f_{L_n}^*)]\overset{\text{(i)}}{\geq}\EE[\ee(\hat{f}_{n-1})]+\EE\left[\lan\partial_u\ell(\hat{f}_{n-1}\circ F(X_n),Y_n)K_{L_n}^T(F(X_n),\cdot),f_{L_n}^*-\hat{f}_{n-1}\ran_K\right]\\
&\quad\quad\quad\quad+\frac{\mu}{2}\EE\left[\left\Vert \hat{f}_{n-1}\circ F-f_{L_n}^*\circ F\right\Vert_{\rho_X}^2\right]\\
\Rightarrow\quad  &\EE\left[\lan\partial_u\ell(\hat{f}_{n-1}\circ F(X_n),Y_n)K_{L_n}^T(F(X_n),\cdot),\hat{f}_{n-1}-f_{L_n}^*\ran_K\right]\\
\geq&\EE\left[\ee(\hat{f}_{n-1})-\ee(f_{L_n}^*)\right]+
\frac{\mu}{2}\EE\left[\left\Vert \hat{f}_{n-1}\circ F-f_{L_n}^*\circ F\right\Vert_{\rho_X}^2\right]\\
\geq&\frac{\mu}{2}\EE\left[\left\Vert \hat{f}_{n-1}\circ F-f_{L_n}^*\circ F\right\Vert_{\rho_X}^2\right],
\end{align*}
where (i) follows from the fact that $\hat{f}_{n-1},f_{L_n}^*\in \HH_{L_n}$ and $\left(\HH_{L_n}, \langle \cdot, \cdot \rangle_K\right)$ is an RKHS associated with the kernel $K_{L_n}^T(x,x')$. This completes the proof.
\end{proof}

\begin{lemma}\label{lemma:B.3}
If the assumptions in \autoref{theorem:strong convergence} hold and 
$$\EE\left[\lan\partial_u\ell\left(\hat{f}_{n-1}\circ F(X_n),Y_n\right)K_{L_n}^T(F(X_n),\cdot), f_{L_n}^*-f_{L_n}\ran_K\right]$$
is defined as in \eqref{eq:B.2}, we obtain
\begin{align*}
&-\EE\left[\lan\partial_u\ell\left(\hat{f}_{n-1}\circ F(X_n),Y_n\right)K_{L_n}^T(F(X_n),\cdot), f_{L_n}^*-f_{L_n}\ran_K\right]\\
\leq& L\cdot\EE\left[\Vert \hat{f}_{n-1}\circ F-f_{L_n}^*\circ F\Vert_{\rho_X}\cdot\Vert f_{L_n}^*\circ F-f_{L_n}\circ F\Vert_{\rho_X}\right]+\frac{L}{2}\Vert f_{L_n}\circ F-f^*\circ F\Vert_{\rho_X}^2.
\end{align*}
\end{lemma}

\begin{proof}
	For notational convenience, let $\hat{h}_{n-1} = \hat{f}_{n-1}\circ F$, $h_{L_n}^* = f_{L_n}^*\circ F$ and $h_{L_n} = f_{L_n}\circ F$ throughout this lemma. We begin by decomposing the following expression.
\begin{equation}\label{eq:B.15}
                \begin{aligned}
&-\EE\left[\lan\partial_u\ell\left(\hat{h}_{n-1}(X_n),Y_n\right)K_{L_n}^T(F(X_n),\cdot), f_{L_n}^*-f_{L_n}\ran_K\right]\\
=& -\EE\left[\lan\left(\partial_u\ell\left(\hat{h}_{n-1}(X_n),Y_n\right)-\partial_u\ell\left(h_{L_n}^*(X_n),Y_n\right)\right)K_{L_n}^T(F(X_n),\cdot), f_{L_n}^*-f_{L_n}\ran_K\right]\\
&-\EE\left[\lan\partial_u\ell\left(h_{L_n}^*(X_n),Y_n\right)K_{L_n}^T(F(X_n),\cdot), f_{L_n}^*-f_{L_n}\ran_K\right].
                \end{aligned} 
            \end{equation}
Let $\DD_{n-1}$ be the $\sigma$-field defined by $\DD_{n-1} = \sigma\left((X_1, Y_1), \dots, (X_{n-1}, Y_{n-1})\right)$. Considering the first term in \eqref{eq:B.15}, one has
\begin{equation}\label{eq:B.16}
                \begin{aligned}
&-\EE\left[\lan\left(\partial_u\ell\left(\hat{h}_{n-1}(X_n),Y_n\right)-\partial_u\ell\left(h_{L_n}^*(X_n),Y_n\right)\right)K_{L_n}^T(F(X_n),\cdot), f_{L_n}^*-f_{L_n}\ran_K\right]\\
\overset{\text{(i)}}{\leq}&\EE\left[\left|\partial_u\ell\left(\hat{h}_{n-1}(X_n),Y_n\right)-\partial_u\ell\left(h_{L_n}^*(X_n),Y_n\right)\right|\cdot\left| h_{L_n}^*(X_n)-h_{L_n}(X_n)\right|\right]\\
\overset{\text{(ii)}}{\leq}&L\cdot\EE\left[\left|\hat{h}_{n-1}(X_n)-h_{L_n}^*(X_n)\right|\cdot\left| h_{L_n}^*(X_n)-h_{L_n}(X_n)\right|\right]\\
=&L\cdot\EE\left[\EE\left[\left|\hat{h}_{n-1}(X_n)-h_{L_n}^*(X_n)\right|\cdot\left| h_{L_n}^*(X_n)-h_{L_n}(X_n)\right|\ \big|\DD_{n-1}\right]\right]\\
\overset{\text{(iii)}}{\leq}&L\cdot\EE\left[\left\Vert\hat{h}_{n-1}-h_{L_n}^*\right\Vert_{\rho_X}\cdot\left\Vert h_{L_n}^*-h_{L_n}\right\Vert_{\rho_X}\right].
                \end{aligned} 
            \end{equation}
Here, (i) follows from the fact that $f_{L_n}, f_{L_n}^* \in \HH_{L_n}$ and that $\left(\HH_{L_n}, \langle\cdot,\cdot\rangle_K\right)$ is an RKHS associated with the kernel $K_{L_n}^T(x,x')$. In (ii), we apply the local $L$-smoothness assumption stated in \autoref{assum:Lsmooth}. In (iii), we use the Cauchy–Schwarz inequality.

Since $\ee(f)$ is convex on $\WW$ by \autoref{lemma:A.5}, and following Section 7.12-1 in \cite{ciarlet2013linear}, we analyze the second term in \eqref{eq:B.15}.
\begin{equation}\label{eq:B.17}
                \begin{aligned}
&-\EE\left[\lan\partial_u\ell\left(f_{L_n}^*\circ F(X_n),Y_n\right)K_{L_n}^T(F(X_n),\cdot), f_{L_n}^*-f_{L_n}\ran_K\right]\\
=&\lan\nabla\ee(f_{L_n}^*)\big|_{\HH_{L_n}}, f_{L_n}-f_{L_n}^*\ran_K\\
\leq&\ee(f_{L_n})-\ee(f_{L_n}^*)\leq\ee(f_{L_n})-\ee(f^*)\overset{\text{(i)}}{\leq}\frac{L}{2}\Vert f_{L_n}\circ F-f^*\circ F\Vert_{\rho_X}^2,
                \end{aligned} 
            \end{equation}
where (i) is due to \autoref{lemma:A.3}. Finally, combining \eqref{eq:B.16} and \eqref{eq:B.17}, we obtain the conclusion of the lemma
\begin{equation*}
\begin{aligned}
&-\EE\left[\lan\partial_u\ell\left(\hat{f}_{n-1}\circ F(X_n),Y_n\right)K_{L_n}^T(F(X_n),\cdot), f_{L_n}^*-f_{L_n}\ran_K\right]\\
\leq& L\cdot\EE\left[\Vert \hat{f}_{n-1}\circ F-f_{L_n}^*\circ F\Vert_{\rho_X}\cdot\Vert f_{L_n}^*\circ F-f_{L_n}\circ F\Vert_{\rho_X}\right]+\frac{L}{2}\Vert f_{L_n}\circ F-f^*\circ F\Vert_{\rho_X}^2.
\end{aligned}
\end{equation*}
\end{proof}

\begin{lemma}\label{lemma:B.4}
If the quantity in the first line of the following expression is defined as in equation \eqref{eq:B.5}, then we obtain
\begin{equation*}
\begin{aligned}
&\EE\left[\left\Vert \hat{f}_{n-1}\circ F-f_{L_n}^*\circ F\right\Vert_{\rho_X}\left(\left\Vert f_{L_n}^*\circ F-f_{L_n}\circ F\right\Vert_{\rho_X}-\frac{\mu}{4L}\left\Vert \hat{f}_{n-1}\circ F-f_{L_n}^*\circ F\right\Vert_{\rho_X}\right)\right]\\
&\leq\frac{4L}{\mu}\left\Vert f_{L_n}\circ F-f_{L_n}^*\circ F\right\Vert_{\rho_X}^2.
\end{aligned}
\end{equation*}
\end{lemma}
\begin{proof}
We define the following measurable set 
$$G=\left\{\left\Vert f_{L_n}^*\circ F-f_{L_n}\circ F\right\Vert_{\rho_X}-\frac{\mu}{4L}\left\Vert \hat{f}_{n-1}\circ F-f_{L_n}^*\circ F\right\Vert_{\rho_X}\geq0\right\}.$$ Meanwhile, the complement of $G$ is
$$G^c = \left\{\left\Vert f_{L_n}^*\circ F-f_{L_n}\circ F\right\Vert_{\rho_X}-\frac{\mu}{4L}\left\Vert \hat{f}_{n-1}\circ F-f_{L_n}^*\circ F\right\Vert_{\rho_X}<0\right\}.$$ 
For notational convenience, let $\hat{h}_{n-1} = \hat{f}_{n-1}\circ F$, $h_{L_n}^* = f_{L_n}^*\circ F$ and $h_{L_n} = f_{L_n}\circ F$ throughout this lemma. We then define the corresponding indicator functions $\X_G$ and $\X_{G^c}$, and decompose the original expression accordingly using these indicators, which yields
\begin{align*}
&\EE\left[\left\Vert \hat{h}_{n-1}-h_{L_n}^*\right\Vert_{\rho_X}\left(\left\Vert h_{L_n}^*-h_{L_n}\right\Vert_{\rho_X}-\frac{\mu}{4L}\left\Vert \hat{h}_{n-1}-h_{L_n}^*\right\Vert_{\rho_X}\right)\right]\\
=&\EE\left[\left\Vert \hat{h}_{n-1}-h_{L_n}^*\right\Vert_{\rho_X}\left(\left\Vert h_{L_n}^*-h_{L_n}\right\Vert_{\rho_X}-\frac{\mu}{4L}\left\Vert \hat{h}_{n-1}-h_{L_n}^*\right\Vert_{\rho_X}\right) \X_G\right]\\
+&\EE\left[\left\Vert \hat{h}_{n-1}-h_{L_n}^*\right\Vert_{\rho_X}\left(\left\Vert h_{L_n}^*-h_{L_n}\right\Vert_{\rho_X}-\frac{\mu}{4L}\left\Vert \hat{h}_{n-1}-h_{L_n}^*\right\Vert_{\rho_X}\right)\X_{G^c}\right]\\
\leq&\EE\left[\left\Vert \hat{h}_{n-1}-h_{L_n}^*\right\Vert_{\rho_X}\left(\left\Vert h_{L_n}^*-h_{L_n}\right\Vert_{\rho_X}-\frac{\mu}{4L}\left\Vert \hat{h}_{n-1}-h_{L_n}^*\right\Vert_{\rho_X}\right) \X_G\right]\\
\leq&\EE\left[\Vert \hat{h}_{n-1}-h_{L_n}^*\Vert_{\rho_X}\Vert h_{L_n}^*-h_{L_n}\Vert_{\rho_X}\X_G\right]\notag\\
\overset{\text{(i)}}{\leq}&\frac{4L}{\mu}\left\Vert h_{L_n}-h_{L_n}^*\right\Vert_{\rho_X}^2.
\end{align*}
Here, (i) follows from the definition of the set $G$. This completes the proof.
\end{proof}
\begin{lemma}\label{lemma:B.5}
Suppose that \autoref{assum:Radon-Nikodym derivative} holds. For any $f \in \HH_{L_n}$ with 
$$L_n = \min\left\{k \,\vert\, \dim{\Pi_k^d} \geq n^\theta \right\},$$
 we have
\begin{align*}
\Vert f\circ F\Vert_{\rho_X}^2\geq\frac{A^2_2}{A_1}\frac{b_{\rho}\Omega_{d-1}}{(2d)^{2s}}\,n^{-2\theta s}\Vert f\Vert_K^2.
\end{align*}
Here, $A_1 \geq A_2 > 0$ denote the upper and lower bounds of $a_k \cdot \left(\dim \Pi_k^d\right)^{2s}$ for all $k$, respectively, i.e.,
$$A_2\left(\dim\Pi_k^d\right)^{-2s}\leq a_k\leq A_1 \left(\dim\Pi_k^d\right)^{-2s}.$$
\end{lemma}
\begin{proof}
We choose $f \in \HH_{L_n}$ and set $f=\sum_{k=0}^{L_n}\sum_{j=1}^{\dim\HH_k^d}f_{k,j}Y_{k,j}$. Since $a_k > 0$ and $\lim_{k\to\infty} a_k \cdot \left(\dim \Pi_k^d\right)^{2s} = l$, it follows that there exist constants $A_1 \geq A_2 > 0$ such that $A_2\left(\dim\Pi_k^d\right)^{-2s}\leq a_k\leq A_1 \left(\dim\Pi_k^d\right)^{-2s}$ and for any $p\geq k$, we have 

$$\frac{A^2_2}{A_1}\frac{\left(\dim\Pi_p^d\right)^{-2s}}{a_k}\leq \frac{A_2}{A_1}\frac{a_{p}}{a_k}\leq \frac{A_2\left(\dim\Pi_p^d\right)^{-2s}}{a_k}\leq\frac{A_2\left(\dim\Pi_k^d\right)^{-2s}}{a_k}\leq1.$$
Combining the above two inequalities and \autoref{lemma:A.7}, we have
\begin{equation*}
\begin{aligned}
\Vert f\circ F\Vert_{\rho_X}^2\geq&b_\rho\Omega_{d-1}\Vert f\Vert_{\omega}^2= \frac{b_{\rho}\Omega_{d-1}}{\Omega_{d-1}}\int_{\SS^{d-1}}\left(\sum_{k=0}^{L_n}\sum_{j=1}^{\dim\HH_k^d}f_{k,j}Y_{k,j}\right)^2d\omega\\
=&b_{\rho}\Omega_{d-1}\sum_{k=0}^{L_n}\sum_{j=1}^{\dim\HH_k^d}f_{k,j}^2\geq b_{\rho}\Omega_{d-1}\frac{A^2_2}{A_1}\left(\dim\Pi_{L_n}^d\right)^{-2s}\sum_{k=0}^{L_n}\sum_{j=1}^{\dim\HH_k^d}\frac{\left(f_{k,j}\right)^2}{a_k}\\
\overset{\text{(i)}}{\geq}&\frac{A^2_2}{A_1}\frac{b_{\rho}\Omega_{d-1}}{(2d)^{2s}}\,n^{-2\theta s}\Vert f\Vert_K^2.
\end{aligned}
\end{equation*}
In (i), we use $\dim\Pi_{L_n-1}^d\leq n^\theta\leq\dim\Pi_{L_n}^d$ and $\dim\Pi_{L_n}^d\leq 2d\cdot\dim\Pi_{L_n-1}^d$ in Lemma 12 in \cite{bai2025truncated}, where we defined $\dim\Pi_{-1}^d=1$.

\end{proof}
\begin{lemma}\label{lemma:B.6}
If conditions in \autoref{theorem:strong convergence} hold, for $L_m\geq L_n\in\NN$, we have
\begin{align*}
\Vert f_{L_n}^*\circ F-f_{L_n}\circ F\Vert_{\rho_X}^2&\leq\frac{L}{\mu}B_\rho\Omega_{d-1} A_1^{2r}\Vert f^*\Vert_{\WW^r}^2\left(n+1\right)^{-4\theta sr},\\
\Vert f_{L_n}\circ F-f^*\circ F\Vert_{\rho_X}^2&\leq  B_\rho\Omega_{d-1} A_1^{2r}\Vert f^*\Vert_{\WW^r}^2\left(n+1\right)^{-4\theta sr}
\end{align*}
and we also have
\begin{align*}
\left\Vert f_{L_n}-f^*\right\Vert_{K}^2&\leq A_1^{2r-1}\left(n+1\right)^{-2\theta s(2r-1)}\Vert f^*\Vert_{\WW^r}^2,\\
\left\Vert f_{L_n}-f_{L_m}\right\Vert_{K}^2&\leq A_1^{2r-1}\left(n+1\right)^{-2\theta s(2r-1)}\Vert f^*\Vert_{\WW^r}^2.
\end{align*}
\end{lemma}
\begin{proof}
First, we use local $\mu-$strong convexity to obtain
\begin{equation}\label{eq:B.18}
                \begin{aligned}
\ell(f_{L_n}\circ F(X_n),Y_n)\geq&\ell(f_{L_n}^*\circ F(X_n),Y_n)\\
&+\partial_u\ell(f_{L_n}^*\circ F(X_n),Y_n)(f_{L_n}\circ F(X_n)-f_{L_n}^*\circ F(X_n))\\
&+\frac{\mu}{2}(f_{L_n}^*\circ F(X_n)-f_{L_n}\circ F(X_n))^2.
                \end{aligned} 
            \end{equation}
Taking expectations on both sides of \eqref{eq:B.18} yields
\begin{equation}\label{eq:B.19}
                \begin{aligned}
&\ee(f_{L_n})-\ee(f_{L_n}^*)\\
\geq&\EE\left[\lan\partial_u\ell(f_{L_n}^*\circ F(X_n),Y_n)K_{L_n}^T(F(X_n),\cdot),f_{L_n}-f_{L_n}^*\ran_K\right]+\frac{\mu}{2}\left\Vert f_{L_n}\circ F-f_{L_n}^*\circ F\right\Vert_{\rho_X}^2\\
\overset{\text{(i)}}{\geq}&\frac{\mu}{2}\left\Vert f_{L_n}\circ F-f_{L_n}^*\circ F\right\Vert_{\rho_X}^2.
                \end{aligned} 
            \end{equation}
Here, (i) follows from the Euler inequality of the convex function $\ee(f)$ at its minimizer $f_{L_n}^*$ over the convex set $\WW \cap \HH_{L_n}$ (see \autoref{lemma:A.5} and Theorem 7.12-3 in \cite{ciarlet2013linear}). Then by \autoref{lemma:A.3}, we use \eqref{eq:B.19} to obtain
\begin{align}\label{eq:B.20}
\frac{\mu}{2}\left\Vert f_{L_n}\circ F-f_{L_n}^*\circ F\right\Vert_{\rho_X}^2\leq\ee(f_{L_n})-\ee(f_{L_n}^*)\leq\ee(f_{L_n})-\ee(f^*)\leq \frac{L}{2}\left\Vert f_{L_n}\circ F-f^*\circ F\right\Vert_{\rho_X}^2.
\end{align}
Following a similar argument as in the proof of \autoref{lemma:B.5}, for $k \geq l$, we have 
$$1\leq A_1\frac{\left(\dim\Pi_k^d\right)^{-2s}}{a_k} \leq A_1\frac{\left(\dim\Pi_l^d\right)^{-2s}}{a_k}.$$
Let us denote $f^*=\sum_{k=0}^\infty\sum_{j=1}^{\dim\HH_k^d}f_{k,j}^*Y_{k,j}$. By applying \autoref{lemma:A.7}, we obtain
\begin{equation}\label{eq:B.21}
                \begin{aligned}
\left\Vert f_{L_n}\circ F-f^*\circ F\right\Vert_{\rho_X}^2\leq&B_\rho\Omega_{d-1}\Vert f_{L_n}-f^*\Vert_\omega^2\\
=&
\frac{B_{\rho}\Omega_{d-1}}{\Omega_{d-1}}\int_{\SS^{d-1}}\left(\sum_{k=L_n+1}^\infty\sum_{j=1}^{\dim\HH_k^d}f_{k,j}^*Y_{k,j}\right)^2d\omega\\
= &B_{\rho}\Omega_{d-1}\sum_{k=L_n+1}^\infty\sum_{j=1}^{\dim\HH_k^d}\left(f_{k,j}^*\right)^2\\
\leq& B_\rho\Omega_{d-1} A_1^{2r}\left(\dim\Pi_{L_n+1}^d\right)^{-4sr}\sum_{k=L_n+1}^\infty\sum_{j=1}^{\dim\HH_k^d}\frac{\left(f_{k,j}^*\right)^2}{a_k^{2r}}\\
\overset{\text{(i)}}{\leq}& B_\rho\Omega_{d-1} A_1^{2r}\Vert f^*\Vert_{\WW^r}^2\left(n+1\right)^{-4\theta sr}. 
                \end{aligned} 
            \end{equation}
In (i), we use $(n+1)^\theta\leq\dim\Pi_{L_{n+1}}^d\leq\dim\Pi_{L_{n}+1}^d$. Combining \eqref{eq:B.20} and \eqref{eq:B.21}, one has
\begin{align*}
\left\Vert f_{L_n}\circ F-f_{L_n}^*\circ F\right\Vert_{\rho_X}^2\leq \frac{L}{\mu}\left\Vert f_{L_n}\circ F-f^*\circ F\right\Vert_{\rho_X}^2\leq\frac{L}{\mu}B_\rho\Omega_{d-1} A_1^{2r}\Vert f^*\Vert_{\WW^r}^2\left(n+1\right)^{-4\theta sr}.
\end{align*}
Next we prove the last two inequalities, 
\begin{equation*}
\begin{aligned}
\left\Vert f_{L_n}-f^*\right\Vert_{K}^2&=\sum_{k=L_n+1}^\infty\sum_{j=1}^{\dim\HH_k^d}\frac{(f_{k,j}^*)^2}{a_k}
\leq\sum_{k=L_n+1}^\infty\sum_{j=1}^{\dim\HH_k^d}\frac{(f_{k,j}^*)^2}{a_k}A_1^{2r-1}\frac{\left(\dim\Pi_{L_n+1}^d\right)^{-2s(2r-1)}}{a_k^{2r-1}}\\
&\leq A_1^{2r-1}\left(\dim\Pi_{L_n+1}^d\right)^{-2s(2r-1)}\sum_{k=L_n+1}^\infty\sum_{j=1}^{\dim\HH_k^d}\frac{(f_{k,j}^*)^2}{a_k^{2r}}\\
&\leq A_1^{2r-1}\left(n+1\right)^{-2\theta s(2r-1)}\Vert f^*\Vert_{\WW^r}^2,
\end{aligned}
            \end{equation*}
and
\begin{align*}
\left\Vert f_{L_n}-f_{L_m}\right\Vert_{K}^2&\leq\left\Vert f_{L_n}-f_{L_m}\right\Vert_{K}^2+\left\Vert f_{L_m}-f^*\right\Vert_{K}^2 \notag\\
&=\left\Vert f_{L_n}-f^*\right\Vert_{K}^2\leq A_1^{2r-1}\left(n+1\right)^{-2\theta s(2r-1)}\Vert f^*\Vert_{\WW^r}^2.
\end{align*}
The proof is now complete.
\end{proof}

\begin{lemma}\label{lemma:B.7}
If  $\frac{1}{\log(2)}\leq c\leq\frac{2}{\log(3)}$ and $t=\frac{2r}{2r+1}$, then we have
\begin{align*}
\sum_{k=1}^n\prod_{l=k+1}^n\left(1-c\frac{\log(l+1)}{l}\right)k^{-2t}\left(\log(k+1)\right)^2\leq(4r+2)(\log(n+1))^2(n+1)^{-\frac{2r-1}{2r+1}}.
\end{align*}
\end{lemma}
\begin{proof}
Since $\frac{1}{\log(2)}\leq c\leq\frac{2}{\log(3)}$, it follows that $0\leq\left(1-c\frac{\log(l+1)}{l}\right)\leq\left(1-\frac{1}{l}\right)=\frac{l-1}{l}$ for all $l\geq2$. We can then obtain
\begin{equation*}
\begin{aligned}
&\sum_{k=1}^n\prod_{l=k+1}^n\left(1-c\frac{\log(l+1)}{l}\right)k^{-2t}\left(\log(k+1)\right)^2\\
\leq&\sum_{k=1}^n\prod_{l=k+1}^n\left(1-\frac{1}{l}\right)k^{-2t}\left(\log(k+1)\right)^2\\
\leq&\left(\log(n+1)\right)^2\sum_{k=1}^n\left(\prod_{l=k+1}^n\frac{l-1}{l}\right)k^{-2t}\\
=&\left(\log(n+1)\right)^2\frac{1}{n}\sum_{k=1}^nk^{-2t+1}\leq 4\left(\log(n+1)\right)^2\frac{1}{n+1}\sum_{k=1}^n(k+1)^{-2t+1}\\
\leq&4\left(\log(n+1)\right)^2\frac{1}{n+1}\int_{1}^{n+1}x^{1-2t}dx\leq\frac{2\left(\log(n+1)\right)^2}{(n+1)(1-t)}(n+1)^{2-2t}\\
=&(4r+2)\left(\log(n+1)\right)^2(n+1)^{-\frac{2r-1}{2r+1}}.
\end{aligned}
\end{equation*}
This completes the proof.
\end{proof}

\begin{lemma}\label{lemma:B.8}
If the assumptions in \autoref{theorem:strong convergence} hold, we have
\begin{equation}\label{eq:B.22}
                \begin{aligned}
&\left(c\log(2)-1\right)\prod_{l=2}^n\left(1-c\frac{\log(l+1)}{l}\right)\left\Vert \hat{f}_{0}-f_{L_1}\right\Vert_K^2\\ &+\sum_{k=1}^n\prod_{l=k+1}^n\left(1-c\frac{\log(l+1)}{l}\right)\left\Vert f_{L_k}-f_{L_{k+1}}\right\Vert_K^2\\
\leq&\left(2Q^2+2A_1^{2r-1}\Vert f^*\Vert_{\WW^r}^2\right)(n+1)^{-\frac{2r-1}{2r+1}}.
                \end{aligned} 
            \end{equation}
\end{lemma}

\begin{proof}
First, we consider the second term in \eqref{eq:B.22}
\begin{align}\label{eq:B.23}
&\sum_{k=1}^n\prod_{l=k+1}^n\left(1-c\frac{\log(l+1)}{l}\right)\left\Vert f_{L_k}-f_{L_{k+1}}\right\Vert_K^2\notag\\
\leq&\sum_{k=1}^{\frac{n}{2}-\frac{1}{2}}\prod_{l=k+1}^{n}\left(1-c\frac{\log(l+1)}{l}\right)\left\Vert f_{L_k}-f_{L_{k+1}}\right\Vert_K^2\notag\\
&+\sum_{k=\frac{n}{2}-\frac{1}{2}}^n\prod_{l=k+1}^{n}\left(1-c\frac{\log(l+1)}{l}\right)\left\Vert f_{L_k}-f_{L_{k+1}}\right\Vert_K^2\notag\\
\overset{\text{(i)}}{\leq}&\sum_{k=1}^{\frac{n}{2}-\frac{1}{2}}\prod_{l=k+1}^{n}\left(1-c\frac{\log(l+1)}{l}\right)\left\Vert f_{L_k}-f_{L_{k+1}}\right\Vert_K^2+\sum_{k=\frac{n}{2}-\frac{1}{2}}^n\left\Vert f_{L_k}-f_{L_{k+1}}\right\Vert_K^2\notag\\
\overset{\text{(ii)}}{=}&\sum_{k=1}^{\frac{n}{2}-\frac{1}{2}}\prod_{l=k+1}^{n}\left(1-c\frac{\log(l+1)}{l}\right)\left\Vert f_{L_k}-f_{L_{k+1}}\right\Vert_K^2+\left\Vert f_{L_{\frac{n}{2}-\frac{1}{2}}}-f_{L_{n+1}}\right\Vert_K^2,\\
=&\prod_{l=2}^{n}\left(1-c\frac{\log(l+1)}{l}\right)\left\Vert f_{L_1}-f_{L_{2}}\right\Vert_K^2 \notag\\
&+\sum_{k=2}^{\frac{n}{2}-\frac{1}{2}}\prod_{l=k+1}^{n}\left(1-c\frac{\log(l+1)}{l}\right)\left\Vert f_{L_k}-f_{L_{k+1}}\right\Vert_K^2+\left\Vert f_{L_{\frac{n}{2}-\frac{1}{2}}}-f_{L_{n+1}}\right\Vert_K^2.\notag
            \end{align}
Here, (i) follows from $\frac{1}{\log(2)} \leq c \leq \frac{2}{\log(3)}$, which implies that $0 \leq \left(1 - c\frac{\log(l+1)}{l} \right) \leq \left(1 - \frac{1}{l} \right) \leq 1$ for all $l \geq 2$. Consider the two terms $f_{L_{m+1}} - f_{L_m}$ and $f_{L_{k+1}} - f_{L_k}$ for indices $m > k$. The difference $f_{L_{k+1}} - f_{L_k}$ belongs to $\HH_{L_{k+1}}$, while the difference $f_{L_{m+1}} - f_{L_m} = (f_{L_{m+1}} - f^*) - (f_{L_m} - f^*)$ lies in the orthogonal complement $\HH_{L_{k+1}}^\perp$. Therefore, $f_{L_{m+1}} - f_{L_m}$ and $f_{L_{k+1}} - f_{L_k}$ are orthogonal, and condition (ii) is satisfied.

Since $\hat{f}_0 = 0$, we now bound the first terms in both \eqref{eq:B.22} and \eqref{eq:B.23},
\begin{equation}\label{eq:B.24}
                \begin{aligned}
&\left(c\log(2)-1\right)\prod_{l=2}^n\left(1-c\frac{\log(l+1)}{l}\right)\left\Vert \hat{f}_{0}-f_{L_1}\right\Vert_K^2+\prod_{l=2}^{n}\left(1-c\frac{\log(l+1)}{l}\right)\left\Vert f_{L_1}-f_{L_{2}}\right\Vert_K^2\\
\leq&\left(c\log(2)-1\right)\prod_{l=2}^n\left(1-\frac{1}{l}\right)\left\Vert \hat{f}_{0}-f_{L_1}\right\Vert_K^2+\prod_{l=2}^{n}\left(1-\frac{1}{l}\right)\left\Vert f_{L_1}-f_{L_{2}}\right\Vert_K^2\\
\leq&\frac{1}{n}\left\Vert f_{L_1}\right\Vert_K^2+\frac{1}{n}\left\Vert f_{L_1}-f_{L_{2}}\right\Vert_K^2=\frac{1}{n}\left\Vert f_{L_2}\right\Vert_K^2.
                \end{aligned} 
            \end{equation}
For $2\leq k\leq\frac{n}{2}-\frac{1}{2}$, we have
\begin{align}\label{eq:B.25}
\prod_{l=k+1}^{n}\left(1-c\frac{\log(l+1)}{l}\right)\leq&\exp\left(\sum_{l=k+1}^n\log\left(1-c\frac{\log(l+1)}{l}\right)\right)\notag\\
\leq&\exp\left(-c\sum_{l=k+1}^n\frac{\log(l+1)}{l}\right)\leq\exp\left(-c\sum_{l=k+1}^n\frac{\log(l)}{l}\right)\notag\\
\overset{\text{(i)}}{\leq}&\exp\left(-c\int_{x=k+1}^{n+1}\frac{\log(x)}{x}dx\right)\notag\\
=&\exp\left(-\frac{c}{2}\left[\left(\log(n+1)\right)^2-\left(\log\left(k+1\right)\right)^2\right]\right)\notag\\
\leq&\exp\left(-\frac{c}{2}\left[\left(\log(n+1)\right)^2-\left(\log\left(\frac{n+1}{2}\right)\right)^2\right]\right)\\
\leq&\exp\left(-\frac{c}{2}\left[\left(\log(n+1)\right)^2-\left(\log\left(n+1\right)-\log(2)\right)^2\right]\right)\notag\\
=&\exp\left(\frac{c}{2}\left(\log(2)\right)^2\right)\exp\left(-c\log(2)\log(n+1)\right)\notag\\
\leq&2\exp\left(-c\log(2)\log(n+1)\right)=\frac{2}{(n+1)^{c\log(2)}}\notag\\
\overset{\text{(ii)}}{\leq}&\frac{2}{n+1}.\notag
                \end{align} 
The function $\frac{\log(x)}{x}$ has derivative $\frac{1 - \log(x)}{x^2}$, so it is decreasing for $x \geq e$. Thus, the inequality in (i) holds. In (ii), we use the inequality $\frac{1}{\log(2)} \leq c \leq \frac{2}{\log(3)}$.
Next, we return to the second term in \eqref{eq:B.23}. By incorporating \eqref{eq:B.25}, we then obtain
\begin{equation}\label{eq:B.26}
                \begin{aligned}
&\sum_{k=2}^{\frac{n}{2}-\frac{1}{2}}\prod_{l=k+1}^{n}\left(1-c\frac{\log(l+1)}{l}\right)\left\Vert f_{L_k}-f_{L_{k+1}}\right\Vert_K^2\\
\leq&\frac{2}{n+1}\sum_{k=2}^{\frac{n}{2}-\frac{1}{2}}\left\Vert f_{L_k}-f_{L_{k+1}}\right\Vert_K^2=\frac{2}{n+1}\left\Vert f_{L_2}-f_{L_{\frac{n+1}{2}}}\right\Vert_K^2.
                \end{aligned} 
            \end{equation}
Finally, substituting the estimates from \eqref{eq:B.24} and \eqref{eq:B.26} into \eqref{eq:B.23} yields
\begin{equation*}
\begin{aligned}
&\left(c\log(2)-1\right)\prod_{l=2}^n\left(1-c\frac{\log(l+1)}{l}\right)\left\Vert \hat{f}_{0}-f_{L_1}\right\Vert_K^2\\ &+\sum_{k=1}^n\prod_{l=k+1}^n\left(1-c\frac{\log(l+1)}{l}\right)\left\Vert f_{L_k}-f_{L_{k+1}}\right\Vert_K^2\\
\leq&\left(c\log(2)-1\right)\prod_{l=2}^n\left(1-c\frac{\log(l+1)}{l}\right)\left\Vert \hat{f}_{0}-f_{L_1}\right\Vert_K^2 +\prod_{l=2}^{n}\left(1-c\frac{\log(l+1)}{l}\right)\left\Vert f_{L_1}-f_{L_{2}}\right\Vert_K^2\\ &+\sum_{k=2}^{\frac{n}{2}-\frac{1}{2}}\prod_{l=k+1}^{n}\left(1-c\frac{\log(l+1)}{l}\right)\left\Vert f_{L_k}-f_{L_{k+1}}\right\Vert_K^2+\left\Vert f_{L_{\frac{n}{2}-\frac{1}{2}}}-f_{L_{n+1}}\right\Vert_K^2\\
\leq&\frac{1}{n}\left\Vert f_{L_2}\right\Vert_K^2+\frac{2}{n+1}\left\Vert f_{L_2}-f_{L_{\frac{n+1}{2}}}\right\Vert_K^2+\left\Vert f_{L_{\frac{n}{2}-\frac{1}{2}}}-f_{L_{n+1}}\right\Vert_K^2\\
\leq&\frac{2}{n+1}\left\Vert f_{L_{\frac{n+1}{2}}}\right\Vert_K^2+\left\Vert f_{L_{\frac{n}{2}-\frac{1}{2}}}-f_{L_{n+1}}\right\Vert_K^2\\
\leq&\frac{2}{n+1}\left\Vert f^*\right\Vert_K^2+\left\Vert f_{L_{\frac{n}{2}-\frac{1}{2}}}-f_{L_{n+1}}\right\Vert_K^2\\
\overset{\text{(i)}}{\leq}&\frac{2}{n+1}Q^2+A_1^{2r-1}\left(\frac{n+1}{2}\right)^{-2\theta s(2r-1)}\Vert f^*\Vert_{\WW^r}^2\\
\leq&\left(2Q^2+2A_1^{2r-1}\Vert f^*\Vert_{\WW^r}^2\right)(n+1)^{-\frac{2r-1}{2r+1}}.
\end{aligned}
\end{equation*}
Here, (i) follows from \autoref{assum:regularity condition} that $f^* \in \WW=\left\{f\in\HH_K\,|\ \Vert f\Vert_K\leq Q\right\}$ and from the inequality $\left\Vert f_{L_n}-f_{L_m}\right\Vert_{K}^2\leq A_1^{2r-1}\left(n+1\right)^{-2\theta s(2r-1)}\Vert f^*\Vert_{\WW^r}^2$ for $L_m \geq L_n \in \mathbb{N}$, as stated in \autoref{lemma:B.6}. This completes the proof.
\end{proof}

\subsection{Proof of \autoref{theorem:mean result}}\label{proof of theorem 1}

We first prove part (a) of \autoref{theorem:mean result} in \autoref{subsec:alphasuffix} and \autoref{subsec:convergence of last iteration}; the proof of this part relies on the result of \autoref{theorem:strong convergence}. We then turn to part (b) in \autoref{proof of mean result b}, where we likewise begin by establishing a result analogous to \autoref{theorem:strong convergence}.

\subsubsection{Convergence Analysis of Suffix Averaging for \autoref{theorem:mean result} (a)}\label{subsec:alphasuffix}

Let the constant be $\widetilde{C}=\left[\frac{\left(2Q^2+3A_1^{2r-1}\Vert f^*\Vert_{\WW^r}^2\right)}{(\log(2))^2}+(4r+2)P^2\right]$. Then, the convergence result in \autoref{theorem:strong convergence} can be rewritten as follows
\begin{align*}
\EE\left[\left\Vert\hat{f}_n-f^*\right\Vert_{K}^2\right]&\leq\widetilde{C}\left(\log(n+1)\right)^2(n+1)^{-\frac{2r-1}{2r+1}}.
\end{align*}
Based on the recursive formula of $\hat{f}_n$ in \eqref{eq:generalized T-kernel SGD}, we obtain
\begin{equation}\label{eq:C.1}
                \begin{aligned}
&\EE\left[\left\Vert \hat{f}_n-f^*\right\Vert_K^2\right]\\
=&\EE\left[\left\Vert P_\WW\left(\hat{f}_{n-1}-\gamma_n\partial_u\ell\left(\hat{f}_{n-1}\circ F(X_n),Y_n\right)K_{L_n}^T(F(X_n),\cdot)\right)-f^*\right\Vert_K^2\right]\\
\leq&\EE\left[\left\Vert \hat{f}_{n-1}-\gamma_n\partial_u\ell\left(\hat{f}_{n-1}\circ F(X_n),Y_n\right)K_{L_n}^T(F(X_n),\cdot)-f^*\right\Vert_K^2\right]\\
\leq&\EE\left[\left\Vert \hat{f}_{n-1}-f^*\right\Vert_K^2\right]\\
&-2\gamma_n\EE\left[\lan\partial_u\ell\left(\hat{f}_{n-1}\circ F(X_n),Y_n\right)K_{L_n}^T(F(X_n),\cdot), \hat{f}_{n-1}-f^*\ran_K\right]+\gamma_n^2M_1^2\\
\overset{\text{(i)}}{=}&\EE\left[\left\Vert \hat{f}_{n-1}-f^*\right\Vert_K^2\right]\\
&-2\gamma_n\EE\left[\lan\partial_u\ell\left(\hat{f}_{n-1}\circ F(X_n),Y_n\right)K_{L_n}^T(F(X_n),\cdot), \hat{f}_{n-1}-f_{L_n}\ran_K\right]+\gamma_n^2M_1^2,
                \end{aligned} 
            \end{equation}
where (i) follows from the orthogonality between $K_{L_n}^T(F(X_n), \cdot) \in \mathcal{H}_{L_n}$ and $f_{L_n} - f^* \in \mathcal{H}_{L_n}^\perp$. Next, we consider the second term in the final expression of \eqref{eq:C.1}
\begin{equation}\label{eq:C.2}
                \begin{aligned}
&\EE\left[\lan\partial_u\ell\left(\hat{f}_{n-1}\circ F(X_n),Y_n\right)K_{L_n}^T(F(X_n),\cdot), \hat{f}_{n-1}-f_{L_n}\ran_K\right]\\
\overset{\text{(i)}}{=}&\EE\left[\lan\EE\left[\partial_u\ell\left(\hat{f}_{n-1}\circ F(X_n),Y_n\right)K_{L_n}^T(F(X_n),\cdot)\big|\ \DD_{n-1} \right], \hat{f}_{n-1}-f_{L_n}\ran_K\right]\\
=&\EE\left[\lan\nabla\ee(\hat{f}_{n-1})\big|_{\HH_{L_n}}, \hat{f}_{n-1}-f_{L_n}\ran_K\right]\\
\overset{\text{(ii)}}{\geq}&\EE\left[\ee(\hat{f}_{n-1})-\ee(f_{L_n})\right].
                \end{aligned} 
            \end{equation}
In (i), we define $\mathcal{D}_{n-1}$ as the $\sigma$-field generated by the observations $$\mathcal{D}_{n-1} = \sigma\left((X_1, Y_1), \dots, (X_{n-1}, Y_{n-1})\right).$$ In (ii), we use the convexity of $\ee(f)$ on the set $\mathcal{W} \cap \mathcal{H}_{L_n}$, as established in \autoref{lemma:A.5}.

Substituting \eqref{eq:C.2} into \eqref{eq:C.1} yields
\begin{equation*}
\begin{aligned}
&\EE\left[\left\Vert \hat{f}_n-f^*\right\Vert_K^2\right]\leq\EE\left[\left\Vert \hat{f}_{n-1}-f^*\right\Vert_K^2\right]-2\gamma_n\EE\left[\ee(\hat{f}_{n-1})-\ee(f_{L_n})\right]+\gamma_n^2M_1^2\\
\Rightarrow\quad &2\gamma_n\EE\left[\ee(\hat{f}_{n-1})-\ee(f_{L_n})\right]\leq\EE\left[\left\Vert \hat{f}_{n-1}-f^*\right\Vert_K^2\right]-\EE\left[\left\Vert \hat{f}_n-f^*\right\Vert_K^2\right]+\gamma_n^2M_1^2\\
\Rightarrow\quad &\EE\left[\ee(\hat{f}_{n-1})-\ee(f_{L_n})\right]\leq\frac{1}{2\gamma_n}\left(\EE\left[\left\Vert \hat{f}_{n-1}-f^*\right\Vert_K^2\right]-\EE\left[\left\Vert \hat{f}_n-f^*\right\Vert_K^2\right]\right)+\frac{\gamma_n}{2}M_1^2.
\end{aligned}
\end{equation*}
Summing the above inequality from $(1 - \alpha)n + 1$ to $n$, we obtain
\begin{equation*}
\begin{aligned}
&\sum_{k=(1-\alpha)n+1}^n\EE\left[\ee(\hat{f}_{k-1})-\ee(f_{L_k})\right]\\
\leq&\sum_{k=(1-\alpha)n+1}^n\frac{1}{2\gamma_k}\left(\EE\left[\left\Vert \hat{f}_{k-1}-f^*\right\Vert_K^2\right]-\EE\left[\left\Vert \hat{f}_k-f^*\right\Vert_K^2\right]\right)+\sum_{k=(1-\alpha)n+1}^n\frac{\gamma_k}{2}M_1^2\\
\leq&\frac{1}{2\gamma_{(1-\alpha)n}}\EE\left[\left\Vert \hat{f}_{(1-\alpha)n}-f^*\right\Vert_K^2\right]\\
&+\sum_{k=(1-\alpha)n}^{n-1}\EE\left[\left\Vert \hat{f}_k-f^*\right\Vert_K^2\right]\left(\frac{1}{2\gamma_{k+1}}-\frac{1}{2\gamma_k}\right)+\sum_{k=(1-\alpha)n+1}^n\frac{\gamma_k}{2}M_1^2\\
\overset{\text{(i)}}{\leq}&\left[\frac{\widetilde{C}}{2\gamma_{0}}+\frac{2r\widetilde{C}}{\gamma_0}+\frac{\gamma_0}{2}M_1^2(2r+1)\right]\log(n+1)n^{\frac{1}{2r+1}}.
\end{aligned}
\end{equation*}
Here, we obtain (i) by applying the estimate from \autoref{lemma:C.1}. By Jensen's inequality for the convex function $\ee(f)$ on $\mathcal{W}$, we have
\begin{equation}\label{eq:C.3}
                \begin{aligned}
&\EE\left[\ee\left(\bar{f}_{\alpha n}\right)-\frac{1}{\alpha n}\sum_{k=(1-\alpha)n+1}^{n}\ee\left(f_{L_k}\right)\right]\leq\frac{1}{\alpha n}\sum_{k=(1-\alpha)n+1}^{n}\EE\left[\ee(\hat{f}_{k-1})-\ee(f_{L_k})\right]\\
\leq&\frac{1}{\alpha}\left[\frac{\widetilde{C}}{2\gamma_{0}}+\frac{2r\widetilde{C}}{\gamma_0}+\frac{\gamma_0}{2}M_1^2(2r+1)\right]\log(n+1)n^{-\frac{2r}{2r+1}}.
                \end{aligned} 
            \end{equation}
Then we bound the term
\begin{equation}\label{eq:C.4}
                \begin{aligned}
&\frac{1}{\alpha n}\sum_{k=(1-\alpha)n+1}^{n}\left[\ee\left(f_{L_k}\right)-\ee(f^*)\right]\overset{\text{(i)}}{\leq}
\frac{1}{\alpha n}\sum_{k=(1-\alpha)n+1}^{n}\frac{L}{2}\left\Vert f_{L_k}\circ F-f^*\circ F\right\Vert_{\rho_X}^2\\
\overset{\text{(ii)}}{\leq}&\frac{1}{\alpha n}\frac{L}{2}B_\rho\Omega_{d-1} A_1^{2r}\Vert f^*\Vert_{\WW^r}^2\sum_{k=(1-\alpha)n+1}^{n}(k+1)^{-\frac{2r}{2r+1}}\\
\leq&\frac{1}{\alpha n}\frac{L}{2}B_\rho\Omega_{d-1} A_1^{2r}\Vert f^*\Vert_{\WW^r}^2\int_{x=(1-\alpha)n}^nx^{-\frac{2r}{2r+1}}dx\\
\leq&\frac{(2r+1)LB_\rho\Omega_{d-1} A_1^{2r}\Vert f^*\Vert_{\WW^r}^2}{2\alpha}n^{-\frac{2r}{2r+1}}.
                \end{aligned} 
            \end{equation}
In (i), we apply \autoref{lemma:A.3}, and in (ii), we apply \autoref{lemma:B.6}. Finally, we complete the proof by combining \eqref{eq:C.3} and \eqref{eq:C.4}.
\begin{equation*}
\begin{aligned}
&\EE\left[\ee\left(\bar{f}_{\alpha n}\right)-\ee\left(f^*\right)\right]\\
\leq&\EE\left[\ee\left(\bar{f}_{\alpha n}\right)-\frac{1}{\alpha n}\sum_{k=(1-\alpha)n+1}^{n}\ee\left(f_{L_k}\right)\right]+\frac{1}{\alpha n}\sum_{k=(1-\alpha)n+1}^{n}\left[\ee\left(f_{L_k}\right)-\ee(f^*)\right]\\
\leq&\frac{1}{\alpha}\left[\frac{\widetilde{C}}{2\gamma_{0}}+\frac{2r\widetilde{C}}{\gamma_0}+\frac{\gamma_0}{2}M_1^2(2r+1)+\frac{(2r+1)LB_\rho\Omega_{d-1} A_1^{2r}\Vert f^*\Vert_{\WW^r}^2}{2\log(2)}\right]\log(n+1)n^{-\frac{2r}{2r+1}}.
\end{aligned}
\end{equation*}

\subsubsection{Convergence Analysis of the Last Iteration for \autoref{theorem:mean result} (a)}\label{subsec:convergence of last iteration}
In this section, we use the results from \autoref{sec:strong convergence prove} and \autoref{subsec:alphasuffix} to analyze the convergence of $\hat{f}_n$. First, we choose $0 \leq m \leq i \leq n$, so that $\hat{f}_i, \hat{f}_m \in \mathcal{H}_{L_i} \cap \mathcal{W}$, and we have
\begin{equation*}
\begin{aligned}
&\EE\left[\left\Vert \hat{f}_{i+1}-\hat{f}_m\right\Vert_K^2\right]\\
=&\EE\left[\left\Vert P_\WW\left(\hat{f}_{i}-\gamma_{i+1}\partial_u\ell\left(\hat{f}_{i}\circ F(X_{i+1}),Y_{i+1}\right)K_{L_{i+1}}^T\left(F(X_{i+1}),\cdot\right)\right)-\hat{f}_m\right\Vert_K^2\right]\\
\leq&\EE\left[\left\Vert \hat{f}_{i}-\gamma_{i+1}\partial_u\ell\left(\hat{f}_{i}\circ F(X_{i+1}),Y_{i+1}\right)K_{L_{i+1}}^T\left(F(X_{i+1}),\cdot\right)-\hat{f}_m\right\Vert_K^2\right]\\
\leq&\EE\left[\left\Vert \hat{f}_{i}-\hat{f}_m\right\Vert_K^2\right]\\
&-2\gamma_{i+1}\EE\left[\lan\partial_u\ell\left(\hat{f}_{i}\circ F(X_{i+1}),Y_{i+1}\right)K_{L_{i+1}}^T\left(F(X_{i+1}),\cdot\right), \hat{f}_{i}-\hat{f}_m\ran_K\right]+\gamma_{i+1}^2M_1^2.
\end{aligned}
\end{equation*}
Since $\ee(f)$ is convex on $\mathcal{W}$, we have
\begin{equation}\label{eq:C.5}
                \begin{aligned}
&\EE\left[\ee\left(\hat{f}_{i}\right)-\ee\left(\hat{f}_m\right)\right]\\
\leq&\EE\left[\lan\partial_u\ell\left(\hat{f}_{i}\circ F(X_{i+1}),Y_{i+1}\right)K_{L_{i+1}}^T\left(F(X_{i+1}),\cdot\right), \hat{f}_{i}-\hat{f}_m\ran_K\right]\\
\leq&\frac{1}{2\gamma_{i+1}}\left(\EE\left[\left\Vert \hat{f}_{i}-\hat{f}_m\right\Vert_K^2\right]-\EE\left[\left\Vert \hat{f}_{i+1}-\hat{f}_m\right\Vert_K^2\right]\right)+\frac{\gamma_{i+1}}{2}M_1^2.
                \end{aligned} 
            \end{equation}
We sum both sides of \eqref{eq:C.5} from $i = n - k$ to $n$, where $k$ is an integer such that $1 \leq k \leq \frac{n}{2}$, and set $m = n - k$
\begin{equation}\label{eq:C.6}
                \begin{aligned}
&\sum_{i=n-k}^n\EE\left[\ee\left(\hat{f}_{i}\right)-\ee\left(\hat{f}_{n-k}\right)\right]\\
\leq&\sum_{i=n-k}^n\frac{1}{2\gamma_{i+1}}\left(\EE\left[\left\Vert \hat{f}_{i}-\hat{f}_{n-k}\right\Vert_K^2\right]-\EE\left[\left\Vert \hat{f}_{i+1}-\hat{f}_{n-k}\right\Vert_K^2\right]\right)+\sum_{i=n-k}^n\frac{\gamma_{i+1}}{2}M_1^2\\
\leq&\sum_{i=n-k+1}^n\EE\left[\left\Vert \hat{f}_{i}-\hat{f}_{n-k}\right\Vert_K^2\right]\left(\frac{1}{2\gamma_{i+1}}-\frac{1}{2\gamma_{i}}\right)+\sum_{i=n-k}^n\frac{\gamma_{i+1}}{2}M_1^2\\
\overset{\text{(i)}}{\leq}&\left[\frac{8\widetilde{C}}{\gamma_0}+\gamma_0M_1^2\right](k+1)(n+1)^{-\frac{2r}{2r+1}}\log(n+2),
                \end{aligned} 
            \end{equation}
where (i) is due to \autoref{lemma:C.3}. 

Let $S_k = \frac{1}{k+1} \sum_{i=n-k}^n \EE\left[\ee\left(\hat{f}_{i}\right)\right]$ denote the average expected population risk over the last $k+1$ iterations. Then, by applying \eqref{eq:C.6}, we obtain
\begin{equation}\label{eq:C.7}
-\EE\left[\ee\left(\hat{f}_{n-k}\right)\right]\leq -S_k+\left[\frac{8\widetilde{C}}{\gamma_0}+\gamma_0M_1^2\right](n+1)^{-\frac{2r}{2r+1}}\log(n+2).
\end{equation}
Combining the definition of $S_k$ with \eqref{eq:C.7} yields
\begin{equation}\label{eq:C.8}
                \begin{aligned}
kS_{k-1}&=(k+1)S_k-\EE\left[\ee\left(\hat{f}_{n-k}\right)\right]=kS_k+\left(S_k-\EE\left[\ee\left(\hat{f}_{n-k}\right)\right]\right)\\
&\leq kS_k+\left[\frac{8\widetilde{C}}{\gamma_0}+\gamma_0M_1^2\right](n+1)^{-\frac{2r}{2r+1}}\log(n+2)\\
\Rightarrow\quad S_{k-1}&\leq S_k+\frac{1}{k}\left[\frac{8\widetilde{C}}{\gamma_0}+\gamma_0M_1^2\right](n+1)^{-\frac{2r}{2r+1}}\log(n+2).
                \end{aligned} 
            \end{equation}
Applying \eqref{eq:C.8} recursively for $k = 1$ to $\frac{n-1}{2}$, we obtain
\begin{equation*}
\begin{aligned}
\EE\left[\ee\left(\hat{f}_n\right)\right]&=S_0\leq S_{\frac{n-1}{2}}+\left[\frac{8\widetilde{C}}{\gamma_0}+\gamma_0M_1^2\right](n+1)^{-\frac{2r}{2r+1}}\log(n+2)\sum_{k=1}^{\frac{n-1}{2}}\frac{1}{k}\\
&\leq S_{\frac{n-1}{2}}+\left[\frac{8\widetilde{C}}{\gamma_0}+\gamma_0M_1^2\right](n+1)^{-\frac{2r}{2r+1}}\log(n+2)\left(1+\log\left(\frac{n-1}{2}\right)\right)\\
&\leq S_{\frac{n-1}{2}}+2\left[\frac{8\widetilde{C}}{\gamma_0}+\gamma_0M_1^2\right](n+1)^{-\frac{2r}{2r+1}}\left(\log(n+2)\right)^2.
\end{aligned}
\end{equation*}
Based on the estimates of inequalities \eqref{eq:C.3} and \eqref{eq:C.4} in the convergence analysis of $\alpha$-suffix averaging, we obtain
\begin{align*}
&S_{\frac{n-1}{2}}-\ee(f^*)\\
\leq&2\left[\frac{\widetilde{C}}{2\gamma_{0}}+\frac{2r\widetilde{C}}{\gamma_0}+\frac{\gamma_0}{2}M_1^2(2r+1)+\frac{(2r+1)LB_\rho\Omega_{d-1} A_1^{2r}\Vert f^*\Vert_{\WW^r}^2}{2\log(2)}\right]\log(n+2)(n+1)^{-\frac{2r}{2r+1}}.
\end{align*}
Combining the two estimates above, we obtain the error for the last iteration stated in \autoref{theorem:mean result},
\begin{equation}\label{eq:last-step equation}
\begin{aligned}
&\EE\left[\ee\left(\hat{f}_n\right)-\ee(f^*)\right]\\
\leq&2\left[\frac{8\widetilde{C}}{\gamma_0}+\gamma_0M_1^2\right](n+1)^{-\frac{2r}{2r+1}}\left(\log(n+2)\right)^2\\
&+2\left[\frac{\widetilde{C}}{2\gamma_{0}}+\frac{2r\widetilde{C}}{\gamma_0}+\frac{\gamma_0}{2}M_1^2(2r+1)+\frac{(2r+1)LB_\rho\Omega_{d-1} A_1^{2r}\Vert f^*\Vert_{\WW^r}^2}{2\log(2)}\right]\log(n+2)(n+1)^{-\frac{2r}{2r+1}}.
\end{aligned}
\end{equation}

\subsubsection{Proof of \autoref{theorem:mean result} (b)}\label{proof of mean result b}
Proceeding as in the proof of \autoref{sec:strong convergence prove}, and letting $M_1^2:=M^2\kappa^2$, one obtains
\begin{align*}
&\EE\left[\Vert \hat{f}_n-f_{L_n}\Vert_K^2\right]\leq \EE\left[\Vert \hat{f}_{n-1}-f_{L_n}\Vert_K^2\right]-2\gamma_n\EE\left[\ee(\hat{f}_{n-1})-\ee(f_{L_n})\right]+\gamma_n^2M_1^2\\
\leq& \EE\left[\Vert \hat{f}_{n-1}-f_{L_n}\Vert_K^2\right]+2\gamma_n\left(\ee(f_{L_n})-\ee(f^*)\right)+\gamma_n^2M_1^2\\
=& \EE\left[\Vert \hat{f}_{n-1}-f_{L_{n-1}}\Vert_K^2\right]+\Vert f_{L_{n-1}}-f_{L_{n}}\Vert_K^2+2\gamma_n\left(\ee(f_{L_n})-\ee(f^*)\right)+\gamma_n^2M_1^2\\
\leq &\Vert \hat{f}_{0}-f_{L_{0}}\Vert_K^2+\sum_{k=1}^{n}\Vert f_{L_{k-1}}-f_{L_{k}}\Vert_K^2+2\sum_{k=1}^n\gamma_k\left(\ee(f_{L_k})-\ee(f^*)\right)+M_1^2\sum_{k=1}^n\gamma_k^2\\
\overset{\text{(i)}}{=}&\Vert f_{L_{n}}\Vert_K^2+2\sum_{k=1}^n\gamma_k\left(\ee(f_{L_k})-\ee(f^*)\right)+M_1^2\sum_{k=1}^n\gamma_k^2\\
\overset{\text{(ii)}}{\leq}&  (2d)^{2s(1-2r)}A_2^{2r-1}n^{2s(1-2r)\theta}\Vert f^*\Vert_{\WW^{r}}^2+
\gamma_0LB_\rho\Omega_{d-1} A_1^{2r}\Vert f^*\Vert_{\WW^{r}}^2\sum_{k=1}^nk^{-\frac{4r}{2r+1}}+M_1^2\gamma_0^2\sum_{k=1}^nk^{-\frac{4r}{2r+1}}\\
\overset{\text{(iii)}}{\leq}&(2d)^{2s(1-2r)}A_2^{2r-1}n^{-\frac{2r-1}{2r+1}}\Vert f^*\Vert_{\WW^{r}}^2+2\left(
\gamma_0LB_\rho\Omega_{d-1} A_1^{2r}\Vert f^*\Vert_{\WW^{r}}^2+M_1^2\gamma_0^2\right)\sum_{k=1}^n(k+1)^{-\frac{4r}{2r+1}}\\
\overset{\text{(iv)}}{\leq}&\left((2d)^{2s(1-2r)}A_2^{2r-1}\Vert f^*\Vert_{\WW^{r}}^2+2\left(
\gamma_0LB_\rho\Omega_{d-1} A_1^{2r}\Vert f^*\Vert_{\WW^{r}}^2+M_1^2\gamma_0^2\right)\frac{2r+1}{1-2r}\right)(n+1)^{-\frac{2r-1}{2r+1}}\\
=:&P_2^2(n+1)^{-\frac{2r-1}{2r+1}}.
\end{align*}
In (i), we use the orthogonality of $f_{L_{m+1}} - f_{L_m}$ and $f_{L_{k+1}} - f_{L_k}$ for $k\neq m$, together with $\hat{f}_0=0$. In (ii), we use \autoref{lemma:C.4}, and by following the same argument as in the proof of \autoref{lemma:B.6}, we obtain
$$\ee(f_{L_n})-\ee(f^*)\leq \frac{L}{2}B_\rho\Omega_{d-1} A_1^{2r}\Vert f^*\Vert_{\WW^{r}}^2(n+1)^{-4\theta sr}.$$
In (iii), we use the fact that $\frac{4r}{2r+1}\leq1$, which implies $k^{-\frac{4r}{2r+1}}\leq 2(k+1)^{-\frac{4r}{2r+1}}$. In deriving (iv), we further use the following inequality
$$ \sum_{k=1}^n(k+1)^{-\frac{4r}{2r+1}}\leq \int_1^{n+1} u^{-\frac{4r}{2r+1}} du\leq \frac{2r+1}{1-2r}(n+1)^{-\frac{2r-1}{2r+1}}.$$
Arguing as in the proof of \autoref{subsec:alphasuffix}, one obtains
\begin{equation*}
\begin{aligned}
&\EE\left[\ee(\hat{f}_{n-1})-\ee(f_{L_n})\right]\\
\leq&\frac{1}{2\gamma_n}\left(\EE\left[\left\Vert \hat{f}_{n-1}-f_{L_{n-1}}\right\Vert_K^2\right]-\EE\left[\left\Vert \hat{f}_n-f_{L_n}\right\Vert_K^2\right]\right)+\frac{1}{2\gamma_n}\left\Vert f_{L_n}-f_{L_{n-1}}\right\Vert_K^2+\frac{\gamma_n}{2}M_1^2.
\end{aligned}
\end{equation*}
Summing both sides of the above inequality, we have
\begin{equation}\label{eq:C.9}
\begin{aligned}
&\EE\left[\ee(\bar{f}_{\alpha n})-\ee(f^{*})\right]\leq\frac{1}{\alpha n}\sum_{k = (1-\alpha)n+1}^n\EE\left[\ee(\hat{f}_{k-1})-\ee(f^{*})\right]\\
\leq&\frac{1}{\alpha n}\sum_{k=(1-\alpha)n+1}^n\EE\left[\ee(\hat{f}_{k-1})-\ee(f_{L_k})\right]+\frac{1}{\alpha n}\sum_{k=(1-\alpha)n+1}^n\left(\ee(f_{L_k})-\ee(f^*)\right)\\
\leq&\frac{1}{\alpha n}\sum_{k=(1-\alpha)n+1}^n\frac{1}{2\gamma_k}\left(\EE\left[\left\Vert \hat{f}_{k-1}-f_{L_{k-1}}\right\Vert_K^2\right]-\EE\left[\left\Vert \hat{f}_k-f_{L_k}\right\Vert_K^2\right]\right)+\frac{1}{\alpha n}\sum_{k=(1-\alpha)n+1}^n\frac{\gamma_k}{2}M_1^2\\
&+\frac{1}{\alpha n}\sum_{k=(1-\alpha)n+1}^n\frac{1}{2\gamma_k}\left\Vert f_{L_k}-f_{L_{k-1}}\right\Vert_K^2
+\frac{1}{\alpha n}\sum_{k=(1-\alpha)n+1}^n\left(\ee(f_{L_k})-\ee(f^*)\right).
\end{aligned}
\end{equation}
Next, we derive upper bounds for each term in the above expression. We first consider
\begin{equation}\label{eq:C.10}
\frac{1}{\alpha n}\sum_{k=(1-\alpha)n+1}^n\frac{\gamma_k}{2}M_1^2\leq \frac{M_1^2\gamma_0}{\alpha}\frac{1}{n}\int_1^{n+1}u^{-\frac{2r}{2r+1}}du\leq\frac{2M_1^2\gamma_0}{\alpha}(2r+1)(n+1)^{-\frac{2r}{2r+1}}.
\end{equation}
Since the step-size sequence $\{\gamma_n\}$ is monotonically decreasing, one has
\begin{equation}\label{eq:C.11}
\begin{aligned}
&\frac{1}{\alpha n}\sum_{k=(1-\alpha)n+1}^n\frac{1}{2\gamma_k}\left\Vert f_{L_k}-f_{L_{k-1}}\right\Vert_K^2\leq\frac{1}{\alpha n}\frac{1}{2\gamma_n}\sum_{k=(1-\alpha)n+1}^n\left\Vert f_{L_k}-f_{L_{k-1}}\right\Vert_K^2\\
\leq&\frac{1}{2\alpha \gamma_0}n^{-1+\frac{2r}{2r+1}}\left\Vert f_{L_n}\right\Vert_K^2\overset{\text{(i)}}{\leq}\frac{1}{2\alpha \gamma_0}(2d)^{2s(1-2r)}A_2^{2r-1}\Vert f^*\Vert_{\WW^{r}}^2n^{-1+\frac{2r}{2r+1}}n^{\frac{1-2r}{2r+1}}\\
=&\frac{1}{2\alpha \gamma_0}(2d)^{2s(1-2r)}A_2^{2r-1}\Vert f^*\Vert_{\WW^{r}}^2n^{-\frac{2r}{2r+1}},
\end{aligned}
\end{equation}
where (i) is due to \autoref{lemma:C.4}. Then, we have
\begin{equation}\label{eq:C.12}
\begin{aligned}
&\frac{1}{\alpha n}\sum_{k=(1-\alpha)n+1}^n\left(\ee(f_{L_k})-\ee(f^*)\right)\leq  \frac{L}{2\alpha}B_\rho\Omega_{d-1} A_1^{2r}\Vert f^*\Vert_{\WW^{r}}^2\frac{1}{n}\sum_{k=(1-\alpha)n+1}^n(k+1)^{-4\theta sr}\\
=&\frac{L}{2\alpha}B_\rho\Omega_{d-1} A_1^{2r}\Vert f^*\Vert_{\WW^{r}}^2\frac{1}{n}\sum_{k=(1-\alpha)n+1}^n(k+1)^{-\frac{2r}{2r+1}}\\
\leq &\frac{L}{\alpha}B_\rho\Omega_{d-1} A_1^{2r}\Vert f^*\Vert_{\WW^{r}}^2\frac{1}{n+1}\int_{1}^{n+1}u^{-\frac{2r}{2r+1}}du\\
\leq&\frac{L}{\alpha}B_\rho\Omega_{d-1} A_1^{2r}\Vert f^*\Vert_{\WW^{r}}^2(2r+1)(n+1)^{-\frac{2r}{2r+1}}.
\end{aligned}
\end{equation}
Next, we bound the first term on the right-hand side of the inequality.
\begin{equation}\label{eq:C.13}
\begin{aligned}
&\frac{1}{\alpha n}\sum_{k=(1-\alpha)n+1}^n\frac{1}{2\gamma_k}\left(\EE\left[\left\Vert \hat{f}_{k-1}-f_{L_{k-1}}\right\Vert_K^2\right]-\EE\left[\left\Vert \hat{f}_k-f_{L_k}\right\Vert_K^2\right]\right)\\
\leq&\frac{1}{\alpha n}\frac{1}{2\gamma_{(1-\alpha)n}}\EE\left[\left\Vert \hat{f}_{(1-\alpha)n}-f_{L_{(1-\alpha)n}}\right\Vert_K^2\right]+\frac{1}{\alpha n}\sum_{k=(1-\alpha)n}^{n-1}\EE\left[\left\Vert \hat{f}_k-f_{L_k}\right\Vert_K^2\right]\left(\frac{1}{2\gamma_{k+1}}-\frac{1}{2\gamma_k}\right)\\
\leq&\frac{1}{2\gamma_0\alpha n}\left((1-\alpha)n\right)^{\frac{2r}{2r+1}}P_2^2((1-\alpha)n+1)^{-\frac{2r-1}{2r+1}}\\
&+\frac{1}{2\gamma_0\alpha n}\sum_{k=(1-\alpha)n}^{n-1}P_2^2(k+1)^{-\frac{2r-1}{2r+1}}\left((k+1)^{\frac{2r}{2r+1}}-k^{\frac{2r}{2r+1}}\right)\\
\leq&\frac{P_2^2}{\gamma_0\alpha }\left(n+1\right)^{-\frac{1}{2r+1}}(n+1)^{-\frac{2r-1}{2r+1}}+\frac{P_2^2}{2\gamma_0\alpha }\frac{2r}{2r+1}\frac{1}{n}\sum_{k=(1-\alpha)n}^{n-1}(k+1)^{-\frac{2r-1}{2r+1}}k^{-\frac{1}{2r+1}}\\
\leq&\frac{P_2^2}{\gamma_0\alpha }\left(n+1\right)^{-\frac{2r}{2r+1}}+\frac{P_2^2}{\gamma_0\alpha }\frac{2r}{2r+1}\frac{1}{n}\sum_{k=(1-\alpha)n}^{n-1}(k+1)^{-\frac{2r-1}{2r+1}}(k+1)^{-\frac{1}{2r+1}}\\
\leq&\frac{P_2^2}{\gamma_0\alpha }\left(n+1\right)^{-\frac{2r}{2r+1}}+\frac{2P_2^2}{\gamma_0\alpha }\frac{2r}{2r+1}\frac{1}{n+1}\int_1^{n+1}u^{-\frac{2r}{2r+1}}du\\
\leq&\left(\frac{P_2^2}{\gamma_0\alpha }+\frac{2P_2^2}{\gamma_0\alpha }2r\right)\left(n+1\right)^{-\frac{2r}{2r+1}}.
\end{aligned}
\end{equation}
Finally, substituting the bounds in \eqref{eq:C.10}, \eqref{eq:C.11}, \eqref{eq:C.12}, and \eqref{eq:C.13} into \eqref{eq:C.9} completes the proof of part (b) of \autoref{theorem:mean result}, namely,
$$ \EE\left[\ee(\bar{f}_{\alpha n})-\ee(f^{*})\right]\leq\O\left(n^{-\frac{2r}{2r+1}}\right).$$

\subsubsection{Technical Results}
\begin{lemma}\label{lemma:C.1}
Assuming that the assumptions and conclusions of \autoref{theorem:strong convergence} hold, then we have
\begin{align*}
&\frac{1}{2\gamma_{(1-\alpha)n}}\EE\left[\left\Vert \hat{f}_{(1-\alpha)n}-f^*\right\Vert_K^2\right]\\
&+\sum_{k=(1-\alpha)n}^{n-1}\EE\left[\left\Vert \hat{f}_k-f^*\right\Vert_K^2\right]\left(\frac{1}{2\gamma_{k+1}}-\frac{1}{2\gamma_k}\right)+\sum_{k=(1-\alpha)n+1}^n\frac{\gamma_k}{2}M_1^2\\
\leq&\left[\frac{\widetilde{C}}{2\gamma_{0}}+\frac{2r\widetilde{C}}{\gamma_0}+\frac{\gamma_0}{2}M_1^2(2r+1)\right]\log(n+1)n^{\frac{1}{2r+1}}.
\end{align*}
\end{lemma}
\begin{proof}
We now present the proof directly
\begin{align*}
&\frac{1}{2\gamma_{(1-\alpha)n}}\EE\left[\left\Vert \hat{f}_{(1-\alpha)n}-f^*\right\Vert_K^2\right]\\
&+\sum_{k=(1-\alpha)n}^{n-1}\EE\left[\left\Vert \hat{f}_k-f^*\right\Vert_K^2\right]\left(\frac{1}{2\gamma_{k+1}}-\frac{1}{2\gamma_k}\right)+\sum_{k=(1-\alpha)n+1}^n\frac{\gamma_k}{2}M_1^2\\
\leq&\frac{\left((1-\alpha)n\right)^{\frac{2r}{2r+1}}}{2\gamma_{0}\log\left((1-\alpha)n+1\right)}
\widetilde{C}\left(\log\left((1-\alpha)n+1\right)\right)^2\left((1-\alpha)n+1\right)^{-\frac{2r-1}{2r+1}}\\
&+\frac{\widetilde{C}}{2\gamma_0}\sum_{k=(1-\alpha)n}^{n-1}\left(\log(k+1)\right)^2(k+1)^{-\frac{2r-1}{2r+1}}\left(\frac{(k+1)^{\frac{2r}{2r+1}}}{\log(k+2)}-
\frac{k^{\frac{2r}{2r+1}}}{\log(k+1)}\right)\\ 
&+\frac{\gamma_0}{2}M_1^2\sum_{k=(1-\alpha)n+1}^nk^{-\frac{2r}{2r+1}}\log(k+1)\notag\\
\leq&\frac{\widetilde{C}}{2\gamma_{0}}\log\left((1-\alpha)n+1\right)\left((1-\alpha)n\right)^{\frac{1}{2r+1}}\\
&+\frac{\widetilde{C}}{2\gamma_0}\sum_{k=(1-\alpha)n}^{n-1}\left(\log(k+1)\right)^2(k+1)^{-\frac{2r-1}{2r+1}}\left(\frac{(k+1)^{\frac{2r}{2r+1}}}{\log(k+1)}-
\frac{k^{\frac{2r}{2r+1}}}{\log(k+1)}\right)\\ 
&+\frac{\gamma_0}{2}M_1^2\log(n+1)\sum_{k=(1-\alpha)n+1}^nk^{-\frac{2r}{2r+1}}\\
\leq&\frac{\widetilde{C}}{2\gamma_{0}}\log\left(n+1\right)n^{\frac{1}{2r+1}}
+\frac{\widetilde{C}}{2\gamma_0}\log(n+1)\sum_{k=(1-\alpha)n}^{n-1}(k+1)^{-\frac{2r-1}{2r+1}}\left((k+1)^{\frac{2r}{2r+1}}-k^{\frac{2r}{2r+1}}\right)\\ 
&+\frac{\gamma_0}{2}M_1^2\log(n+1)\int_{x=(1-\alpha)n}^nx^{-\frac{2r}{2r+1}}dx\\
\overset{\text{(i)}}{\leq}&\frac{\widetilde{C}}{2\gamma_{0}}\log\left(n+1\right)n^{\frac{1}{2r+1}}+
\frac{\widetilde{C}}{2\gamma_0}\log(n+1)\sum_{k=(1-\alpha)n}^{n-1}(k+1)^{-\frac{2r-1}{2r+1}}\left(\frac{2r}{2r+1}k^{-\frac{1}{2r+1}}\right)\\ &+\frac{\gamma_0}{2}M_1^2(2r+1)\log(n+1)n^{\frac{1}{2r+1}}\\
\overset{\text{(ii)}}{\leq}&\frac{\widetilde{C}}{2\gamma_{0}}\log\left(n+1\right)n^{\frac{1}{2r+1}}+
\frac{2\widetilde{C}}{2\gamma_0}\log(n+1)\frac{2r}{2r+1}\sum_{k=(1-\alpha)n}^{n-1}(k+1)^{-\frac{2r}{2r+1}}\\
&+\frac{\gamma_0}{2}M_1^2(2r+1)\log(n+1)n^{\frac{1}{2r+1}}\\
\leq&\frac{\widetilde{C}}{2\gamma_{0}}\log\left(n+1\right)n^{\frac{1}{2r+1}}+\frac{\widetilde{C}}{\gamma_0}(2r)\log(n+1)n^{\frac{1}{2r+1}}
+\frac{\gamma_0}{2}M_1^2(2r+1)\log(n+1)n^{\frac{1}{2r+1}}\\
\leq&\left[\frac{\widetilde{C}}{2\gamma_{0}}+\frac{2r\widetilde{C}}{\gamma_0}+\frac{\gamma_0}{2}M_1^2(2r+1)\right]\log(n+1)n^{\frac{1}{2r+1}},
\end{align*}
where (i) follows from Lagrange's mean value theorem. In (ii), we use the inequality $(k+1)^{\frac{1}{2r+1}} / k^{\frac{1}{2r+1}} \leq 2$. This completes the proof.
\end{proof}

\begin{lemma}\label{lemma:C.2}
Assuming the conditions of \autoref{theorem:strong convergence} hold, then for $\frac{n}{2} \leq n - k \leq i \leq n$, we have
\begin{align}
\EE\left[\left\Vert \hat{f}_{i}-\hat{f}_{n-k}\right\Vert_K^2\right]\leq8\widetilde{C}\left(\log(i+1)\right)^2\left(n+1\right)^{-\frac{2r-1}{2r+1}}.\notag
\end{align}
\end{lemma}
\begin{proof}
We complete the proof directly through the following derivation
\begin{equation*}
\begin{aligned}
\EE\left[\left\Vert \hat{f}_{i}-\hat{f}_{n-k}\right\Vert_K^2\right]&\leq2\EE\left[\left\Vert \hat{f}_{i}-f^*\right\Vert_K^2\right]+2\EE\left[\left\Vert \hat{f}_{n-k}-f^*\right\Vert_K^2\right]\\
&\leq2\widetilde{C}\left(\log(i+1)\right)^2\left(i+1\right)^{-\frac{2r-1}{2r+1}}+2\widetilde{C}\left(\log(n-k+1)\right)^2\left(n-k+1\right)^{-\frac{2r-1}{2r+1}}\\
&\leq2\widetilde{C}\left(\log(i+1)\right)^2\left[\left(i+1\right)^{-\frac{2r-1}{2r+1}}+\left(n-k+1\right)^{-\frac{2r-1}{2r+1}}\right]\\
&\leq4\widetilde{C}\left(\log(i+1)\right)^2\left(\frac{n+1}{2}\right)^{-\frac{2r-1}{2r+1}}\leq8\widetilde{C}\left(\log(i+1)\right)^2\left(n+1\right)^{-\frac{2r-1}{2r+1}}.
\end{aligned}
\end{equation*}

\end{proof}

\begin{lemma}\label{lemma:C.3}
Assuming the conditions of \autoref{theorem:strong convergence} hold, and noting that the first term in the following inequality is defined in \eqref{eq:C.6}, we obtain
\begin{align*}
&\sum_{i=n-k+1}^n\EE\left[\left\Vert \hat{f}_{i}-\hat{f}_{n-k}\right\Vert_K^2\right]\left(\frac{1}{2\gamma_{i+1}}-\frac{1}{2\gamma_{i}}\right)+\sum_{i=n-k}^n\frac{\gamma_{i+1}}{2}M_1^2\\
\leq&\left[\frac{8\widetilde{C}}{\gamma_0}+\gamma_0M_1^2\right](k+1)(n+1)^{-\frac{2r}{2r+1}}\log(n+2).
\end{align*}
\end{lemma}
\begin{proof}
This proof is similar to that of \autoref{lemma:C.1}. We present the proof directly
\begin{align*}
&\sum_{i=n-k+1}^n\EE\left[\left\Vert \hat{f}_{i}-\hat{f}_{n-k}\right\Vert_K^2\right]\left(\frac{1}{2\gamma_{i+1}}-\frac{1}{2\gamma_{i}}\right)+\sum_{i=n-k}^n\frac{\gamma_{i+1}}{2}M_1^2\\
\overset{\text{(i)}}{\leq}&\frac{8\widetilde{C}}{2\gamma_0}(n+1)^{-\frac{2r-1}{2r+1}}\sum_{i=n-k+1}^n\left(\log(i+1)\right)^2\left(\frac{(i+1)^{\frac{2r}{2r+1}}}{\log(i+2)}-
\frac{i^{\frac{2r}{2r+1}}}{\log(i+1)}\right)\\
&+\frac{\gamma_0M_1^2}{2}\sum_{i=n-k}^n(i+1)^{-\frac{2r}{2r+1}}\log(i+2)\\
\overset{\text{(ii)}}{\leq}&\frac{8\widetilde{C}}{2\gamma_0}(n+1)^{-\frac{2r-1}{2r+1}}\log(n+1)\sum_{i=n-k+1}^ni^{-\frac{1}{2r+1}}
+\frac{\gamma_0M_1^2}{2}\sum_{i=n-k}^n(i+1)^{-\frac{2r}{2r+1}}\log(i+2)\\
\overset{\text{(iii)}}{\leq}&\frac{8\widetilde{C}}{2\gamma_0}k(n+1)^{-\frac{2r-1}{2r+1}}\log(n+1)\left(\frac{n+1}{2}\right)^{-\frac{1}{2r+1}}
+\frac{\gamma_0M_1^2}{2}(k+1)\left(\frac{n+1}{2}\right)^{-\frac{2r}{2r+1}}\log(n+2)\\
\leq&\left[\frac{8\widetilde{C}}{\gamma_0}+\gamma_0M_1^2\right](k+1)(n+1)^{-\frac{2r}{2r+1}}\log(n+2),
\end{align*}
where (i) is due to the inequality in \autoref{lemma:C.2}:
\begin{align*}
\EE\left[\left\Vert \hat{f}_{i}-\hat{f}_{n-k}\right\Vert_K^2\right]\leq8\widetilde{C}\left(\log(i+1)\right)^2\left(n+1\right)^{-\frac{2r-1}{2r+1}}.
\end{align*}
In (ii), we apply Lagrange’s mean value theorem and use the inequality $\frac{1}{\log(i+2)} \leq \frac{1}{\log(i+1)}$. In (iii), we use the condition $\frac{n}{2} \leq n - k \leq n$. This completes the proof.
\end{proof}

\begin{lemma}\label{lemma:C.4}
Assuming the conditions of \autoref{theorem:mean result} (b) hold,  we have
\begin{align*}
\Vert f_{L_n}\Vert_K^2\leq (2d)^{2s(1-2r)}A_2^{2r-1}n^{2s(1-2r)\theta}\Vert f^*\Vert_{\WW^{r}}^2. 
\end{align*}
\end{lemma}
\begin{proof}
We present the proof directly
\begin{equation*}
\begin{aligned}
&\Vert f_{L_n}\Vert_K^2=\sum_{k=0}^{L_n}\sum_{j=1}^{\dim\HH_k^d}\frac{\left(f_{k,j}^*\right)^2}{a_k}\overset{\text{(i)}}{\leq} \sum_{k=0}^{L_n}\sum_{j=1}^{\dim\HH_k^d}\frac{\left(f_{k,j}^*\right)^2}{a_k}a_k^{1-2r}A_2^{2r-1}\left(\dim\Pi_k^d\right)^{2s(1-2r)}\\
\overset{\text{(ii)}}{\leq}& (2d)^{2s(1-2r)}A_2^{2r-1}\left(\dim\Pi_{L_n-1}^d\right)^{2s(1-2r)}\sum_{k=0}^{L_n}\sum_{j=1}^{\dim\HH_k^d}\frac{\left(f_{k,j}^*\right)^2}{a_k^{2r}}\\
\leq&(2d)^{2s(1-2r)}A_2^{2r-1}n^{2s(1-2r)\theta}\Vert f^*\Vert_{\WW^{r}}^2,
\end{aligned}
\end{equation*}
where (i) is due to
$$ A_2^{1-2r}\left(\dim\Pi_k^d\right)^{-2s(1-2r)}\leq a_k^{1-2r},\quad \forall\ 0<r\leq\frac12.$$
In (ii), we use $\dim\Pi_{L_n-1}^d\leq n^\theta\leq\dim\Pi_{L_n}^d$ and $\dim\Pi_{L_n}^d\leq 2d\cdot\dim\Pi_{L_n-1}^d$ in Lemma 12 in \cite{bai2025truncated}.
\end{proof}

\subsection{Proof of \autoref{proposition:least-square result}}\label{proof of Proposition 1}

In this section, we prove \autoref{proposition:least-square result}. By Euler's inequality (Section 7.12-3 in \cite{ciarlet2013linear}), we have for any $f \in \WW$ that
$$\lan \nabla\ee(f^*)\big|_{\HH_K},f-f^*\ran_K\geq0.$$
Combining this with the inequality in \autoref{lemma:A.5}, we obtain
$$\ee(f)-\ee(f^*)\geq\frac{\mu}{2}\Vert f\circ F-f^*\circ F\Vert_{\rho_X}^2.$$
Finally, we complete the proof by applying the following inequalities
\begin{equation*}
\begin{aligned}
\EE\left[\left\Vert\hat{f}_n\circ F-f^*\circ F\right\Vert_{\rho_X}^2\right]&\leq\frac{2}{\mu}\EE\left[\ee\left(\hat{f}_n\right)-\ee\left(f^*\right)\right]\leq \O\left(n^{-\frac{2r}{2r+1}}\left(\log(n+1)\right)^2\right)\\
\EE\left[\left\Vert\bar{f}_{\alpha n}\circ F-f^*\circ F\right\Vert_{\rho_X}^2\right]&\leq\frac{2}{\mu}\EE\left[\ee\left(\bar{f}_{\alpha n}\right)-\ee\left(f^*\right)\right]\leq \O\left(n^{-\frac{2r}{2r+1}}\log(n+1)\right).
\end{aligned}
\end{equation*}

\subsection{Proof of \autoref{theorem:optimal storage}}\label{subsection:proof of theorem:optimal storage}

In this section, we provide the proof of \autoref{theorem:optimal storage}. We consider the following Sobolev ellipsoid characterized by parameters $s>\frac{1}{2}$ and $r \geq \frac{1}{2}$, with $l:=\lim_{k\to\infty} a_k\cdot\left(\dim \Pi_k^d\right)^{2s} \in(0,\infty)$,
$$\mathcal{S}(4sr,Q)=\left\{\sum_{k=0}^\infty\sum_{j=1}^{\dim\HH_k^d}f_{k,j}Y_{k,j}\,\big\vert \sum_{k=0}^\infty\sum_{j=1}^{\dim\HH_k^d}\frac{f_{k,j}^2}{a_k^{2r}}\leq Q^2\right\}.$$
It is straightforward to verify that $\mathcal{S}(4sr,Q) \subseteq \WW^{r}\left(\SS^{d-1}\right)$. Moreover, since $0 < a_k \leq 1$, we also have $\mathcal{S}(4sr,Q) \subseteq \WW$. Consequently, we have $\mathcal{S}(4sr,Q) \subseteq \WW^{r}\left(\SS^{d-1}\right)\cap\WW$. By arranging the orthonormal eigensystem $\{(a_k^{r}, Y_{k,j})\}_{0\leq k,1\leq j\leq\dim\HH_k^d}$ in lexicographic order, we obtain the sequence $\{(\lambda_j, \phi_j)\}_{j\geq 1}$. It is then immediate that $\{\phi_j\}_{j\geq 1} = \{Y_{0,1},Y_{1,1},Y_{1,2},\dots,Y_{2,1},Y_{2,2},\cdots\}$. Using the bound $A_2\left(\dim\Pi_k^d\right)^{-2s}\leq a_k\leq A_1 \left(\dim\Pi_k^d\right)^{-2s}$ together with Lemma 6 in \cite{bai2025truncated}, we obtain
$$A_2^{r}d^{-2sr}\frac{1}{j^{2sr}}\leq \lambda_j\leq A_1^{r}\frac{1}{j^{2sr}}\quad \forall j\in\NN.$$
Using the rearranged orthonormal eigensystem $\{(\lambda_j, \phi_j)\}_{j \geq 1}$, we can rewrite the Sobolev ellipsoid $\mathcal{S}(4sr, Q)$ as
$$\mathcal{S}(4sr, Q)=\left\{\sum_{j=1}^{\infty}f_j\phi_j\,\Bigg\vert \sum_{j=1}^{\infty}\frac{f_j^2}{\lambda_j^2}\leq Q^2\right\}.$$
As in the proof of Example 5.12 in \cite{wainwright2019high}, we obtain the asymptotic bounds for the metric entropy of $\mathcal{S}(4sr, Q)$. Specifically, there exist constants $A_3 \geq1\geq A_4 > 0$ such that
$$A_4\left(\frac{1}{\delta}\right)^{\frac{1}{2sr}}\leq\log N\left(\delta;\mathcal{S}(4sr, Q),\Vert\cdot\Vert_\omega\right)\leq A_3\left(\frac{1}{\delta}\right)^{\frac{1}{2sr}}\quad\text{for all small enough } \delta > 0.$$
Here we take an arbitrary estimator $G_n = D_n \circ E_n$, which is an $l_n$-sized estimator as described in \autoref{theorem:optimal storage} with $l_n = o\left(n^{\tfrac{1}{2s(2r+1)}}\right)$. We next introduce the notion of an $\epsilon$-net with respect to the decoder $D_n$, which is used to characterize how the collection of $l_n$-sized estimators $G(l_n)$ can approximate the function class under an error tolerance $\epsilon$,
$$\text{net}\left(\epsilon,l_n,D_n,\mathcal{S}(4sr, Q)\right)=\left\{f\in\mathcal{S}(4sr, Q)\,\bigg\vert\, \exists\, b_n\in\{0,1\}^{l_n}, \text{such that } \Vert f-D_n(b_n)\Vert_\omega\leq \epsilon\right\}.$$
Furthermore, by the definition of $l_n$, there exists a sequence $m_n$ such that $l_n=o(m_n)$ and $m_n= o\left(n^{\frac{1}{2s(2r+1)}}\right)$. Here, setting $\delta = m_n^{-2sr}$, the metric entropy satisfies
$$\log_2 N\left(m_n^{-2sr};\mathcal{S}(4sr, Q),\Vert\cdot\Vert_\omega\right)\geq  A_4\log_2(e)m_n\geq A_4m_n.$$
Since $l_n = o(m_n)$, the set $D_n\left(\{0,1\}^{l_n}\right)$, which contains at most $2^{l_n}$ elements, cannot form an $m_n^{-2sr}$-cover of $\mathcal{S}(4sr, Q)$ for sufficiently large $n$, namely
$$\mathcal{S}(4sr, Q)\backslash \text{net}\left(m_n^{-2sr},l_n,D_n,\mathcal{S}(4sr, Q)\right)\neq\emptyset.$$
Let us denote $\alpha_n=E_n\left(\{(X_i,Y_i)\}_{1\leq i \leq n}\right)\in\{0,1\}^{l_n}$. Then we have
\begin{align*}
&\sup_{f^*\in \WW^{r}\left(\SS^{d-1}\right)\cap\WW}\EE\left[\Vert G_n\left(\{(X_i,Y_i)\}_{1\leq i \leq n}\right)-f^*\Vert_\omega^2\right]\\
\geq&\sup_{f^*\in \mathcal{S}(4sr, Q)}\EE\left[\Vert G_n\left(\{(X_i,Y_i)\}_{1\leq i \leq n}\right)-f^*\Vert_\omega^2\right]\\
=&\sup_{f^*\in \mathcal{S}(4sr, Q)}\EE\left[\Vert D_n\left(\alpha_n\right)-f^*\Vert_\omega^2\right]\\
\geq&\sup_{f^*\in \mathcal{S}(4sr, Q)\backslash \text{net}\left(m_n^{-2sr},l_n,D_n,\mathcal{S}(4sr, Q)\right)}\EE\left[\Vert D_n\left(\alpha_n\right)-f^*\Vert_\omega^2\right]\\
\geq&\sup_{f^*\in \mathcal{S}(4sr, Q)\backslash \text{net}\left(m_n^{-2sr},l_n,D_n,\mathcal{S}(4sr, Q)\right)}\inf_{\alpha_n\in\{0,1\}^{l_n}}\Vert D_n\left(\alpha_n\right)-f^*\Vert_\omega^2\geq \left(m_n^{-2sr}\right)^2.
\end{align*}
Consequently, we obtain
\begin{align*}
\inf_{G_n\in G(l_n)}\sup_{f^*\in \WW^{r}\left(\SS^{d-1}\right)\cap\WW}\EE\left[n^{\frac{2r}{2r+1}}\Vert G_n\left(\{(X_i,Y_i)\}_{1\leq i \leq n}\right)-f^*\Vert_\omega^2\right]\geq n^{\frac{2r}{2r+1}}m_n^{-4sr}.
\end{align*}
Taking the limit as $n \to \infty$ on both sides yields the conclusion of \autoref{theorem:optimal storage}.

\subsection{Lemmas and Proofs for Robustness to Hyperparameter Perturbation}\label{subsec:Lemmas and Proofs for Robustness to Hyperparameter Perturbation}
\begin{lemma}
Suppose that the assumptions in part (a) of \autoref{theorem:mean result} hold. For $r_1 > r>\frac12$, let $\theta = \frac{1}{2s(2r_1+1)}$ and choose the step size $\gamma_n = \gamma_0 n^{-\frac{2r_1}{2r_1+1}} \log(n+1)$, where $\gamma_0 = c\frac{A_14(2d)^{2s}}{A^2_2 b_{\rho} \mu \Omega_{d-1}}$ for some constant $c \in \left[\frac{1}{\log 2}, \frac{2}{\log 3}\right]$. Then, for any $\alpha \in (0,1)$, the following bounds hold:
\begin{align*}
\EE\left[\left\Vert\hat{f}_n-f^*\right\Vert_{K}^2\right]&\leq\O\left((n+1)^{-\frac{2r-1}{2r_1+1}}\log(n+1)\right),\\
\EE\left[\ee\left(\bar{f}_{\alpha n}\right)-\ee\left(f^*\right)\right]&\leq \O\left(n^{-\frac{2r}{2r_1+1}}\right).
\end{align*}
\end{lemma}
\begin{proof}
Since the proof of the first bound follows the same line of argument as that of \autoref{theorem:strong convergence}, and the proof of the second bound is analogous to that of \autoref{theorem:mean result}, we present only a brief outline here.
\begin{equation*}
\begin{aligned}
&\EE\left[\left\Vert \hat{f}_n-f_{L_{n+1}}\right\Vert_K^2\right]\\
\leq&\left(1-\frac{A^2_2}{A_1}\frac{b_{\rho}\mu\Omega_{d-1}}{4(2d)^{2s}}\gamma_n n^{-2\theta s}\right)\EE\left[\left\Vert \hat{f}_{n-1}-f_{L_n}\right\Vert_K^2\right]\\
&+\gamma_n\left(\frac{\mu}{2}+\frac{8L^2}{\mu}\right)\frac{L}{\mu}B_\rho\Omega_{d-1} A_1^{2r}\Vert f^*\Vert_{\WW^r}^2\left(n+1\right)^{-4\theta sr}\\
&+\gamma_nLB_\rho\Omega_{d-1} A_1^{2r}\Vert f^*\Vert_{\WW^r}^2\left(n+1\right)^{-4\theta sr}+\gamma_n^2M_1^2+\left\Vert f_{L_{n+1}}-f_{L_n}\right\Vert_K^2\\
\leq&\left(1-c \frac{\log(n+1)}{n}\right)\EE\left[\left\Vert \hat{f}_{n-1}-f_{L_n}\right\Vert_K^2\right]+D_1n^{-\frac{2r_1+2r}{2r_1+1}}\log(n+1)+\left\Vert f_{L_{n+1}}-f_{L_n}\right\Vert_K^2\\
\leq&\left(c\log(2)-1\right)\prod_{l=2}^n\left(1-c\frac{\log(l+1)}{l}\right)\left\Vert \hat{f}_{0}-f_{L_1}\right\Vert_K^2\\
&+\sum_{k=1}^n\prod_{l=k+1}^n\left(1-c\frac{\log(l+1)}{l}\right)\left\Vert f_{L_k}-f_{L_{k+1}}\right\Vert_K^2\\
&+D_1\sum_{k=1}^n\prod_{l=k+1}^n\left(1-c\frac{\log(l+1)}{l}\right)k^{-\frac{2r_1+2r}{2r_1+1}}\log(k+1)\\
\leq&D_2n^{-\frac{2r-1}{2r_1+1}}+D_3n^{-\frac{2r-1}{2r_1+1}}\log(n+1)\leq D_4n^{-\frac{2r-1}{2r_1+1}}\log(n+1).
\end{aligned}
\end{equation*}
We use that $D_i$ for $i\geq1$ denotes a constant independent of $n$, together with the inequality $n^{-\frac{2r_1}{2r_1+1}}\log(n+1)\leq (n+1)^{-\frac{2r}{2r_1+1}}$, for $n$ large enough. Proceeding as in the proof of \autoref{lemma:B.8}, we obtain
\begin{equation*}
\begin{aligned}
&\left(c\log(2)-1\right)\prod_{l=2}^n\left(1-c\frac{\log(l+1)}{l}\right)\left\Vert \hat{f}_{0}-f_{L_1}\right\Vert_K^2\\
&+\sum_{k=1}^n\prod_{l=k+1}^n\left(1-c\frac{\log(l+1)}{l}\right)\left\Vert f_{L_k}-f_{L_{k+1}}\right\Vert_K^2\leq D_2n^{-\frac{2r-1}{2r_1+1}},
\end{aligned}
\end{equation*}
which yields the bound involving $D_2$. We now derive the bound involving $D_3$:
\begin{equation*}
\begin{aligned}
&D_1\sum_{k=1}^n\prod_{l=k+1}^n\left(1-c\frac{\log(l+1)}{l}\right)k^{-\frac{2r_1+2r}{2r_1+1}}\log(k+1)\\
\leq&D_1\sum_{k=1}^n\prod_{l=k+1}^n\left(1-\frac{1}{l}\right)k^{-\frac{2r_1+2r}{2r_1+1}}\log(k+1)\\
\leq&D_1\sum_{k=1}^n\frac{k}{n}k^{-\frac{2r_1+2r}{2r_1+1}}\log(k+1)\leq 4D_1\log(n+1)\frac{1}{n+1}\sum_{k=1}^n(k+1)^{-\frac{2r-1}{2r_1+1}}\\
\leq&4D_1\log(n+1)\frac{1}{n+1}\int_{1}^{n+1}u^{-\frac{2r-1}{2r_1+1}}du\leq 4D_1\log(n+1)\frac{2r_1+1}{2(r_1-r)+2}(n+1)^{-\frac{2r-1}{2r_1+1}}\\
=:&D_3(n+1)^{-\frac{2r-1}{2r_1+1}}\log(n+1).
\end{aligned}
\end{equation*}
We next prove the second part of the lemma. Proceeding as in the proof of \autoref{theorem:mean result}, we obtain
\begin{align*}
&\sum_{k=(1-\alpha)n+1}^n\EE\left[\ee(\hat{f}_{k-1})-\ee(f_{L_k})\right]\\
\leq&\sum_{k=(1-\alpha)n+1}^n\frac{1}{2\gamma_k}\left(\EE\left[\left\Vert \hat{f}_{k-1}-f^*\right\Vert_K^2\right]-\EE\left[\left\Vert \hat{f}_k-f^*\right\Vert_K^2\right]\right)+\sum_{k=(1-\alpha)n+1}^n\frac{\gamma_k}{2}M_1^2\\
\leq&\frac{1}{2\gamma_{(1-\alpha)n}}\EE\left[\left\Vert \hat{f}_{(1-\alpha)n}-f^*\right\Vert_K^2\right]\\
&+\sum_{k=(1-\alpha)n}^{n-1}\EE\left[\left\Vert \hat{f}_k-f^*\right\Vert_K^2\right]\left(\frac{1}{2\gamma_{k+1}}-\frac{1}{2\gamma_k}\right)+\sum_{k=(1-\alpha)n+1}^n\frac{\gamma_k}{2}M_1^2\\
\leq&\frac{D_4}{2\gamma_0}((1-\alpha)n)^{\frac{2r_1-2r+1}{2r_1+1}}+\frac{M_1^2\gamma_0}{2}\sum_{k=(1-\alpha)n+1}^n k^{-\frac{2r_1}{2r_1+1}}\log(k+1)\\
&+\frac{D_4}{2\gamma_0}\sum_{k=(1-\alpha)n}^{n-1}k^{-\frac{2r-1}{2r_1+1}}\log(k+1)\left(\frac{(k+1)^{\frac{2r_1}{2r_1+1}}}{\log(k+2)}
-\frac{k^{\frac{2r_1}{2r_1+1}}}{\log(k+1)}\right)\\
\leq&\frac{D_4}{2\gamma_0}n^{\frac{2r_1-2r+1}{2r_1+1}}+M_1^2\gamma_0\log(n+1)\int_1^{n+1} u^{-\frac{2r_1}{2r_1+1}}du\\
+&\frac{D_4}{\gamma_0}\sum_{k=(1-\alpha)n}^{n-1}(k+1)^{-\frac{2r-1}{2r_1+1}}\left((k+1)^{\frac{2r_1}{2r_1+1}}
-k^{\frac{2r_1}{2r_1+1}}\right)\\
\leq&\frac{D_4}{2\gamma_0}n^{\frac{2r_1-2r+1}{2r_1+1}}+M_1^2\gamma_0(2r_1+1)\log(n+1)(n+1)^{\frac{1}{2r_1+1}}+\frac{2D_4}{\gamma_0}\frac{2r_1}{2r_1+1}\sum_{k=(1-\alpha)n}^{n-1}(k+1)^{-\frac{2r}{2r_1+1}}\\
\leq&\frac{D_4}{2\gamma_0}n^{1-\frac{2r}{2r_1+1}}+M_1^2\gamma_0(2r_1+1)\log(n+1)(n+1)^{\frac{1}{2r_1+1}}+\frac{2D_4}{\gamma_0}\frac{2r_1}{2r_1+1-2r}(n+1)^{1-\frac{2r}{2r_1+1}}\\
\leq&D_5n^{1-\frac{2r}{2r_1+1}}.
\end{align*}
The last inequality follows from the condition $r_1>r$, which implies that there exists a constant $D_6$ such that $\log(n+1)(n+1)^{\frac{1}{2r_1+1}}\leq D_6 n^{\frac{2r_1-2r+1}{2r_1+1}}$ for $n\in\NN_+$. As in \eqref{eq:C.4}, we obtain
\begin{equation*}
                \begin{aligned}
&\frac{1}{\alpha n}\sum_{k=(1-\alpha)n+1}^{n}\left[\ee\left(f_{L_k}\right)-\ee(f^*)\right]
\leq\frac{1}{\alpha n}\frac{L}{2}B_\rho\Omega_{d-1} A_1^{2r}\Vert f^*\Vert_{\WW^r}^2\sum_{k=(1-\alpha)n+1}^{n}(k+1)^{-\frac{2r}{2r_1+1}}\\
\leq&\frac{L}{\alpha}B_\rho\Omega_{d-1} A_1^{2r}\Vert f^*\Vert_{\WW^r}^2\frac{2r_1+1}{2r_1+1-2r}(n+1)^{-\frac{2r}{2r_1+1}}=:D_7(n+1)^{-\frac{2r}{2r_1+1}}.
                \end{aligned} 
            \end{equation*}
The proof is completed by combining the above two inequalities.
\begin{equation*}
\begin{aligned}
&\EE\left[\ee(\bar{f}_{\alpha n})-\ee(f^{*})\right]\leq\frac{1}{\alpha n}\sum_{k = (1-\alpha)n+1}^n\EE\left[\ee(\hat{f}_{k-1})-\ee(f^{*})\right]\\
\leq&\frac{1}{\alpha n}\sum_{k=(1-\alpha)n+1}^n\EE\left[\ee(\hat{f}_{k-1})-\ee(f_{L_k})\right]+\frac{1}{\alpha n}\sum_{k=(1-\alpha)n+1}^n\left(\ee(f_{L_k})-\ee(f^*)\right)\\
\leq&\frac{D_5}{\alpha }n^{-\frac{2r}{2r_1+1}}+D_7(n+1)^{-\frac{2r}{2r_1+1}}=\O\left(n^{-\frac{2r}{2r_1+1}}\right).
\end{aligned}
\end{equation*}
\end{proof}

\begin{lemma}
Suppose that the assumptions in part (a) of \autoref{theorem:mean result} hold. For $r\geq\frac{1}{2}$, let $\theta >\frac{1}{8sr}$ and choose the step size $\gamma_n = \gamma_0 n^{-\frac{1}{2}} \left(\log(n+1)\right)^{-1}$ with $\gamma_0 \in(0,1]$. Then, for any $\alpha \in (0,1)$, the following bound holds:
\begin{align*}
\EE\left[\ee\left(\bar{f}_{\alpha n}\right)-\ee\left(f^*\right)\right]&\leq \O\left(n^{-\frac12}\log(n+1)\right).
\end{align*}
\end{lemma}
\begin{remark}
As the proof of this lemma is entirely analogous to that of the previous lemma, we present the result without repeating the argument.
\end{remark}
\end{document}